\documentclass{article}

\usepackage[T2A]{fontenc}			
\usepackage[utf8]{inputenc}			
\usepackage{microtype}
\usepackage{graphicx}
\usepackage{subfigure}
\usepackage{booktabs} 
\usepackage{amsmath, amssymb, amsthm}
\usepackage{color}

\usepackage{hyperref}
\usepackage[capitalize,nameinlink]{cleveref}

\usepackage{geometry} 

\newcommand{\Atype}{\mathsf{A}}
\newcommand{\Gtype}{\mathsf{G}}
\newcommand{\Htype}{\mathsf{H}}

\newcommand{\trspop}{\mathrm{Trsp}}
\newcommand{\matmul}{\mathrm{MatMul}}
\newcommand{\lincomb}{\mathrm{LinComb}}
\newcommand{\nonlin}{\mathrm{Nonlin}}

\newcommand{\trsp}{\top}

\newcommand{\HH}{\mathcal{H}}

\newcommand{\PP}{\mathcal{P}}
\newcommand{\LL}{\mathcal{L}}

\newcommand{\XX}{\mathcal{X}}
\newcommand{\EE}{\mathbb{E}\,}
\newcommand{\DD}{\mathcal{D}}
\newcommand{\RR}{\mathbb{R}}

\newcommand{\NN}{\mathcal{N}}

\newcommand{\FF}{\mathcal{F}}

\newcommand{\citet}{\cite}

\newcommand{\Rad}[2]{\mathrm{Rad}({#1} \, | \, {#2})}

\DeclareMathOperator*{\argmin}{arg\,min}

\DeclareMathOperator*{\diag}{diag}
\DeclareMathOperator*{\tr}{tr}
\DeclareMathOperator*{\supp}{supp}

\DeclareMathOperator*{\erf}{erf}

\newtheorem{theorem}{Theorem}

\newtheorem{lemma}{Lemma}

\newtheorem{conjecture}{Conjecture}

\makeatletter
\def\mathcolor#1#{\@mathcolor{#1}}
\def\@mathcolor#1#2#3{%
  \protect\leavevmode
  \begingroup
    \color#1{#2}#3%
  \endgroup
}

\title{Neural Tangent Kernel: A Survey}

\author{
    \textbf{Eugene Golikov*}
    \\École Polytechnique Fédérale de Lausanne, Switzerland
    \\\texttt{evgenii.golikov@epfl.ch}
    \and
    \textbf{Eduard Pokonechnyy}
    \\Moscow Institute of Physics and Technology, Russia
    \\\texttt{pokonechnyy.ep@phystech.edu}
    \and
    \textbf{Vladimir Korviakov}
    \\Huawei Technologies, Russia
    \\\texttt{korviakov.vladimir1@huawei.com}
}

\begin{document}
    \maketitle

    \tableofcontents

    \begin{abstract}

    A seminal work of \cite{jacot2018neural} demonstrated that training a neural network under certain parameterization is equivalent to performing a certain kernel method as width goes to infinity.
    This equivalence opened a promising direction of applying results of rich literature on kernel methods to neural nets which were much harder to tackle.
    The present survey covers key results on kernel convergence as width goes to infinity, finite-width corrections, applications, and discussion of limitations of the corresponding method.
    \end{abstract}

    \section{Definition and the explicit solution for square loss}
    \label{sec:definition}

    Consider a generic parametric model $f(x; \theta): \, \XX \times \RR^N \to \RR$ differentiable with respect to weights $\theta$.
    We aim to minimize square loss over a dataset $(\vec x, \vec y)$ of size $m$: $\frac{1}{2} \sum_{j=1}^m (y_j - f(x_j; \theta))^2 \to \min_\theta$.
    A continuous-time gradient descent dynamics (gradient flow) corresponds to the following ordinary differential equation (ODE):
    \begin{equation}
        \dot\theta_t
        = -\nabla_\theta\left(\frac{1}{2} \sum_{j=1}^m (y_j - f(x_j; \theta_t))^2\right)
        = \sum_{j=1}^m (y_j - f(x_j; \theta_t)) \nabla_\theta f(x_j; \theta_t).
    \end{equation}
    Let us abbreviate the prediction at a given data point $x$ at time $t$, $f(x; \theta_t)$, as $f_t(x)$.
    Under the dynamics above, this quantity evolves as
    \begin{equation}
        \dot f_t(x)
        = \dot\theta_t^T \nabla f_t(x)
        = \sum_{j=1}^m (y_j - f_t(x_j)) \nabla_\theta^T f_t(x_j) \nabla_\theta f_t(x).
        \label{eq:f_t_dynamics}
    \end{equation}
    
    If we perceive $\nabla_\theta f_t(x)$ as a feature map $\Phi_t: \, \XX \to \RR^N$, the scalar product above becomes a kernel evaluated at a pair $(x_j,x)$.
    This kernel is called an empirical neural tangent kernel (NTK) and is denoted by $\hat\Theta_t$:
    \begin{equation}
        \hat\Theta_t(x,x')
        = \nabla_\theta^T f_t(x) \nabla_\theta f_t(x').
    \end{equation}
    This definition allows for a shorter representation of the prediction dynamics (\ref{eq:f_t_dynamics}):
    \begin{equation}
        \dot f_t(x)
        = \hat\Theta_t(x,\vec x) (\vec y - f_t(\vec x)),
        \label{eq:f_t_dynamics_emp_ntk}
    \end{equation}
    where by convention, $\hat\Theta_t(x,\vec x) \in \RR^{1 \times m}$.

    Assume that the empirical NTK does not evolve with time, i.e $\hat\Theta_t(x,x') = \hat\Theta_0(x,x')$ $\forall x,x' \in \XX$.
    This assumption is equivalent to assuming the model $f(x;\theta)$ to be linear as a function of its weights:
    \begin{equation}
        f(x; \theta)
        = f(x; \theta_0) + \nabla_\theta^T f(x; \theta_0) (\theta - \theta_0).
    \end{equation}

    When the kernel is constant, Eq.(\ref{eq:f_t_dynamics_emp_ntk}) is easily integrable.
    Indeed, on the train dataset,
    \begin{equation}
        \dot f_t(\vec x)
        = \hat\Theta_0(\vec x,\vec x) (\vec y - f_t(\vec x)),
        \label{eq:f_t_dynamics_train}
    \end{equation}
    which gives
    \begin{equation}
        f_t(\vec x)
        = f_0(\vec x) - \left(I - e^{-\hat\Theta_0(\vec x, \vec x) t}\right) (f_0(\vec x) - \vec y).
    \end{equation}
    Plugging it back to Eq.(\ref{eq:f_t_dynamics_emp_ntk}) gives
    \begin{equation}
        \dot f_t(x)
        = \hat\Theta_t(x,\vec x) e^{-\hat\Theta_0(\vec x, \vec x) t} (\vec y - f_0(\vec x)),
    \end{equation}
    and finally,
    \begin{equation}
        f_t(x) 
        = f_0(x) - \hat\Theta_0(x, \vec x) \hat\Theta_0^{-1}(\vec x, \vec x) \left(I - e^{-\hat\Theta_0(\vec x, \vec x) t}\right) (f_0(\vec x) - \vec y).
        \label{eq:lin_solution_square_loss}
    \end{equation}

    While the exact solution above is based on the constant kernel assumption, one can prove that the kernel is indeed nearly constant in certain settings, see \cref{sec:convergence}.
    This allows one to transfer results that hold for linearized models to original ones.
    
    For example, $f_t(\vec x)$ converges to $\vec y$ (i.e. the model learns the dataset) as long as the Gram matrix is positive definite: $\hat\Theta_0(\vec x,\vec x) \geq \lambda_0$ for some $\lambda_0 > 0$, see Eq.(\ref{eq:f_t_dynamics_train}).
    The same result holds without the constant kernel assumption, as long as $\hat\Theta_t(\vec x,\vec x)$ stays sufficiently close to $\hat\Theta_0(\vec x,\vec x)$, and therefore, say, $\hat\Theta_t(\vec x,\vec x) \geq \lambda_0/2$.
    Indeed,
    \begin{equation}
        \frac{d}{dt}\left(\frac{1}{2} \| \vec y - f_t(\vec x) \|_2^2\right)
        = -(\vec y - f_t(\vec x))^T \hat\Theta_t(\vec x, \vec x) (\vec y - f_t(\vec x))
        \leq -\frac{\lambda_0}{2} \| \vec y - f_t(\vec x) \|_2^2,
    \end{equation}
    which gives
    \begin{equation}
        \| \vec y - f_t(\vec x) \|_2^2
        \leq e^{-\lambda_0 t} \| \vec y - f_0(\vec x) \|_2^2
        \to 0 \quad \text{as $t \to \infty$};
    \end{equation}
    see \cite{du2018gradient} for the formal result.
    This result is not trivial, since loss surfaces of generic neural nets are non-convex, and therefore any local optimization method (e.g. the gradient flow) may get stuck in a spurious local minimum.
    See \cite{arora2019fine} for other results of a similar kind.

    Also, if one assumes the kernel to be nearly constant, one can identify certain pathologies affecting the learning process by analyzing the initial kernel: see \cite{martens2021rapid} discussing trainability of very deep nets and \cite{dupuis2021dnn,tancik2020fourier} fixing blurry results of image regression.

    Finally, the exact solution (\ref{eq:lin_solution_square_loss}) can be used as a substitute for the usual gradient descent training routine.
    A naive approach for evaluating Eq.(\ref{eq:lin_solution_square_loss}) would be to compute the initial kernel $\hat\Theta_0(\vec x, \vec x)$ and then to invert it.
    Naively computing the kernel requires $O(N m^2)$ time and $O(m^2)$ memory, while inverting it takes $O(m^3)$ more time.
    Such an approach is infeasible for datasets of realistic sizes (i.e. $m \gtrsim 10^5$), asking for major optimizations, see \cite{novak2019neural,novakfast,meanti2020kernel}.
    Nevertheless, for $m \lesssim 10^4$, the direct approach is feasible and gives promising results, see \cite{arora2019harnessing}.
    Also, in certain scenarios, the kernel can be efficiently scaled from small $m$ to larger ones, see \cite{radhakrishnan2021simple}.

    \section{Kernel convergence}
    \label{sec:convergence}

    The goal of this section is to validate the constant kernel assumption: $\hat\Theta_t(x,x') = \hat\Theta_0(x,x')$ $\forall x,x' \in \XX$.
    The main result is: under certain parameterization, the empirical NTK of a neural network becomes constant as width goes to infinity.
    Before stating this result formally, we provide an illustrative example.

    Consider a neural network with one hidden layer, scalar input, and Gaussian-initialized weights:
    \begin{equation}
        f(x; a_{1:n}, w_{1:n})
        = \sum_{i=1}^n a_i \phi(w_i x),
        \quad
        a_{1:n} \sim \NN(0, n^{-1} I),
        \quad
        w_{1:n} \sim \NN(0, I).
        \label{eq:1_hid_net_standard}
    \end{equation}
    Here $n$ is width of the hidden layer; following a standard initialization scheme \cite{he2015delving}, initialization variance of each layer is inversely proportional to the number of input neurons.

    The above parameterization of the network is the one typically used in practice; we shall refer it as standard.
    However, the parameterization we need is a different one:
    \begin{equation}
        f(x; a_{1:n}, w_{1:n})
        = \frac{1}{\sqrt{n}} \sum_{i=1}^n a_i \phi(w_i x),
        \quad
        a_{1:n} \sim \NN(0, I),
        \quad
        w_{1:n} \sim \NN(0, I).
        \label{eq:1_hid_net_ntk}
    \end{equation}
    We shall refer it as NTK-parameterization.
    Note that it does not alter the distribution of neurons, both hidden and output, at initialization but it does alter the gradient flow:
    \begin{equation}
        \dot a_k 
        = \frac{1}{\sqrt{n}} \sum_{j=1}^m \phi(w_k x_j),
        \quad
        \dot w_k
        = \frac{1}{\sqrt{n}} \sum_{j=1}^m a_k \phi'(w_k x_j) x_j.
    \end{equation}
    Here input and output weights receive $O(n^{-1/2})$ increments, while both of them are $O(1)$ at initialization.
    Hence $a_k(t) \to a_k(0)$ and $w_k(t) \to w_k(0)$ as $n \to \infty$ for any fixed $k \in \mathbb{N}$ and $t \in \RR_+$.

    Compare with gradient flow under standard parameterization:
    \begin{equation}
        \dot a_k 
        = \sum_{j=1}^m \phi(w_k x_j),
        \quad
        \dot w_k 
        = \sum_{j=1}^m a_k \phi'(w_k x_j) x_j.
    \end{equation}
    Here the output weights are $O(n^{-1/2})$ at initialization but receive $O(1)$ increments for $t=0$, while the input weights are $O(1)$ at initialization but receive $O(n^{-1/2})$ increments for $t=0$.

    Let us write the NTK under NTK parameterization:
    \begin{multline}
        \hat\Theta_t(x,x')
        = \sum_{i=1}^n \left(\partial_{a_i} f(x) \partial_{a_i} f(x') + \partial_{w_i} f(x) \partial_{w_i} f(x')\right)
        =\\= \frac{1}{n} \sum_{i=1}^n \left(\phi(w_i(t) x) \phi(w_i(t) x') + a_i^2(t) \phi'(w_i(t) x) \phi'(w_i(t) x') x x'\right).
    \end{multline}
    Since $a_k(t) \to a_k(0)$ and $w_k(t) \to w_k(0)$ as $n \to \infty$ for any fixed $k \in \mathbb{N}$ and $t \in \RR_+$, the above expression is asymptotically equivalent to
    \begin{equation}
        \hat\Theta_0(x,x')
        = \frac{1}{n} \sum_{i=1}^n \left(\phi(w_i(0) x) \phi(w_i(0) x') + a_i^2(0) \phi'(w_i(0) x) \phi'(w_i(0) x') x x'\right),
    \end{equation}
    which converges (almost surely) to
    \begin{equation}
        \Theta(x,x')
        = \EE_{a,w \sim \NN(0,1)} \left(\phi(w x) \phi(w x') + a^2 \phi'(w x) \phi'(w x') x x'\right)
    \end{equation}
    as $n \to \infty$ due to the (strong) Law of Large Numbers.
    The limit kernel $\Theta(x,x')$ depends neither on a timestep $t$, nor on initialization.
    This kernel is typically referred as NTK, contrasting to the empirical NTK $\hat\Theta_t$.

    Since under standard parameterization the weights receive increments asymptotically at least comparable to initialization, one cannot expect that the empirical NTK stops evolving as $n \to \infty$ in this setting.
    Moreover, the initial empirical NTK diverges with width:
    \begin{multline}
        \hat\Theta_0(x,x')
        = \sum_{i=1}^n \left(\phi(w_i(0) x) \phi(w_i(0) x') + a_i^2(0) \phi'(w_i(0) x) \phi'(w_i(0) x') x x'\right)
        \sim\\\sim n \times \EE_{w \sim \NN(0,1)} \phi(w x) \phi(w x').
    \end{multline}

    The above kernel convergence result holds in more general settings.
    Consider a fully-connected network with $L$ layers under NTK parameterization:
    \begin{equation}
        f(x) = h_L(x),
        \quad
        h_l(x) = \frac{1}{\sqrt{n_{l-1}}} W_l x_{l-1}(x),
        \quad
        x_{l-1}(x) = \phi(h_{l-1}(x)),
        \quad
        x_0(x) = x,
    \end{equation}
    where $W_1 \in \RR^{n_1 \times n_0}$, $W_L \in \RR^{1 \times n_{L-1}}$, and $W_l \in \RR^{n_l \times n_{l-1}}$ for all other $l$.
    Here all weights are initialized with independent standard Gaussians.
    Suppose we aim to optimize a generic differentiable loss $\ell$ instead of the quadratic one:
    \begin{equation}
        \dot\theta_t
        = -\nabla_\theta\left(\sum_{j=1}^m \ell(y_j, f(x_j; \theta_t))\right)
        = \sum_{j=1}^m \left.\frac{\partial \ell(y_j, z)}{\partial z}\right|_{z=f(x_j; \theta_t)} \nabla_\theta f(x_j; \theta_t),
    \end{equation}
    where $\theta$ now is a concatenation of all weights $W_{1:L}$.
    The seminal work of \cite{jacot2018neural} proves the following:
    \begin{theorem}[\cite{jacot2018neural}]
        Under the conditions above, for $\phi$ being $C^2$ and Lipschitz and $\ell$ being $C^1$ and Lipschitz, $\hat\Theta_t(x,x') \to \Theta(x,x')$ in probability as $n_{1:L-1} \to \infty$ sequentially $\forall x,x' \in \XX$ $\forall t \geq 0$.
    \end{theorem}

    In fact, the theorem above can be generalized far from fully-connected nets with smooth activation functions.
    Define a tensor program as a set of initial variables of certain types and a sequence of operations.
    Each of the operations generates a new variable by acting on previously generated ones.
    The variable types are
    \begin{enumerate}
        \item $\Atype$: $n \times n$ matrices with iid $\NN(0,1)$ entries;
        \item $\Gtype$: vectors of size $n$ with asymptotically iid Gaussian entries;
        \item $\Htype$: images of $\Gtype$-vars by coordinatewise nonlinearities.
    \end{enumerate}
    The operations are
    \begin{enumerate}
        \item $\trspop$: $W: \Atype \to W^\trsp: \Atype$;
        \item $\matmul$: $(W: \Atype, \; x: \Htype) \to \frac{1}{\sqrt{n}} W x: \Gtype$; 
        \item {$\lincomb$: $(\{x_i: \Gtype, \; a_i \in \RR\}_{i=1}^k) \to \sum_{i=1}^k a_i x_i: \Gtype$;} 
        \item $\nonlin$: $(\{x_i: \Gtype\}_{i=1}^k, \; \phi: \RR^k \to \RR) \to \phi(x_{1:k}): \Htype$. 
    \end{enumerate}
    The set of initial variables consists of variables of $\Atype$-type and $\Gtype$-type.
    As for input $\Gtype$-vars, we sample $\{x_\alpha: \text{$x$ is an input G-var}\} \sim \NN(\mu^{in}, \Sigma^{in})$ $\forall \alpha \in [n]$.

    The above formalism allows to express forward and backward passes of a very wide class of neural nets (including RNNs, ResNets, and Transformers).
    Besides none of the operations above generates new $\Atype$-vars (new weights), the whole gradient descent training process can be expressed as a single tensor program by backtracking the gradient steps.
    The real power of tensor programs comes from the following theorem:
    \begin{theorem}["Master theorem", \cite{yang2020tensor_iii}]
        \label{thm:master_theorem}
        Consider a tensor program with $M$ $\Gtype$-vars, under above assumptions.
        Suppose all the nonlinearities $\phi$ and a function $\psi: \, \RR^M \to \RR$ are polynomially bounded.
        Then the following holds:
        \begin{equation}
            \frac{1}{n} \sum_{\alpha=1}^n \psi(g^1_\alpha,\ldots,g^M_\alpha)
            \to \EE_{Z \sim \NN(\mu,\Sigma)} \psi(Z)
        \end{equation}
        a.s. as $n \to \infty$, where $\mu$ and $\Sigma$ can be computed using certain recurrent rules.
    \end{theorem}

    It is possible to define the empirical NTK of a tensor program and express it in the form $\frac{1}{n} \sum_{\alpha=1}^n \psi(g^1_\alpha,\ldots,g^M_\alpha)$ for a certain function $\psi$.
    Then the kernel converges by virtue of the above theorem.
    See \cite{yang2020tensor_ii} for the proof of initial kernel convergence and \cite{yang2021tensor_iib} for the proof of kernel convergence for any timestep.

    As an illustration, recall the two-layered net considered at the beginning of the present section.
    Its empirical NTK is given by
    \begin{equation}
        \hat\Theta_0(x,x')
        = \frac{1}{n} \sum_{i=1}^n \left(\phi(w_i(0) x) \phi(w_i(0) x') + a_i^2(0) \phi'(w_i(0) x) \phi'(w_i(0) x') x x'\right).
    \end{equation}
    Here $\Gtype$-vars are $g^1 = w(0) x$, $g^2 = w(0) x'$, $g^3 = a(0) x$, $g^4 = a(0) x'$.
    Taking $\psi(g^1_\alpha,\ldots,g^4_\alpha) = \phi(g^1_\alpha) \phi(g^2_\alpha) + \phi'(g^1_\alpha) \phi'(g^2_\alpha) g^3_\alpha g^4_\alpha$ allows for explicit application of Theorem \ref{thm:master_theorem}.

    \section{Finite-width corrections}
    \label{sec:finite_width}

    While the results discussed in \cref{sec:convergence} hold in the limit of infinite width, they are not directly applicable to real-life finite-width nets for obvious reasons.
    This motivates one to introduce finite-width corrections for the limit NTK.

    First, define a higher-order kernel:
    \begin{equation}
        O_{s,t}(x_{1:s})
        = \nabla^T_\theta O_{s-1,t}(x_{1:s-1}) \nabla_\theta f_t(x_s).
    \end{equation}
    Put $O_{1,t}(x_1) = f_t(x_1)$; this gives $O_{2,t}(x_1,x_2) = \hat\Theta_t(x_1,x_2)$.

    Consider a gradient flow optimization process under square loss:
    \begin{equation}
        \dot\theta_t
        = \sum_{j=1}^m (y_j - f_t(x_j)) \nabla_\theta f_t(x_j).
    \end{equation}
    Under this process, the $s$-order kernel evolves as
    \begin{equation}
        \dot O_{s,t}(x_{1:s})
        = \nabla^T_\theta O_{s,t}(x_{1:s}) \dot\theta
        = O_{s+1,t}(x_{1:s},\vec x) (\vec y - f_t(\vec x)).
    \end{equation}
    This gives an infinite system of ODE's governing the evolution of the kernels.

    If our goal is to obtain a solution ony up to the order of $n^{-1}$, will it allow us to truncate the initially infinite system?
    How many equations should we keep?
    In order to answer these questions, let us estimate the order of growth for $O_{s,t}$.

    Following \cite{Dyer2020Asymptotics}, we start with a definition of a correlation function.
    Let us fix $t=0$ and omit the corresponding subscript for now.
    Define a rank-$k$ derivative tensor $T_{\mu_1 \ldots \mu_k}$ as follows:
    \begin{equation}
        T_{\mu_1 \ldots \mu_k}(x; f) =
        \frac{\partial^k f(x)}{\partial \theta^{\mu_1} \ldots \partial \theta^{\mu_k}}.
    \end{equation}
    For $k=0$ we define $T(x; f) = f(x)$.
    We are now ready to define a correlation function $C$:
    \begin{equation}
        C(x_1,\ldots,x_m) =
        \sum_{\mu_1,\ldots,\mu_{k_m}} \Delta_{\mu_1 \ldots \mu_{k_m}}^{(\pi)} \EE_\theta \left(
            T_{\mu_1 \ldots \mu_{k_1}}(x_1) T_{\mu_{k_1+1} \ldots \mu_{k_2}}(x_2) \ldots T_{\mu_{k_{m-1}+1} \ldots \mu_{k_m}}(x_m) 
        \right).
    \end{equation}
    Here $0 \leq k_1 \leq \ldots \leq k_m$, $k_m$ and $m$ are even, $\pi \in S_{k_m}$ is a permutation, and $\Delta_{\mu_1 \ldots \mu_{k_m}}^{(\pi)} = \delta_{\mu_{\pi(1)} \mu_{\pi(2)}} \ldots \delta_{\mu_{\pi(k_m-1)} \mu_{\pi(k_m)}}$.
    For example, 
    \begin{multline}
        \EE_\theta (f(x) \nabla^T_\theta f(x) \nabla_\theta \nabla^T_\theta f(x_1) \nabla_\theta f(x_2)) = 
        \sum_{\mu,\nu} \EE_\theta (f(x) \partial_\mu f(x) \partial^2_{\mu,\nu} f(x_1) \partial_\nu f(x_2)) =\\= 
        \sum_{\mu_1,\mu_2,\mu_3,\mu_4} \delta_{\mu_1 \mu_2} \delta_{\mu_3 \mu_4} \EE_\theta (f(x) \partial_{\mu_1} f(x) \partial^2_{\mu_2,\mu_3} f(x_1) \partial_{\mu_4} f(x_2)) =
        C(x,x,x_1,x_2)
        \label{eq:dtheta_dt_as_corr_f}
    \end{multline}
    is a correlation function with $m=4$, $k_1=0$, $k_2=1$, $k_3=3$, $k_4=4$, and $\pi(j) = j$.

    If two derivative tensors have two indices that are summed over, we say that they are contracted.
    Formally, we say that $T_{\mu_{k_{i-1}+1} \ldots \mu_{k_i}}(x_i)$ is contracted with $T_{\mu_{k_{j-1}+1} \ldots \mu_{k_j}}(x_j)$ for $1 \leq i,j \leq m$ if there exists an even $s \leq k_m$ such that $k_{i-1} < \pi(s-1) \leq k_i$, while $k_{j-1} < \pi(s) \leq k_j$, or vice versa.

    Define the cluster graph $G_C(V,E)$ as a non-oriented non-weighted graph with vertices $V = \{v_1, \ldots, v_m\}$ and edges $E = \{(v_i,v_j) \, | \, \text{$T(x_i)$ and $T(x_j)$ are contracted in $C$}\}$.
    Let $n_e$ be the number of even-sized connected components of $G_C(V,E)$ and $n_o$ be the number of odd-sized components.
    We are going to use the following conjecture, which is proven in certain scenarios:
    \begin{conjecture}[\cite{Dyer2020Asymptotics}]
        \label{conj:C_asymptotics}
        If $m$ is even, $C(x_1,\ldots,x_m) = O_{n\to\infty}(n^{s_C})$, where $s_C = n_e + n_o / 2 - m / 2$.
        If $m$ is odd, $C(x_1,\ldots,x_m) = 0$.
    \end{conjecture}

    We are also going to use the following lemma:
    \begin{lemma}[\cite{Dyer2020Asymptotics}]
        \label{lemma:derivative_asymptotics}
        Suppose \cref{conj:C_asymptotics} holds.
        Let $C(\vec x) = \EE_\theta F(\vec x; \theta)$ be a correlation function and suppose $C(\vec x) = O(n^{s_C})$ for $s_C$ defined in \cref{conj:C_asymptotics}.
        Then $\EE_\theta d^k F(\vec x; \theta) / dt^k = O(n^{s_C})$ $\forall k \geq 1$.
    \end{lemma}
    \begin{proof}
        Consider the first derivative:
        \begin{multline}
            \EE_\theta \frac{dF(\vec x)}{dt} =
            \EE_\theta (\dot\theta^T \nabla_\theta F(\vec x)) =
            \EE_{x,y} \EE_\theta (\eta (y - f(x)) \nabla^T_\theta f(x) \nabla_\theta F(\vec x)) =\\=
            \eta \EE_{x,y} \EE_\theta (y \nabla^T_\theta f(x) \nabla_\theta F(\vec x)) -
            \eta \EE_{x,y} \EE_\theta (f(x) \nabla^T_\theta f(x) \nabla_\theta F(\vec x)).
        \end{multline}
        This is a sum of a linear combination of correlation functions.
        By \cref{conj:C_asymptotics}, the first sum evaluates to zero, while the second one has $m' = m+2$, $n_e'$ even clusters, and $n_o'$ odd clusters.
        If $\nabla_\theta f(x)$ is contracted with an even cluster of $C$, we have $n_e' = n_e - 1$, $n_o' = n_o + 2$.
        In contrast, if $\nabla_\theta f(x)$ is contracted with an odd cluster of $C$, we have $n_e' = n_e + 1$, $n_o' = n_o$.
        
        In the first case, we have $s_C' = n_e' + n_o'/2 - m'/2 = s_C - 1$, while for the second $s_C' = s_C$.
        In any case, the result is a linear combination of correlation functions with $s_C' \leq s_C$ for each.
    \end{proof}

    Let us return the $t$-subscript.
    Since $O_s$ has $s$ derivative tensors and a single cluster, by virtue of \cref{conj:C_asymptotics}, $\EE_\theta O_{s,0} = O(n^{1 - s/2})$ for even $s$ and $\EE_\theta O_{s,0} = 0$ for odd $s$.
    At the same time, $\EE_\theta \dot O_{s,0} = O(n^{1 - (s+2)/2}) = O(n^{-s/2})$ for even $s$ and $\EE_\theta \dot O_{s,0} = O(n^{1 - (s+1)/2}) = O(n^{1/2 - s/2})$ for odd $s$.

    As for the second moments, we have $\EE_\theta (O_{s,0})^2 = O(n^{2 - s})$ for even $s$ and $\EE_\theta (O_{s,0})^2 = O(n^{1 - s})$ for odd $s$.
    Similarly, we have $\EE_\theta (\dot O_{s,0})^2 = O(n^{2/2 - (2s+2)/2}) = O(n^{-s})$ for even $s$ and $\EE_\theta (\dot O_{s,0})^2 = O(n^{2 - (2s+2)/2}) = O(n^{1 - s})$ for odd $s$.

    The asymptotics for the first two moments implies the asymptotic for a random variable itself:
    \begin{equation}
        O_{s,0}(x_{1:s}) =
        \begin{cases}
            O(n^{1 - s/2}) &\text{for even $s$;}
            \\
            O(n^{1/2 - s/2}) &\text{for odd $s$;}
        \end{cases}
        \qquad
        \dot O_{s,0}(x_{1:s}) =
        \begin{cases}
            O(n^{-s/2}) &\text{for even $s$;}
            \\
            O(n^{1/2 - s/2}) &\text{for odd $s$.}
        \end{cases}
    \end{equation}
    \cref{lemma:derivative_asymptotics} gives $\forall k \geq 1$:
    \begin{equation}
        \left.\frac{d^k O_{s,t}}{dt^k}(x_{1:s})\right|_{t=0} =
        \begin{cases}
            O(n^{-s/2}) &\text{for even $s$;}
            \\
            O(n^{1/2 - s/2}) &\text{for odd $s$.}
        \end{cases}
    \end{equation}
    Then given an analytic activation function, we have $\forall t \geq 0$:
    \begin{equation}
        \dot O_{s,t}(x_{1:s}) =
        \sum_{k=1}^\infty \left.\frac{d^k O_{s,t}}{dt^k}(x_{1:s})\right|_{t=0} \frac{t^k}{k!} =
        \begin{cases}
            O(n^{-s/2}) &\text{for even $s$;}
            \\
            O(n^{1/2 - s/2}) &\text{for odd $s$.}
        \end{cases}
    \end{equation}

    This allows us to write a finite system of ODE for the model evolution up to $O(n^{-1})$ terms: 
    \begin{equation}
        \dot f_{t}(x_1) = 
        O_{2,t}(x_1, \vec x) (\vec y - f_t(\vec x)),
        \qquad
        f_0(x_1) =
        f(x_1; \theta),
        \quad
        \theta \sim
        \NN(0, I),
    \end{equation}
    \begin{equation}
        \dot O_{2,t}(x_1, x_2) = 
        O_{3,t}(x_1, x_2, \vec x) (\vec y - f_t(\vec x)),
        \qquad
        O_{2,0}(x_1, x_2) =
        \nabla_\theta^T f_0(x_1) \nabla_\theta f_0(x_2),
    \end{equation}
    \begin{equation}
        \dot O_{3,t}(x_1, x_2, x_3) = 
        O_{4,t}(x_1, x_2, x_3, \vec x) (\vec y - f_t(\vec x)),
        \qquad
        O_{3,0}(x_1, x_2, x_3) =
        \nabla_\theta^T O_{2,0}(x_1, x_2) \nabla_\theta f_0(x_3),
    \end{equation}
    \begin{equation}
        \dot O_{4,t}(x_1, x_2, x_3, x_4) = 
        O(n^{-2}),
        \qquad
        O_{4,0}(x_1, x_2, x_3, x_4) =
        \nabla_\theta^T O_{3,0}(x_1, x_2, x_3) \nabla_\theta f_0(x_4).
    \end{equation}
    Let us expand all the quantities wrt $n^{-1}$:
    \begin{equation}
        O_{s,t}(x_{1:s}) =
        O_{s,t}^{(0)}(x_{1:s}) + n^{-1} O_{s,t}^{(1)}(x_{1:s}) + O(n^{-2}),
    \end{equation}
    where $O_{s,t}^{(k)}(x_{1:s}) = \Theta_{n\to\infty}(1)$.
    Then the system above transforms into the following:
    \begin{equation}
        \dot f_{t}^{(0)}(x_1) = 
        O_{2,t}^{(0)}(x_1, \vec x) (\vec y - f_t^{(0)}(\vec x)),
    \end{equation}
    \begin{equation}
        \dot f_{t}^{(1)}(x_1) = 
        O_{2,t}^{(1)}(x_1, \vec x) (\vec y - f_t^{(0)}(\vec x)) - O_{2,t}^{(0)}(x_1, \vec x) f_t^{(1)}(\vec x),
    \end{equation}
    \begin{equation}
        O_{2,t}^{(0)}(x_1, x_2) =
        \nabla_\theta^T f_0^{(0)}(x_1) \nabla_\theta f_0^{(0)}(x_2),
    \end{equation}
    \begin{equation}
        \dot O_{2,t}^{(1)}(x_1, x_2) = 
        O_{3,t}^{(1)}(x_1, x_2, \vec x) (\vec y - f_t^{(0)}(\vec x)),
    \end{equation}
    \begin{equation}
        \dot O_{3,t}^{(1)}(x_1, x_2, x_3) = 
        O_{4,t}^{(1)}(x_1, x_2, x_3, \vec x) (\vec y - f_t^{(0)}(\vec x)),
    \end{equation}
    \begin{equation}
        O_{4,t}^{(1)}(x_1, x_2, x_3, x_4) =
        \nabla_\theta^T O_{3,0}^{(0)}(x_1, x_2, x_3) \nabla_\theta f_0^{(0)}(x_4),
    \end{equation}
    where we have ignored the initial conditions for the time being.
    Integrating this system is straightforward:
    \begin{equation}
        f_{t}^{(0)}(\vec x) =
        \vec y + e^{-O_{2,0}^{(0)}(\vec x, \vec x) t} (f_0^{(0)}(\vec x) - \vec y),
    \end{equation}
    For brevity, let us introduce the following definition:
    \begin{equation}
        \Delta f_t^{(0)}(x) =
        e^{-O_{2,0}^{(0)}(x, \vec x) t} (f_0^{(0)}(\vec x) - \vec y).
    \end{equation}
    This gives:
    \begin{equation}
        O_{3,t}^{(1)}(x_1, x_2, x_3) = 
        O_{3,0}^{(1)}(x_1, x_2, x_3) - 
        \int_{0}^t O_{4,0}^{(1)}(x_1, x_2, x_3, \vec x) \Delta f_{t'}^{(0)}(\vec x) \, dt'.
    \end{equation}
    \begin{multline}
        O_{2,t}^{(1)}(x_1, x_2)
        = O_{2,0}^{(1)}(x_1, x_2) - \int_{0}^t O_{3,0}^{(1)}(x_1, x_2, \vec x) \Delta f_{t'}^{(0)}(\vec x) \, dt'
        +\\+
        \int_{0}^{t} \int_{0}^{t''} \Delta f_{t''}^{(0),T}(\vec x) O_{4,0}^{(1)}(x_1, x_2, \vec x, \vec x') \Delta f_{t'}^{(0)}(\vec x') \, dt' \, dt''.
    \end{multline}
    Let us elaborate the terms:
    \begin{equation}
        \int_{0}^t O_{3,0}^{(1)}(x_1, x_2, \vec x) \Delta f_{t'}^{(0)}(\vec x) \, dt'
        = O_{3,0}^{(1)}(x_1, x_2, \vec x) \left(O_{2,0}^{(0)}(\vec x,\vec x)\right)^{-1} \left(I - e^{-O_{2,0}^{(0)}(\vec x, \vec x) t}\right) (f_0^{(0)}(\vec x) - \vec y).
    \end{equation}
    \begin{multline}
        \int_{0}^{t} \int_{0}^{t''} \Delta f_{t''}^{(0),T}(\vec x) O_{4,0}^{(1)}(x_1, x_2, \vec x, \vec x') \Delta f_{t'}^{(0)}(\vec x') \, dt' \, dt''
        =\\= \int_{0}^{t} (f_0^{(0)}(\vec x) - \vec y)^T \left(I - e^{-O_{2,0}^{(0)}(\vec x, \vec x) t'}\right) \left(O_{2,0}^{(0)}(\vec x,\vec x)\right)^{-1} O_{4,0}^{(1)}(x_1, x_2, \vec x, \vec x') e^{-O_{2,0}^{(0)}(\vec x, \vec x) t'} (f_0^{(0)}(\vec x) - \vec y) \, dt'
        =\\= (f_0^{(0)}(\vec x) - \vec y)^T \left(O_{2,0}^{(0)}(\vec x,\vec x)\right)^{-1} O_{4,0}^{(1)}(x_1, x_2, \vec x, \vec x') \left(O_{2,0}^{(0)}(\vec x,\vec x)\right)^{-1} \left(I - e^{-O_{2,0}^{(0)}(\vec x, \vec x) t}\right) (f_0^{(0)}(\vec x) - \vec y)
        -\\-
        \int_{0}^{t} (f_0^{(0)}(\vec x) - \vec y)^T e^{-O_{2,0}^{(0)}(\vec x, \vec x) t'} \left(O_{2,0}^{(0)}(\vec x,\vec x)\right)^{-1} O_{4,0}^{(1)}(x_1, x_2, \vec x, \vec x') e^{-O_{2,0}^{(0)}(\vec x, \vec x) t'} (f_0^{(0)}(\vec x) - \vec y) \, dt'.
    \end{multline}
    Consider the eigenvalue-eigenvector decomposition of $O_{2,0}^{(0)}(\vec x, \vec x)$: $O_{2,0}^{(0)}(\vec x, \vec x) = \sum_{k=1}^m \lambda_1 v_k v_k^T$.
    This helps us integrating the last term:
    \begin{multline}
        \int_{0}^{t} (f_0^{(0)}(\vec x) - \vec y)^T e^{-O_{2,0}^{(0)}(\vec x, \vec x) t'} \left(O_{2,0}^{(0)}(\vec x,\vec x)\right)^{-1} O_{4,0}^{(1)}(x_1, x_2, \vec x, \vec x') e^{-O_{2,0}^{(0)}(\vec x, \vec x) t'} (f_0^{(0)}(\vec x) - \vec y) \, dt'
        =\\= \sum_{k,l=1}^m \int_{0}^{t} e^{-(\lambda_k+\lambda_l) t'} (f_0^{(0)}(\vec x) - \vec y)^T v_k v_k^T \left(O_{2,0}^{(0)}(\vec x,\vec x)\right)^{-1} O_{4,0}^{(1)}(x_1, x_2, \vec x, \vec x') v_l v_l^T (f_0^{(0)}(\vec x) - \vec y) \, dt'
        =\\= \sum_{k,l=1}^m \frac{1}{\lambda_k+\lambda_l} \left(1 - e^{-(\lambda_k+\lambda_l) t}\right) (f_0^{(0)}(\vec x) - \vec y)^T v_k v_k^T \left(O_{2,0}^{(0)}(\vec x,\vec x)\right)^{-1} O_{4,0}^{(1)}(x_1, x_2, \vec x, \vec x') v_l v_l^T (f_0^{(0)}(\vec x) - \vec y).
    \end{multline}

    Recall $\hat\Theta_t(x_1,x_2) = O_{2,t}(x_1,x_2) = O_{2,t}^{(0)}(x_1,x_2) + n^{-1} O_{2,t}^{(1)}(x_1,x_2) + O(n^{-2})$.
    The first term (the limit NTK) does not depend on $t$, $O_{2,t}^{(0)}(x_1,x_2) = O_{2,0}^{(0)}(x_1,x_2) = \Theta(x_1,x_2)$, while the second one (the correction) does.
    Note that computing the second term invokes $O_{4,0}^{(1)}$, the fourth-order tensor, therefore approaching it directly requires $O(m^4)$ memory.
    Integrating the above system further gives the first-order correction for the limit model $f_t^{(1)}$.

    As we shall see in \cref{sec:beyond}, the kernel $\Theta^{NTH}(x_1,x_2) = O_{2,0}^{(0)}(x_1,x_2) + n^{-1} \EE O_{2,\infty}^{(1)}(x_1,x_2)$ can be considered as a label-aware alternative to the usual NTK $\Theta(x_1,x_2) = O_{2,0}^{(0)}(x_1,x_2)$.
    Let us write its explicit definition and refer it later in \cref{sec:beyond}:
    \begin{multline}
        \Theta^{NTH}(x_1,x_2)
        = O_{2,0}^{(0)}(x_1,x_2) + n^{-1} \EE O_{2,\infty}^{(1)}(x_1,x_2)
        =\\= \Theta(x_1,x_2) + n^{-1} \EE\left[O_{2,0}^{(1)}(x_1,x_2)\right] - n^{-1} \EE\left[O_{3,0}^{(1)}(x_1, x_2, \vec x) \Theta^{-1}(\vec x,\vec x) f_0^{(0)}(\vec x)\right]
        +\\+ n^{-1} \vec y^T \Theta^{-1}(\vec x,\vec x) \EE\left[O_{4,0}^{(1)}(x_1, x_2, \vec x, \vec x)\right] \Theta^{-1}(\vec x,\vec x) \vec y
        +\\+ n^{-1} \EE\left[f_0^{(0),T}(\vec x) \Theta^{-1}(\vec x,\vec x) O_{4,0}^{(1)}(x_1, x_2, \vec x, \vec x) \Theta^{-1}(\vec x,\vec x) f_0^{(0)}(\vec x)\right]
        -\\- n^{-1} \sum_{k,l=1}^m \frac{1}{\lambda_k (\lambda_k+\lambda_l)} \vec y^T \vec v_k \vec v_k^T \EE\left[O_{4,0}^{(1)}(x_1, x_2, \vec x, \vec x)\right] \vec v_l \vec v_l^T \vec y
        -\\- n^{-1} \sum_{k,l=1}^m \frac{1}{\lambda_k (\lambda_k+\lambda_l)} \EE\left[f_0^{(0),T}(\vec x) \vec v_k \vec v_k^T O_{4,0}^{(1)}(x_1, x_2, \vec x, \vec x) \vec v_l \vec v_l^T f_0^{(0)}(\vec x)\right].
        \label{eq:lantk_nth}
    \end{multline}

    While the above result is valid under a conjecture, \cref{conj:C_asymptotics}, it can be proven rigorously, see \cite{huang2019dynamics}.



    \section{Computing the limit kernel}
    \label{sec:limit}

    It is not obvious how to compute the limit kernel $\Theta$ predicted by the theorems discussed in \cref{sec:convergence}.
    Fortunately, one can compute the limit kernel exactly for certain classes of models.

    \subsection{Fully-connected nets}
    \label{sec:limit_fc_nets}

    Consider an $L$-layer fully-connected network under NTK parameterization:
    \begin{equation}
        f(x) = h_L(x),
        \quad
        h_l(x) = \frac{1}{\sqrt{n_{l-1}}} W_l x_{l-1}(x),
        \quad
        x_{l-1}(x) = \phi(h_{l-1}(x)),
        \quad
        x_0(x) = x,
    \end{equation}
    where $W_l \in \RR^{n_l \times n_{l-1}}$ $\forall l \in [L]$.
    For simplicity, we assume $n_L=1$, i.e. the output is scalar.

    Since we already know (see \cref{sec:convergence}) that the kernel does not depend on $t$ under NTK parameterization, we consider the case $t=0$ only and omit the $t$-subscript.
    The empirical NTK is given by
    \begin{equation}
        \hat\Theta(x,x')
        = \nabla^T_\theta f(x;\theta) \nabla_\theta f(x';\theta)
        = \sum_{l=1}^L \tr\left(\nabla^T_{W_l} f(x;W_{1:L}) \nabla_{W_l} f(x;W_{1:L})\right).
    \end{equation}
    By chain rule,
    \begin{equation}
        \nabla_{W_l} f(x) 
        = \sum_{i=1}^{n_l} \partial_{h_l^i} f(x) \nabla_{W_l} h_l^i(x) 
        = \frac{1}{\sqrt{n_{l-1}}} \sum_{i=1}^{n_l} \sum_{j=1}^{n_{l-1}} \partial_{h_l^i} f(x) E_{ij} x_{l-1}^j(x) 
        = \frac{1}{\sqrt{n_{l-1}}} \nabla_{h_l} f(x) x_{l-1}^T(x).
    \end{equation}
    Therefore,
    \begin{equation}
        \hat\Theta(x,x')
        = \sum_{l=1}^L \tr\left(\nabla^T_{W_l} f(x) \nabla_{W_l} f(x)\right)
        = \sum_{l=1}^L \frac{1}{n_{l-1}} \left(\nabla^T_{h_l} f(x') \nabla_{h_l} f(x)\right) \times \left(x_{l-1}^T(x) x_{l-1}(x')\right).
    \end{equation}

    If $x_{l-1}$ had iid components with zero mean, $\frac{1}{n_{l-1}} x_{l-1}^T(x) x_{l-1}(x')$ would be an empirical covariance estimated with $n_{l-1}$ samples.
    In fact, when all weights are iid standard Gaussians, components of $h_{l-1}$ become iid Gaussian with zero mean as $n_{1:l-2} \to \infty$ sequentially.
    Hence their images under elementwise maps $\phi$ are also iid.
    
    Proof by induction.
    $h_1(x) = \frac{1}{\sqrt{n_0}} W_1 x$ has iid Gaussian components with zero mean and variance $q_1(x) = x^T x$.
    Suppose components of $h_{l-1}(x)$ become iid Gaussian with zero mean and $q_{l-1}(x)$ variance as $n_{1:l-2} \to \infty$ sequentially.
    Then $h_l(x) = \frac{1}{\sqrt{n_{l-1}}} W_l \phi(h_{l-1}(x))$ converges (in distribution) to a vector of Gaussians with zero mean and variance $q_l(x) = \EE_{z \sim \NN(0,q_{l-1}(x))} \phi^2(z)$ as $n_{1:l-1} \to \infty$ sequentially by the Central Limit Theorem (CLT).

    One can easily generalize the above proof to any finite set of inputs.
    In particular, $[h_l^i(x),h_l^i(x')]^T$ converges to a Gaussian with zero mean and covariance $\Sigma_l(x,x') = \begin{pmatrix} q_l(x) & q_l(x,x')\\q_l(x,x') & q_l(x') \end{pmatrix}$, where $q_l(x,x') = \EE_{[z,z']^T \sim \NN(0,\Sigma_{l-1}(x,x'))} \phi(z) \phi(z')$.
    Hence as $n_{1:l-2} \to \infty$ sequentially, $\frac{1}{n_{l-1}} x_{l-1}^T(x) x_{l-1}(x')$ converges to $q_l(x,x')$.

    Let $g_l(x) = \sqrt{n_l} \nabla_{h_l} f(x)$.
    Since
    \begin{equation}
        \nabla_{h_l^j} f(x) 
        = \sum_{i=1}^{n_{l+1}} \nabla_{h_{l+1}^i} f(x) \nabla_{h_l^j} h_{l+1}^i(x) 
        = \frac{1}{\sqrt{n_l}} \sum_{i=1}^{n_{l+1}} \nabla_{h_{l+1}^i} f(x) W_{l+1}^{ij} \phi'(h_l^j(x)),
    \end{equation}
    we have $g_l(x) = \frac{1}{\sqrt{n_{l+1}}} D_l(x) W_{l+1}^T g_{l+1}(x)$, where $D_l(x) = \diag(\phi'(h_l(x)))$.

    There are two obstacles that prevent us from following the same lines for $g_l$ as for $h_l$.
    First, $g_{l+1}$ depends on $D_{l+1}$ that depends on $h_{l+1}$ that depends on $W_{l+1}$.
    Since $W_{l+1}$ and $g_{l+1}$ are dependent, we cannot guarantee that components of $g_l$ become iid.
    Second, we know the distribution of $h_l$ as all the layers from the input side become infinitely wide sequentially, while induction for $g_l$ should be performed starting from the head.
    Nevertheless, it can be proven rigorously that ignoring these two obstacles still lead to a correct result \cite{yang2020tensor_ii}: $g_l(x)$ converges to a vector of iid Gaussians with zero mean and variance $\dot q_l(x) = \dot q_{l+1}(x) \EE_{z \sim \NN(0,q_l(x))} (\phi')^2(z)$ as $n_{1:L-1} \to \infty$.
    A similar result holds for a pair of inputs: $[g_l^i(x),g_l^i(x')]^T$ converges to a Gaussian with zero mean and covariance $\dot\Sigma_l(x,x') = \begin{pmatrix} \dot q_l(x) & \dot q_l(x,x')\\\dot q_l(x,x') & \dot q_l(x') \end{pmatrix}$, where $\dot q_l(x,x') = \dot q_{l+1}(x,x') \EE_{[z,z']^T \sim \NN(0,\Sigma_l(x,x'))} \phi'(z) \phi'(z')$.
    Hence $\nabla^T_{h_l} f(x') \nabla_{h_l} f(x) = \frac{1}{n_l} g_l^T(x') g_l(x)$ converges to $\dot q_l(x,x')$.

    Putting all together, $\hat\Theta(x,x')$ converges to $\Theta(x,x') = \sum_{l=1}^L \dot q_l(x,x') q_l(x,x')$, where
    \begin{equation}
        q_1(x,x') = x^T x',
        \quad
        q_l(x,x') = \EE_{[z,z']^T \sim \NN(0,\Sigma_{l-1}(x,x'))} \phi(z) \phi(z'),
    \end{equation}
    \begin{equation}
        \dot q_L(x,x') = 1, 
        \quad
        \dot q_l(x,x') = \dot q_{l+1}(x,x') \EE_{[z,z']^T \sim \NN(0,\Sigma_l(x,x'))} \phi'(z) \phi'(z'),
    \end{equation}
    and $\Sigma_l(x,x') = \begin{pmatrix} q_l(x,x) & q_l(x,x')\\q_l(x,x') & q_l(x',x') \end{pmatrix}$.
    Note that the Master theorem of \cite{yang2020tensor_ii} gives similar recurrent formulas for NTK of any architecture expressible by a tensor program and makes them mathematically rigorous.

    In fact, computing the NTK can be performed in a convenient sequential layer-wise manner, as implemented in Neural Tangents\footnote{\url{https://github.com/google/neural-tangents}} \cite{novak2019neural}.
    Define the NTK for the first $l$ layers as $\Theta_{:l}(x,x') = \sum_{l'=1}^l \tr(\nabla_{W_{l'}}^T h_l^i(x) \nabla_{W_{l'}} h_l^i(x'))$; in this case $\Theta_{:L}(x,x') = \Theta(x,x')$.
    Suppose $\Theta_{:l-1}(x,x')$ and $q_{l-1}(x,x')$ are already computed.
    Adding a nonlinearity and a linear layer with weights $W_l$ gives $q_l$ as listed above:
    \begin{equation}
        q_l(x,x') 
        = \EE_{[z,z']^T \sim \NN(0,\Sigma_{l-1}(x,x'))} \phi(z) \phi(z'),
        \quad 
        \text{where $\Sigma_{l-1}(x,x') 
        = \begin{pmatrix} q_{l-1}(x,x) & q_{l-1}(x,x')\\q_{l-1}(x,x') & q_{l-1}(x',x') \end{pmatrix}$.}
        \label{eq:q_iteration}
    \end{equation}
    However, according to a formula above, $\dot q_l$ is computed using $\dot q_{l+1}$, which requires a sequential layer-wise "forward pass" to compute all $q_l$ and a "backward pass" to compute $\dot q_l$.
    In fact, one forward pass is enough:
    \begin{multline}
        \Theta_{:l}(x,x') 
        = \sum_{l'=1}^l \tr(\nabla_{W_{l'}}^T h_l^i(x) \nabla_{W_{l'}} h_l^i(x')) 
        = q_l(x,x') + \sum_{l'=1}^{l-1} \tr(\nabla_{W_{l'}}^T h_l^i(x) \nabla_{W_{l'}} h_l^i(x')) 
        =\\= q_l(x,x') + \Theta_{:l-1}(x,x') \EE_{[z,z']^T \sim \NN(0,\Sigma_{l-1}(x,x'))} \phi'(z) \phi'(z').
        \label{eq:Theta_iteration}
    \end{multline}
    In Neural Tangents, each operation in a neural network is mapped to a corresponding kernel transform.

    \subsection{Convolutional nets}

    The same idea can be applied for convolutional nets as well.
    Consider 1d-convolutions for simplicity.
    In this case, we are dealing with 1d "images" with $d$ pixels: $x \in \RR^{n_0 \times d}$.
    Consider a network with $L$ convolutions under NTK parameterization and an average pooling at the end:
    \begin{equation}
        f^i = \frac{1}{d} \sum_{s=1}^d x_L^{i,s},
        \quad
        h_l^{i,s} = \frac{1}{\sqrt{n_{l-1}}} \sum_{j=1}^{n_{l-1}} \sum_{r \in \ker} W_l^{ijr} x_{l-1}^{j,s+r},
        \quad
        x_{l-1}^{i,s} = \phi(h_{l-1}^{i,s}),
        \quad
        x_0^{i,s} = x^{i,s},
    \end{equation}
    where we omitted the argument $x$ for brevity, $W_l \in \RR^{n_l \times n_{l-1} \times |\ker|}$ with $W_l^{ijr} \sim \NN(0,1)$ iid $\forall l \in [L]$, and $\ker$ denotes the convolution filter; e.g. $\ker = [-1,0,1]$ for a convolution of size $3$.
    For simplicity, we assume $n_L=1$, i.e. the output is scalar.

    As before, the empirical NTK is given as
    \begin{equation}
        \hat\Theta(x,x')
        = \nabla^T_\theta f(x;\theta) \nabla_\theta f(x';\theta)
        = \sum_{l=1}^L \sum_{i=1}^{n_l} \sum_{j=1}^{n_{l-1}} \sum_{r \in \ker} \partial_{W_l^{ijr}} f(x) \partial_{W_l^{ijr}} f(x').
    \end{equation}
    By chain rule,
    \begin{equation}
        \partial_{W_l^{ijr}} f
        = \sum_{s=1}^d \partial_{h_l^{i,s}} f \partial_{W_l^{ijr}} h_l^{i,s}
        = \frac{1}{\sqrt{n_{l-1}}} \sum_{s=1}^{d} \partial_{h_l^{i,s}} f x_{l-1}^{j,s+r}.
    \end{equation}
    Therefore,
    \begin{equation}
        \hat\Theta(x,x')
        = \sum_{l=1}^L \frac{1}{n_{l-1}} \sum_{i=1}^{n_l} \sum_{j=1}^{n_{l-1}} \sum_{r \in \ker} \sum_{s,s'=1}^{d} \partial_{h_l^{i,s}} f(x) \partial_{h_l^{i,s'}} f(x') x_{l-1}^{j,s+r}(x) x_{l-1}^{j,s'+r}(x').
    \end{equation}

    As for the fully-connected case, we are going to prove that $h^{i,s}$ become Gaussian with zero mean and variance given by a certain recurrent formula as $n_{1:l-1} \to \infty$ sequentially.
    However for the convolutional case, not all $h^{i,s}$ become independent: they become independent for different $i$'s but not for different $s$.

    Let us induct on $l$.
    $h_1^{i,s} = \frac{1}{\sqrt{n_0}} \sum_{j=1}^{n_0} \sum_{r \in \ker} W_1^{ijr} x^{j,s+r}$ are independent for any two different $i$'s.
    For a fixed $i$, $h_1^{i,\cdot}$ is a Gaussian vector with zero mean and covariance $q_1^{s,s'} = \frac{1}{n_0} \sum_{j=1}^{n_0} \sum_{r \in \ker} x^{j,s+r} x^{j,s'+r}$.
    Suppose $h_{l-1}^{i,s}$ becomes Gaussian with zero mean, independent for any two different $i$'s, and $q_{l-1}^{s,s'}$ is its covariance as $n_{1:l-2} \to \infty$ sequentially.
    Then $h_l^{i,s} = \frac{1}{\sqrt{n_{l-1}}} \sum_{j=1}^{n_{l-1}} \sum_{r \in \ker} W_l^{ijr} x_{l-1}^{j,s+r}$ converges (in distribution) to a random variable with similar properties but with covariance $q_l^{s,s'} = \EE_{z \sim \NN(0,q_{l-1})} \sum_{r \in \ker} \phi(z^{s+r}) \phi(z^{s'+r})$ as $n_{1:l-1} \to \infty$ sequentially by the Central Limit Theorem (CLT).

    One can easily generalize the above proof to any finite set of inputs.
    In particular, $[h_l^{i,\cdot}(x),h_l^{i,\cdot}(x')]^T \in \RR^{2d}$ converges to a Gaussian with zero mean and covariance $\Sigma_l(x,x') = \begin{pmatrix} q_l(x) & q_l(x,x')\\q_l(x,x') & q_l(x') \end{pmatrix} \in \RR^{2d \times 2d}$, where $q_l^{s,s'}(x,x') = \EE_{[z,z']^T \sim \NN(0,\Sigma_{l-1}(x,x'))} \sum_{r \in \ker} \phi(z^{s+r}) \phi(z^{\prime,s'+r})$.
    Hence as $n_{1:l-2} \to \infty$ sequentially, $\frac{1}{n_{l-1}} \sum_{j=1}^{n_{l-1}} \sum_{r \in \ker} x_{l-1}^{j,s+r}(x) x_{l-1}^{j,s'+r}(x')$ converges to $q_l^{s,s'}(x,x')$.

    Let $g_l^{j,p} = \sqrt{n_l} \nabla_{h_l^{j,p}} f$.
    Since
    \begin{multline}
        \partial_{h_l^{j,p}} f 
        = \sum_{i=1}^{n_{l+1}} \sum_{s=1}^d \partial_{h_{l+1}^{i,s}} f \partial_{h_l^{j,p}} h_{l+1}^{i,s} 
        =\\= \frac{1}{\sqrt{n_l}} \sum_{i=1}^{n_{l+1}} \sum_{s=1}^d \partial_{h_{l+1}^{i,s}} f \sum_{r \in \ker} W_{l+1}^{ijr} 1_{s+r=p} \phi'(h_l^{j,p})
        = \frac{1}{\sqrt{n_l}} \sum_{i=1}^{n_{l+1}} \sum_{r \in \ker} \partial_{h_{l+1}^{i,p-r}} f W_{l+1}^{ijr} \phi'(h_l^{j,p}),
    \end{multline}
    $\partial_{h_L^{j,p}} f = \frac{1}{d} \phi'(h_L^{j,p})$, and $n_L=1$, we have
    \begin{equation}
        g_L^{j,p}
        = \frac{1}{d} \phi'(h_L^{j,p}),
        \quad
        g_l^{j,p}
        = \frac{1}{\sqrt{n_{l+1}}} \sum_{i=1}^{n_{l+1}} \sum_{r \in \ker} g_{l+1}^{i,p-r} W_{l+1}^{ijr} \phi'(h_l^{j,p}).
    \end{equation}

    With the same correctness remark as for convolutional nets, it is possible to show that $g_l^{j,p}$ become independent for different $j$'s and $g_l^{j,\cdot}$ become Gaussian with covariance $\dot q_l^{p,p'}$ as $n_{1:L-1} \to \infty$.
    Covariance is given by the following recurrence: $\dot q_L^{p,p'} = \frac{1}{d^2} \EE_{[z,z']^T \sim \NN(0,\Sigma_L)} \phi'(z^p) \phi'(z^{p'})$, $\dot q_l^{p,p'} = \EE_{z \sim \NN(0,q_l)} \phi'(z^{p}) \phi'(z^{p'}) \sum_{r \in \ker} \dot q_{l+1}^{p-r,p'-r}$.

    A similar result holds for a pair of inputs: $[g_l^{i,\cdot}(x),g_l^{i,\cdot}(x')]^T \in \RR^{2d}$ converges to a Gaussian with zero mean and covariance $\dot\Sigma_l(x,x') = \begin{pmatrix} \dot q_l(x) & \dot q_l(x,x')\\\dot q_l(x,x') & \dot q_l(x') \end{pmatrix} \in \RR^{2d \times 2d}$, where $\dot q_l^{s,s'}(x,x') = \EE_{[z,z']^T \sim \NN(0,\Sigma_l(x,x'))} \phi'(z^{s}) \phi'(z^{\prime,s'}) \sum_{r \in \ker} \dot q_{l+1}^{s-r,s'-r}(x,x')$.
    Hence
    \begin{equation}
        \sum_{i=1}^{n_l} \partial_{h_l^{i,s}} f(x) \partial_{h_l^{i,s'}} f(x') 
        = \frac{1}{n_l} \sum_{i=1}^{n_l} g_l^{i,s}(x) g_l^{i,s'}(x')
        \to \dot q_l^{s,s'}(x,x').
    \end{equation}

    Putting all together, $\hat\Theta(x,x')$ converges to $\Theta(x,x') = \sum_{l=1}^L \sum_{s,s'=1}^d \dot q_l^{s,s'}(x,x') q_l^{s,s'}(x,x')$, where
    \begin{equation}
        q_1^{s,s'}(x,x') = \frac{1}{n_0} \sum_{j=1}^{n_0} \sum_{r \in \ker} x^{j,s+r} x^{\prime,j,s'+r},
    \end{equation}
    \begin{equation}
        q_l^{s,s'}(x,x') = \EE_{[z,z']^T \sim \NN(0,\Sigma_{l-1}(x,x'))} \sum_{r \in \ker} \phi(z^{s+r}) \phi(z^{\prime,s'+r}),
    \end{equation}
    \begin{equation}
        \dot q_L^{s,s'}(x,x') = \frac{1}{d^2} \EE_{[z,z']^T \sim \NN(0,\Sigma_L(x,x'))} \phi'(z^s) \phi'(z^{\prime,s'}), 
    \end{equation}
    \begin{equation}
        \dot q_l^{s,s'}(x,x') = \EE_{[z,z']^T \sim \NN(0,\Sigma_l(x,x'))} \phi'(z^{s}) \phi'(z^{\prime,s'}) \sum_{r \in \ker} \dot q_{l+1}^{s-r,s'-r}(x,x'),
    \end{equation}
    and $\Sigma_l(x,x') = \begin{pmatrix} q_l(x,x) & q_l(x,x')\\q_l(x,x') & q_l(x',x') \end{pmatrix}$.

    Same as for fully-connected nets, computing the NTK can be performed in a convenient sequential layer-wise manner.
    Define the empirical NTK for the first $l$ layers as
    \begin{equation}
        \hat\Theta_{:l}^{s,s'}(x,x') 
        = \sum_{l'=1}^l \sum_{i=1}^{n_{l'}} \sum_{j=1}^{n_{l'-1}} \sum_{r \in \ker} \partial_{W_{l'}^{ijr}} h_l^{1,s}(x) \partial_{W_{l'}^{ijr}} h_l^{1,s'}(x');
    \end{equation}
    in this case, by chain rule,
    \begin{multline}
        \hat\Theta(x,x')
        = \sum_{l=1}^L \sum_{i=1}^{n_l} \sum_{j=1}^{n_{l-1}} \sum_{r \in \ker} \partial_{W_l^{ijr}} f(x) \partial_{W_l^{ijr}} f(x')
        =\\= \sum_{l=1}^L \sum_{i=1}^{n_l} \sum_{j=1}^{n_{l-1}} \sum_{s,s'=1}^d \sum_{r \in \ker} \partial_{W_l^{ijr}} h_L^{1,s}(x) \partial_{W_l^{ijr}} h_L^{1,s'}(x') \partial_{h_L^{1,s}} f(x) \partial_{h_L^{1,s'}} f(x')
        =\\= \frac{1}{d^2} \sum_{s,s'=1}^d \phi'(h_L^{1,s}(x)) \phi'(h_L^{1,s'}(x')) \hat\Theta_{:L}^{s,s'}(x,x'),
    \end{multline}
    and therefore,
    \begin{equation}
        \Theta(x,x')
        = \frac{1}{d^2} \sum_{s,s'=1}^d \dot q_L^{s,s'}(x,x') \Theta_{:L}^{s,s'}(x,x').
    \end{equation}

    Suppose $\hat\Theta_{:l-1}(x,x')$ and $q_{l-1}(x,x')$ are already computed.
    Adding a nonlinearity and a convolutional layer with weights $W_l$ gives $q_l$ as listed above:
    \begin{equation}
        q_l^{s,s'}(x,x') 
        = \EE_{[z,z'] \sim \NN(0,\Sigma_{l-1}(x,x'))} \sum_{r \in \ker} \phi(z^{s+r}) \phi(z^{\prime,s'+r}),
        \label{eq:q_iteration_conv}
    \end{equation}
    where $\Sigma_{l-1}(x,x') = \begin{pmatrix} q_{l-1}(x,x) & q_{l-1}(x,x')\\q_{l-1}(x,x') & q_{l-1}(x',x') \end{pmatrix}$.
    We can compute $\hat\Theta_{:L}$ in a single forward pass using the following recurrence:
    \begin{multline}
        \hat\Theta_{:l}^{s,s'}(x,x') 
        = \sum_{l'=1}^l \sum_{i=1}^{n_{l'}} \sum_{j=1}^{n_{l'-1}} \sum_{\tilde r \in \ker} \partial_{W_{l'}^{ij\tilde r}} h_l^{1,s}(x) \partial_{W_{l'}^{ij\tilde r}} h_l^{1,s'}(x')
        =\\= \frac{1}{n_{l-1}} \sum_{j=1}^{n_{l-1}} \sum_{\tilde r \in \ker} x^{j,s+\tilde r}(x) x^{j,s'+\tilde r}(x')
        +\\+ \sum_{l'=1}^{l-1} \sum_{i=1}^{n_{l'}} \sum_{j=1}^{n_{l'-1}} \sum_{\tilde r \in \ker} \sum_{k,k'=1}^{n_{l-1}} \sum_{p,p'=1}^d \partial_{W_{l'}^{ij\tilde r}} h_{l-1}^{k,p}(x) \partial_{W_{l'}^{ij\tilde r}} h_{l-1}^{k',p'}(x') \partial_{h_{l-1}^{k,p}} h_l^{1,s}(x) \partial_{h_{l-1}^{k',p'}} h_l^{1,s'}(x')
        =\\= \frac{1}{n_{l-1}} \sum_{j=1}^{n_{l-1}} \sum_{\tilde r \in \ker} x^{j,s+\tilde r}(x) x^{j,s'+\tilde r}(x')
        +\\+ \frac{1}{n_{l-1}} \sum_{l'=1}^{l-1} \sum_{i=1}^{n_{l'}} \sum_{j=1}^{n_{l'-1}} \sum_{\tilde r,r,r' \in \ker} \sum_{k,k'=1}^{n_{l-1}} \partial_{W_{l'}^{ij\tilde r}} h_{l-1}^{k,s+r}(x) \partial_{W_{l'}^{ij\tilde r}} h_{l-1}^{k',s'+r'}(x') \times\\\times W_l^{1kr} \phi'(h_{l-1}^{k,s+r}(x)) W_l^{1k'r'} \phi'(h_{l-1}^{k',s'+r'}(x')).
        \label{eq:Theta_iteration_conv}
    \end{multline}
    A limit then gives
    \begin{equation}
        \Theta_{:l}^{s,s'}(x,x') 
        = q_l^{s,s'}(x,x') + \sum_{r,r' \in \ker} \Theta_{:l-1}^{s+r,s'+r'}(x,x') \EE_{[z,z']^T \sim \NN(0,\Sigma_{l-1}(x,x'))} \phi'(z^{s+r}) \phi'(z^{\prime,s'+r'}),
    \end{equation}
    which resembles the corresponding result for fully-connected nets when $\ker = [0]$.

    \subsection{Computing the expectations}

    The only obstacle that prevents explicit computation here is expectations over $[z,z']^T \sim \NN(0,\Sigma_l(x,x'))$.
    Fortunately, these expectations can be computed analytically for certain $\phi$: in particular, for ReLU and the error function.

    We cover only the case of ReLU here as it is more widely used in practice.
    Let us omit the $l$-subscript and the arguments $(x,x')$ for brevity: $\Sigma = \begin{pmatrix} q_{11} & q_{12}\\q_{12} & q_{22} \end{pmatrix}$, and we are interested in $\EE_{[u,v]^T \sim \NN(0,\Sigma)} [u]_+ [v]_+$ and $\EE_{[u,v]^T \sim \NN(0,\Sigma)} 1_{u>0} 1_{v>0}$.

    Following \cite{arora2019exact}, we start with assuming $q_{11} = q_{22} = 1$ and $q_{12} = \lambda$; $\Sigma \geq 0$ implies $|\lambda| \leq 1$.
    Then
    \begin{multline}
        \EE_{[u,v]^T \sim \NN(0,\Sigma)} [u]_+ [v]_+
        = \EE_{[u,\tilde v]^T \sim \NN(0,I)} [u]_+ \left[\lambda u + \sqrt{1-\lambda^2} \tilde v\right]_+
        =\\= \EE_{u \sim \NN(0,1)} \left([u]_+ \int_{-\frac{\lambda}{\sqrt{1-\lambda^2}} u}^\infty \left(\lambda u + \sqrt{1-\lambda^2} \tilde v\right) \frac{1}{\sqrt{2\pi}} e^{-\tilde v^2/2} \, d\tilde v \right)
        =\\= \EE_{u \sim \NN(0,1)} \left(
            [u]_+ \left(
                \lambda u \frac{1}{2} \left(1 - \erf\left(-\frac{\lambda}{\sqrt{2-2\lambda^2}} u\right)\right) + \sqrt{\frac{1-\lambda^2}{2\pi}} e^{-\frac{\lambda^2}{2-2\lambda^2} u^2}
            \right)
        \right)
        =\\= \int_0^\infty u \left(\lambda u \frac{1}{2} \left(1 - \erf\left(-\frac{\lambda}{\sqrt{2-2\lambda^2}} u\right)\right) + \sqrt{\frac{1-\lambda^2}{2\pi}} e^{-\frac{\lambda^2}{2-2\lambda^2} u^2}\right) \frac{1}{\sqrt{2\pi}} e^{-u^2/2} \, du
        =\\= \frac{\lambda}{4} + \int_0^\infty u \left(\lambda u \frac{1}{2} \erf\left(\frac{\lambda}{\sqrt{2-2\lambda^2}} u\right) + \sqrt{\frac{1-\lambda^2}{2\pi}} e^{-\frac{\lambda^2}{2-2\lambda^2} u^2}\right) \frac{1}{\sqrt{2\pi}} e^{-u^2/2} \, du
        =\\= \frac{\lambda}{4} + \frac{\lambda}{2} A + \sqrt{\frac{1-\lambda^2}{2\pi}} B.
    \end{multline}
    \begin{multline}
        A
        = \int_0^\infty u^2 \erf\left(\frac{\lambda}{\sqrt{2-2\lambda^2}} u\right) \frac{1}{\sqrt{2\pi}} e^{-u^2/2} \, du
        = -\int_0^\infty u \erf\left(\frac{\lambda}{\sqrt{2-2\lambda^2}} u\right) \frac{1}{\sqrt{2\pi}} \, d\left(e^{-u^2/2}\right)
        =\\= \int_0^\infty \left(\erf\left(\frac{\lambda}{\sqrt{2-2\lambda^2}} u\right) + u \frac{\lambda}{\sqrt{2-2\lambda^2}} \frac{2}{\sqrt{\pi}} e^{-\frac{\lambda^2}{2-2\lambda^2} u^2} \right) \frac{1}{\sqrt{2\pi}} e^{-u^2/2} \, du
        = C + \frac{\lambda}{\sqrt{2-2\lambda^2}} \frac{2}{\sqrt{\pi}} B.
    \end{multline}
    \begin{equation}
        C
        = \int_0^\infty \erf\left(\frac{\lambda}{\sqrt{2-2\lambda^2}} u\right) \frac{1}{\sqrt{2\pi}} e^{-u^2/2} \, du
        = \frac{1}{\pi} \arctan\left(\frac{\lambda}{\sqrt{1-\lambda^2}}\right)
        = \frac{1}{\pi} \arcsin\lambda.
    \end{equation}
    \begin{equation}
        B
        = \int_0^\infty u e^{-\frac{\lambda^2}{2-2\lambda^2} u^2} \frac{1}{\sqrt{2\pi}} e^{-u^2/2} \, du
        =  \frac{1}{\sqrt{2\pi}}\int_0^\infty u e^{-\frac{1}{2-2\lambda^2} u^2} \, du
        = \frac{1-\lambda^2}{\sqrt{2\pi}}.
    \end{equation}
    Putting all together,
    \begin{multline}
        \EE_{[u,v]^T \sim \NN(0,\Sigma)} [u]_+ [v]_+
        = \frac{\lambda}{4} + \frac{\lambda}{2} A + \sqrt{\frac{1-\lambda^2}{2\pi}} B
        = \frac{\lambda}{4} + \frac{\lambda}{2} C + \frac{\lambda^2}{\sqrt{1-\lambda^2}} \frac{1}{\sqrt{2\pi}} B + \sqrt{\frac{1-\lambda^2}{2\pi}} B
        =\\= \frac{\lambda}{4} + \frac{\lambda}{2} C + \frac{1}{\sqrt{1-\lambda^2}} \frac{1}{\sqrt{2\pi}} B
        = \frac{\lambda}{4} + \frac{\lambda}{2\pi} \arcsin\lambda + \frac{\sqrt{1-\lambda^2}}{2\pi}
        =\\= \frac{\lambda\left(\frac{\pi}{2} + \arcsin\lambda\right) + \sqrt{1-\lambda^2}}{2\pi}
        = \frac{\lambda\left(\pi - \arccos\lambda\right) + \sqrt{1-\lambda^2}}{2\pi}.
    \end{multline}
    And for the second quantity,
    \begin{multline}
        \EE_{[u,v]^T \sim \NN(0,\Sigma)} 1_{u>0} 1_{v>0}
        = \EE_{[u,\tilde v]^T \sim \NN(0,I)} 1_{u>0} 1_{\lambda u + \sqrt{1-\lambda^2} \tilde v > 0}
        =\\= \EE_{u \sim \NN(0,1)} \left(1_{u>0} \int_{-\frac{\lambda}{\sqrt{1-\lambda^2}} u}^\infty \frac{1}{\sqrt{2\pi}} e^{-\tilde v^2/2} \, d\tilde v \right)
        =\\= \EE_{u \sim \NN(0,1)} \left(
            1_{u>0} \frac{1}{2} \left(1 - \erf\left(-\frac{\lambda}{\sqrt{2-2\lambda^2}} u\right)\right)
        \right)
        =\\= \int_0^\infty \frac{1}{2} \left(1 - \erf\left(-\frac{\lambda}{\sqrt{2-2\lambda^2}} u\right)\right) \frac{1}{\sqrt{2\pi}} e^{-u^2/2} \, du
        =\\= \frac{1}{4} + \int_0^\infty \frac{1}{2} \erf\left(\frac{\lambda}{\sqrt{2-2\lambda^2}} u\right) \frac{1}{\sqrt{2\pi}} e^{-u^2/2} \, du
        =\\= \frac{1}{4} + \frac{1}{2} C
        = \frac{\frac{\pi}{2} + \arcsin\lambda}{2\pi}
        = \frac{\pi - \arccos\lambda}{2\pi}.
    \end{multline}

    A general positive semi-definite matrix $\Sigma$ can be expressed as $\Sigma = D \Lambda D$, where $\Lambda = \begin{pmatrix} 1 & \lambda\\\lambda & 1 \end{pmatrix}$, $D = \begin{pmatrix} \sqrt{q_{11}} & 0\\0 & \sqrt{q_{22}} \end{pmatrix}$, and $\lambda = \frac{q_{12}}{\sqrt{q_{11} q_{22}}}$.
    Then, using homogeneity of ReLU,
    \begin{multline}
        \EE_{[u,v]^T \sim \NN(0,\Sigma)} [u]_+ [v]_+
        = \EE_{[u,v]^T \sim \NN(0,D \Lambda D)} [u]_+ [v]_+
        = \EE_{[u,v]^T \sim \NN(0,\Lambda)} [\sqrt{q_{11}} u]_+ [\sqrt{q_{22}} v]_+
        =\\= \sqrt{q_{11} q_{22}} \EE_{[u,v]^T \sim \NN(0,\Lambda)} [u]_+ [v]_+
        = \sqrt{q_{11} q_{22}} \frac{\lambda\left(\pi - \arccos\left(\frac{q_{12}}{\sqrt{q_{11} q_{22}}}\right)\right) + \sqrt{1-\frac{q_{12}^2}{q_{11} q_{22}}}}{2\pi}
        =\\= \frac{\lambda \sqrt{q_{11} q_{22}} \left(\pi - \arccos\left(\frac{q_{12}}{\sqrt{q_{11} q_{22}}}\right)\right) + \sqrt{q_{11} q_{22} - q_{12}^2}}{2\pi}.
    \end{multline}
    \begin{multline}
        \EE_{[u,v]^T \sim \NN(0,\Sigma)} 1_{u>0} 1_{v>0}
        = \EE_{[u,v]^T \sim \NN(0,D \Lambda D)} 1_{u>0} 1_{v>0}
        =\\= \EE_{[u,v]^T \sim \NN(0,\Lambda)} 1_{u>0} 1_{v>0}
        = \frac{\pi - \arccos\left(\frac{q_{12}}{\sqrt{q_{11} q_{22}}}\right)}{2\pi}.
    \end{multline}

    Similar explicit computations are available for convolutional networks \cite{arora2019exact}, as well as for generic tensor programs, as long as the nonlinearities used belong to a certain list (which includes e.g. ReLU and the error function, see \cite{novak2019neural} for a concrete implementation and \cite{yang2020tensor_ii} for generic recurrent formulas in terms of expectations).
    
    However, a typical convolutional network also uses max poolings and other nonlinear maps for which explicit formulas for expectations are not available at the moment.
    In this case, one can rely on a finite-width Monte-Carlo estimate for $\Theta(x,x')$, i.e. $\hat\Theta^{(M)}(x,x') = \frac{1}{M} \sum_{k=1}^M \hat\Theta(x,x')$, where $M$ is a number of independent initializations and $\hat\Theta(x,x')$ is an empirical kernel for width $n$.
    According to convergence results, $\hat\Theta^{(M)}(x,x') \to \Theta(x,x')$ as $n \to \infty$ $\forall M \geq 1$.
    Also, $\hat\Theta^{(M)}(x,x') \to \EE \hat\Theta(x,x')$ as $M \to \infty$ $\forall n \to \infty$.
    Unfortunately, one cannot guarantee that $\EE \hat\Theta(x,x') = \Theta(x,x')$; therefore, $\hat\Theta^{(M)}(x,x')$ can be a biased estimate.
    However, according to experiments of \cite{novak2019neural}, discrepancy between $\hat\Theta^{(M)}$ and $\Theta$ decreases as $M$ grows for any finite $n$.
    This means that the main component of this discrepancy is not bias but variance decreased by adding more Monte-Carlo samples.

    We also have to note that \cite{arora2019exact} reports significant accuracy drops on a CNN of width $n=512$ when using a single-sample Monte-Carlo estimate for the NTK instead of the exact limit NTK.
    However, they haven't provided any results for $M > 1$, therefore, this accuracy drop could be caused by large variance of $\hat\Theta$.

    \subsection{NTK for attention layers}

    A neural tangent kernel is typically considered for architectures for which analytical computation is available, i.e. for fully-connected and convolutional ReLU nets, see \cref{sec:limit}.
    One of the necessary conditions for exact computations to be possible is the fact that the output of each individual pre-activation neuron becomes a Gaussian process in the limit of large width.
    This allows one to apply Master theorem (\cref{thm:master_theorem}), and express the NTK as an expectation over certain Gaussian variables.

    However, there exist layers which does not enjoy Gaussian behavior even in the limit of large width.
    Attention layer is one of the examples:
    \begin{equation}
        f(x)
        = \mathrm{Softmax}\left(G(x)\right) V(x),
        \qquad
        G(x)
        = \frac{1}{\sqrt{n}} Q^T(x) K(x),
    \end{equation}
    where we define queries $Q(x) = x W_Q$, keys $K(x) = x W_K$, and values $V(x) = x W_V$.
    Dimensions of the corresponding matrices are: $W_Q \in \RR^{n_0 \times n}$, $W_K \in \RR^{n_0 \times n}$, and $W_V \in \RR^{n_0 \times n_H}$, and $x \in \RR^{d \times n_0}$.

    If $W_Q$ and $W_K$ are independent with iid zero mean unit variance entries then $G_{\alpha\beta}(x) = n^{-1/2} \sum_{i=1}^n \sum_{j,k=1}^{n_0} x_{\alpha,j} x_{\beta,k} W_Q^{ji} W_K^{ki}$ converges by CLT to a Gaussian variable.
    The resulting limit matrix is therefore $d \times d$ matrix with (non-degenerate) Gaussian entries.
    Since $d$ stays fixed as $n \to \infty$, we cannot apply any limit theorem to reason about the distribution of $f_i(x)$ for some $i \in [n_H]$.

    \cite{hron2020infinite} consider a multi-head attention layer and show that it does enjoy Gaussian process behavior as width and number of heads go to infinity simultaneously:
    \begin{equation}
        f(x)
        = [f^1(x), \ldots, f^n(x)] W_O,
        \qquad
        f_i(x)
        = \mathrm{Softmax}\left(G_i(x)\right) V_i(x),
        \qquad
        G_i(x)
        = \frac{1}{\sqrt{n}} Q_i^T(x) K_i(x),
    \end{equation}
    where $W_O \in \RR^{n_H n \times n_H}$ and all $Q_i$, $K_i$, and $V_i$ are iid for different $i \in [n]$.
    To gain some intuition about the result of \cite{hron2020infinite}, consider $n_H=1$, i.e. outputs of all individual heads are scalars and the final output is also a scalar.
    In this case, $f(x)$ is a product of a vector with $n$ iid entries and a matrix with iid $\NN(0,n^{-1})$ entries.
    This product tends to a Gaussian as $n \to \infty$ by CLT.
    Considering a set of inputs gives a random Gaussian vector similar to the fully-connected case, see \cref{sec:limit_fc_nets}.

    \cite{hron2020infinite} gives exact formulas for covariances $q(x,x')$ and the kernel $\Theta(x,x')$; they are implemented as layers in NeuralTangents \cite{novak2019neural}.

    \section{Computational aspects}
    \label{sec:computations}

    \subsection{Inference optimizations}

    Suppose one is able to compute (or approximate) the limit kernel, $\Theta(x,x')$, on any pair of points $(x,x')$.
    The result of kernel regression at convergence ($t \to \infty$) in the limit of inifinite width is then given by (see Eq.~(\ref{eq:lin_solution_square_loss})):
    \begin{equation}
        f_\infty(x) 
        = f_0(x) - \Theta(x, \vec x) \Theta^{-1}(\vec x, \vec x) (f_0(\vec x) - \vec y).
        \label{eq:inf_wide_solution_square_loss}
    \end{equation}
    where $\Theta(\vec x, \vec x) \in \RR^{m \times m}$ and $\Theta(x, \vec x) \in \RR^{1 \times m}$.
    For multi-class problems, $f(x) \in \RR^k$, where $k$ is the number of classes, and the kernel evaluated at two points becomes a $k \times k$ matrix:
    \begin{equation}
        \hat\Theta_{jj'}(x,x')
        = \nabla_\theta^T f^j(x) \nabla_\theta f^{j'}(x').
    \end{equation}

    Define a Gram matrix as $\hat\Theta_{ik+j,i'k+j'}(\vec x, \vec x) = \hat\Theta_{jj'}(x_i,x_{i'})$ and its limit counterpart $\Theta(\vec x, \vec x) \in \RR^{mk \times mk}$ accordingly; similarly for $\Theta(x, \vec x) \in \RR^{k \times mk}$.
    If one defines $f_0^{ik+j}(\vec x) = f_0^j(x_i)$, the corresponding solution takes the same form as Eq.~(\ref{eq:inf_wide_solution_square_loss}).

    Evaluating this quantity naively requires storing and inverting the kernel Gram matrix $\Theta(\vec x, \vec x) \in \RR^{mk \times mk}$.
    Storing it requires $O(m^2 k^2)$ memory, while inverting it takes $O(m^3 k^3)$ time, making such a naive approach computationally infeasible for datasets with $m k \gtrsim 10^4$ (nevertheless, for small datasets, the naive approach for computing the NTK estimator (\ref{eq:inf_wide_solution_square_loss}) is feasible and may provide advantage over traditional SGD training, see \cite{arora2019harnessing}).

    Let us start with discussing two important optimizations implemented in Neural Tangents \cite{novak2019neural}.
    Note that as discussed in \cref{sec:limit}, for a fully-connected net (and, in fact, for any tensor program, see \cite{yang2019tensor_i}) preactivations of different neurons on a given layer become iid as width goes to infinity.
    This implies $\Theta_{jj'}(x,x') = \Theta_{11}(x,x') 1_{j=j'}$.
    Therefore the kernel Gram matrix has a block structure: $\Theta(\vec x, \vec x) = \Theta|_{k=1}(\vec x, \vec x) \otimes I_{k \times k}$.
    This reduces memory footprint to $O(m^2)$ and the time requirement to $O(m^3)$.

    The second optimization deals with convolutional networks.
    Note that computing $\Theta(x,x')$ requires computing all intermediate covariances $q_l(x,x')$.
    These covariances were scalars for fully-connected nets since different neurons of a given layer became iid as width went to infinity.
    However, for an image with $d$ pixels, different pixels of a given layer are dependent since their preactivations are computed using same weight matrices.
    That's why for convolutional nets, one has to construct intermediate covariance matrices of size $d \times d$; storing and computing them for each pair of points requires $O(m^2 d^2)$ memory and time, even surpassing the time required for Gram matrix inversion when $d^2 > m$ (this happens e.g. for CIFAR10 for which $d = 32 \times 32 = 1024$, $m = 50 000$, $k = 10$).
    However, as was noted e.g. in \cite{xiao2018dynamical}, if no pooling is used in the network, it suffices to compute and store $d$ independent $m \times m$ blocks of this covariance matrix, boiling down to $O(m^2 d)$ time requirement which is usually not greater than $O(m^3)$ time required for inversion.

    So far, the main computational bottleneck was the time required for inverting the kernel Gram matrix.
    This problem is not specific for NTK; it appears for any regularized kernel regression problem:
    \begin{equation}
        \hat f_\lambda
        = \argmin_{f \in \HH} \sum_{j=1}^m \ell(y_j, f(x_j)) + \lambda \| f \|_\HH^2.
        \label{eq:kernel_regression}
    \end{equation}
    Here $\HH$ is a Hilbert space of functions of the form $f(x) = \Phi^T(x) \theta$; the corresponding scalar product is $\langle \Phi^T(x) \theta, \Phi^T(x) \theta' \rangle = \theta^T \theta'$.
    Hence $\| f \|_\HH^2 = \langle f, f \rangle = \|\theta\|_2^2$ for $f(x) = \Phi^T(x) \theta$.

    Problem~(\ref{eq:kernel_regression}) has an associated kernel, which we denote with the same letter as NTK: $\Theta(x,x') = \Phi^T(x) \Phi(x')$.
    Due to the representer theorem \cite{kimeldorf1970correspondence}, any solution of Problem~(\ref{eq:kernel_regression}) has the form $f(x) = \sum_{j=1}^m \alpha_j \Theta(x,x_j)$.
    
    For now, consider quadratic loss: $\ell(y,z) = \frac{1}{2} \| y - z \|_2^2$.
    The problem above becomes:
    \begin{equation}
        \vec\alpha
        = \argmin_{\vec\alpha \in \RR^m} \frac{1}{2} \sum_{j=1}^m \left( \sum_{j'=1}^m \alpha_{j'} \Theta(x_j,x_{j'}) - y_j \right)^2 + \lambda \left\| \sum_{j=1}^m \alpha_j \Phi(x_j) \right\|_2^2.
    \end{equation}
    This problem is convex, therefore any critical point of the corresponding functional is a solution:
    \begin{equation}
        (\Theta(\vec x, \vec x) + \lambda I) \vec\alpha
        = \vec y.
    \end{equation}
    As long as $\Theta(\vec x, \vec x) + \lambda I$ is invertible, the solution is $\vec\alpha = (\Theta(\vec x, \vec x) + \lambda I)^{-1} \vec y$.
    Putting $\lambda = 0$, we recover expected Eq.(\ref{eq:inf_wide_solution_square_loss}) (since $\EE f_0(x) = 0$).

    While the represeneter theorem guarantees that it suffices to look for solutions only of the form $f(x) = \sum_{j=1}^m \alpha_j \Theta(x,x_j)$ instead of inspecting the whole $\HH$, we, following \cite{meanti2020kernel}, consider further contracting the search space by sampling $m'$ points $(\tilde x_1, \ldots, \tilde x_{m'})$ uniformly out of $m$ and looking for solutions of the form $f(x) = \sum_{j=1}^{m'} \tilde\alpha_j \Theta(x,\tilde x_j)$.
    This is known as Nystr\"om approximation.
    The minimization problem then becomes:
    \begin{equation}
        \vec{\tilde\alpha}
        = \argmin_{\vec{\tilde\alpha} \in \RR^{m'}} \frac{1}{2} \sum_{j=1}^m \left( \sum_{j'=1}^{m'} \tilde\alpha_{j'} \Theta(x_j,\tilde x_{j'}) - y_j \right)^2 + \lambda \left\| \sum_{j=1}^{m'} \tilde\alpha_j \Phi(\tilde x_j) \right\|_2^2.
    \end{equation}
    This problem is again convex and its critical points satisfy the following:
    \begin{equation}
        \left(\Theta\left(\vec{\tilde x}, \vec x\right) \Theta\left(\vec x, \vec{\tilde x}\right) + \lambda \Theta\left(\vec{\tilde x}, \vec{\tilde x}\right)\right) \vec{\tilde \alpha}
        = \Theta\left(\vec{\tilde x}, \vec x\right) \vec y.
        \label{eq:critical_points_nystrom}
    \end{equation}

    Computing the kernel-kernel product takes $O(m {m'}^2)$ time and solving the above system directly takes $O({m'}^3)$ time.
    The space requirement can be put to $O({m'}^2)$ as the "rectangular Gram matrix" can be computed in $m' \times m'$ blocks.

    Conjugate gradient methods are iterative methods designed for approximately solving linear systems of the form $A \vec z = \vec b$ without explicitly inverting the matrix $A$.
    The main operation used by these methods on each iteration is a matrix-vector product.
    In our case, the matrix-vector product requires $O(mm' + {m'}^2)$ time; note that it allows one to avoid computing the kernel-kernel product explicitly, by computing two matrix-vector product instead, costing $O(mm')$ time each.
    
    Putting all together, solving system~(\ref{eq:critical_points_nystrom}) with $s$ iterations of a conjugate gradient method requires $O(s(mm' + {m'}^2))$ time and $O({m'}^2)$ space.
    Based on certain theoretical results, \cite{meanti2020kernel} suggest taking $m' = O(\sqrt{m})$ and $s = O(\log m)$.
    The resulting $O(m \sqrt{m} \log m)$ time and $O(m)$ space allows for applying their method to datasets of size up to $m \sim 10^6$ (the size of ImageNet).
    \cite{meanti2020kernel} also discuss several optimizations aiming for improving GPU-efficiency of the method.
    While their method is publicly available as an open-source library\footnote{\url{https://github.com/FalkonML/falkon}}, we are not aware of any of its applications to NTK.

    \subsection{Computing the empirical kernel}

    All the previous discussion of the current section assumed that the kernel, $\Theta$, can be efficiently computed.
    This is the case for certain models for which analytic computations are available.
    Indeed, for $L$-layer fully-connected nets, the limit Gram matrix $\Theta(\vec x, \vec x)$ can be computed in $O(m^2 L)$ time while storing it requires $O(m^2)$ space, see Eqs. (\ref{eq:q_iteration}) and (\ref{eq:Theta_iteration}).
    For more complex models, e.g. for those including max-poolings, closed-form analytic expressions for the limit kernel are not currently available.
    However, the empirical kernel, $\hat\Theta$, can always be computed explicitly and is close to $\Theta$ for sufficiently large width (see convergence theorems in \cref{sec:convergence}).
    For this reason, we are looking for ways to compute $\hat\Theta$ efficiently.

    In order to simplify the illustration, we will discuss only time requirements in the sequel.
    Recall the empirical kernel is a product of two jacobians: $\hat\Theta_{jj'}(x,x') = \nabla^T_\theta f^j(x) \nabla_\theta f^{j'}(x')$.
    Therefore the time cost for computing the kernel consists of the time required to compute the jacobian and the time required for jacobian contraction.

    Denote $[FP]$ the cost of a single forward pass for our network; a single backward pass has approximately the same cost.
    Then computing a jacobian for a given point $x$ takes $O(k [FP])$ time.
    Contracting two jacobians for fixed $j$ and $j'$ takes $O(N)$ time, where $N$ is the total number of parameters: $\theta \in \RR^N$.
    Putting all together, computing the full $mk \times mk$ Gram matrix takes $O(m k [FP] + m^2 k^2 N)$ time.

    \cite{novakfast} propose a method for computing the NTK-vector product.
    It can be directly embedded into the method of \cite{meanti2020kernel} using conjugate gradients, or used for computing the kernel explicitly by applying it to columns of the $k \times k$ identify matrix.

    Their method boils down to casting a matrix-vector product where the matrix is the empirical NTK to a vector-jacobian product followed by a jacobian-vector product: $\sum_{j'=1}^k \hat\Theta_{jj'}(x,x') v_{j'} = \nabla^T_\theta f^j(x) \sum_{j'=1}^k \nabla_\theta f^{j'}(x') v_{j'}$.
    Both matrix-vector products can be computed in $O([FP])$ time.
    Therefore this method allows to compute the full $mk \times mk$ Gram matrix in $O(m^2 k [FP])$ time, which improves over the jacobian contraction method as long as $[FP] < C k N$ for a certain constant $C$.
    Memory requirements that we do not show here are, in fact, same for both methods, see \cite{novakfast}.

    \cite{novakfast} also propose another optimization exploiting certain stucture of the function $f$: e.g. weights of a fully-connected net are aligned sequentially, while weights of a convolutional layer are aranged in blocks.
    We do not discuss it in the present survey.
    Both optimizations are publicly available as JAX \cite{jax2018github} function transformations.\footnote{\url{https://github.com/iclr2022anon/fast_finite_width_ntk}}.

    \section{Applications}

    \subsection{A kernel method}


    \subsubsection{Supervised learning on small datasets}

    The NTK is a kernel, therefore it can be used in any kernel method itself, i.e. kernel ridge regression or kernel SVM.
    However, computing the kernel Gram matrix on a dataset of size $m$ requires $O(m^2)$ time, which is infeasible for large datasets.
    One can either rely on certain approximations, e.g. Nystr\"om approximation, see \cref{sec:computations}, or restrict oneself to small datasets.

    One possible advantage of kernel methods over neural nets is lower variance.
    Indeed, the only variance of a kernel method is induced by sampling the dataset, while a neural network has several more sources of variance; e.g. initialization randomness and batch sampling.
    It is likely that this difference in variances is especially important when the dataset is small.

    The other advantage of kernel methods is having smaller number of hyperparamaters compared to neural nets.
    This makes kernel methods useful as robust baseline methods that may outperform large neural nets in a situation when there is no budget for careful hyperparamater tuning.
    As an illustration, \cite{arora2019harnessing} demonstrated that kernel regression with 14-layer CNTK consistently outperforms ResNet-34 trained with standard hyperparameters on a random subset of CIFAR-10 with $\leq 640$ samples.

    \subsubsection{Neural architecture search using NTK conditional number}

    There are other setups where computing the Gram matrix on a small dataset is sufficient.
    For example, \cite{chen2021neural} proposes a condition number of the NTK Gram matrix as a proxy-measure of a given architecture performance; this proxy-measure is then used to guide neural architecture search (NAS).
    In this case, we do not need the Gram matrix itself but only the condition number, which motivates computing the matrix on a small subset of examples.
    While the condition number on a random subset Gram matrix provides only a random estimate, possibly noisy and biased, of a true condition number, the way we use it does not require exact estimates.
    Indeed, a performance measure in NAS algorithms is mainly used to cut-off pathologic, low-performing models from a population, rather than finding the best one.
    Therefore any measure that correlates positively with performance suffices.

    The use of condition number as a proxy-measure of performance relies on two hypotheses: (1) performance correlates with trainability, and (2) trainability correlates with NTK condition number.
    The first hypothesis is mainly motivated by a natural implication "bad trainability implies low performance".
    To motivate the second hypothesis, let us consider kernel ridge regression trained with usual discrete-time gradient descent:
    \begin{equation}
        f_{t+1}(\vec x)
        = f_t(\vec x) + \eta \Theta(\vec x, \vec x) (\vec y - f_t(\vec x)),
    \end{equation}
    where now $t$ is a discrete time-step and $\eta$ is a learning rate.

    Consider eigenvalue decomposition of the kernel: $\Theta(\vec x, \vec x) = \sum_{k=1}^m \lambda_k \vec v_k \vec v_k^T$, where $\lambda_1 \geq \ldots \geq \lambda_m \geq 0$, and $(\vec v_k)_{k=1}^m$ forms an orthonormal basis.
    Let us decompose our model's predictions as $f_t(\vec x) = \sum_{k=1}^m u_{t,k} \vec v_k$.
    Then the dynamics above decomposes as
    \begin{equation}
        u_{t+1,k}
        = u_{t,k} + \eta \lambda_k (\vec y^T \vec v_k - u_{t,k}).
    \end{equation}
    This gives
    \begin{equation}
        u_{t+1,k} - \vec y^T \vec v_k
        = (1 - \eta \lambda_k) (u_{t,k} - \vec y^T \vec v_k),
    \end{equation}
    and the solution is therefore
    \begin{equation}
        u_{t,k}
        = \vec y^T \vec v_k + (1 - \eta \lambda_k)^t (u_{0,k} - \vec y^T \vec v_k).
    \end{equation}

    The dynamics above converges as $t \to \infty$ for any $u_{0,k}$ if and only if $\eta < 2 / \lambda_k$.
    Since this should hold for all $k \in [m]$ and the maximal $\lambda$ is $\lambda_1$, we need to have $\eta < 2 / \lambda_1$.
    Therefore the $m$-th principal component converges at rate $\eta \lambda_m < 2 \lambda_m / \lambda_1$.
    $\kappa = \lambda_m / \lambda_1$ is our condition number.
    We see that small condition number implies low trainability and thus, by the first hypothesis, low performance.

    Using a combination of two proxy-measures, the condition number and the number of linear regions (we do not discuss it here), \cite{chen2021neural} constructed a NAS method that provided state-of-the-art performance on NAS-Bench-201 \cite{dong2020bench}, while using much smaller time compared to most of the other methods.
    \cite{chen2021neural} tested their method on CIFAR10 and ImageNet as well.
    In both cases, their method demonstrated competetive performance while using orders of magnitude less time.

    \subsubsection{Matrix completion and image impainting}

    In some cases, posing the problem as kernel regression allows for certain optimizations.
    In particular, \cite{radhakrishnan2021simple} proposed approaching the problem of matrix completion by minimizing the following loss:
    \begin{equation}
        \LL(\theta)
        = \sum_{(i,j) \in S} (Y_{ij} - \tr(f(Z;\theta) M^{(ij)}))^2,
    \end{equation}
    where $S \subset [k] \times [d]$ is a set of coordinates of known entries of the target matrix $Y \in \RR^{k \times d}$, $M^{(ij)} \in \RR^{k \times d}$ has $1$ at position $(i,j)$ and $0$ elsewhere, $f(\cdot;\theta)$ is a neural network with parameters $\theta$, $n_0$ inputs and $k$ outputs, and $Z \in \RR^{n_0 \times d}$ is an a-priori given matrix.
    The model $f$ is applied to each column of $Z$ seperately, therefore $f(Z;\theta)$ is $k \times d$ matrix.

    The above setup can be treated as a usual $l_2$ regression problem on a dataset $(Y_{ij}, M^{(ij)})_{(i,j) \in S}$.
    The corresponding empirical NTK is defined as $\hat K(M^{(ij)}, M^{(i'j')}) = \nabla^T_\theta \tr(f(Z;\theta) M^{(ij)}) \nabla_\theta \tr(f(Z;\theta) M^{(i'j')})$.
    Naturally, it does not depend on target matrix entries $Y$, and since there is only a finite set of possible inputs $M^{(ij)}$ (namely, $k d$), the resulting $kd \times kd$ Gram matrix will be the same for all possible matrix completion problems of a given target matrix dimensions.
    In other words, one can precompute the Gram matrix once and use it to all possible matrix completion problems of given dimensions.
    In contrast, original neural network formulation would require training a new network for each dataset $(Y_{ij}, M^{(ij)})_{(i,j) \in S}$.

    When $f(\cdot;\theta)$ is given by a fully-connected network with $L$ layers, \cite{radhakrishnan2021simple} provide a closed-form formula for its limit NTK: $K(M^{(ij)}, M^{(i'j')}) = \kappa_L\left(z_{\cdot,j}^T z_{\cdot,j'}\right) 1_{i=i'}$, where $\kappa_L$ is given by a certain recurrent relation.
    As we see, according to this kernel, elements of different rows of $Y$ are orthogonal (does not effect each other), while similarity of elements of the same row is given by a scalar product of the corresponding columns of $Z$.
    Therefore columns of $Z$ encodes a-priori similarities between columns of $Y$.

    The matrix $Z$ is called a feature-prior matrix.
    The ideal feature-prior matrix would be the target matrix $Y$ itself.
    Since one does not have access to it, \cite{radhakrishnan2021simple} suggest using the output $\hat Y$ of a separate matrix completion method instead.
    The resulting joint method performs better than the backbone one on popular collaborative filtering and virtual drug screening datasets.

    Image impainting can be viewed as a special case of matrix completion.
    Apart from using the same Gram matrix for all problems of a given size, image impainting with convolutional networks allows for one more optimization.
    
    When $f$ is a convolutional network, we pose the problem a bit differently to above.
    Suppose $f$ has $n_0$ input channels, $1$ output channel, and it maps an image to an image of the same size.
    Suppose $Z \in \RR^{n_0 \times 2^p \times 2^q}$ and it is treated as a $2^p \times 2^q$ image with $n_0$ channels.
    This in contrast to the previous considerations, where $Z$ was a matrix with columns treated as different inputs to a vector-valued model.
    Similar to the above, $Y \in \RR^{2^p \times 2^q}$ is a target image, and $M^{(ij)}$ of the same size has $1$ at $(i,j)$ and zero elsewhere.

    Note that $f$ applied to the "image" $Z$ has $2^p \times 2^q$ output and therefore its NTK $\Theta$ is a $2^p \times 2^q \times 2^p \times 2^q$ tensor.
    Suppose $f$ has no downsampling or upsampling layers.
    \cite{radhakrishnan2021simple} provides exact formula for the corresponding limit NTK in terms of the limit NTK of the model $f$ in this case: $K(M^{(ij)}, M^{(i'j')}) = \Theta(Z,Z)_{i,j,i',j'}$.

    Now suppose $f$ has $s$ downsampling and $s$ upsampling layers.
    Computing the Gram matrix for its NTK requires $O(2^{2p+2q})$ memory and $O(L 2^{2p+2q})$ time, where $L$ is the number of convolutions in $f$.
    It is already prohibitive for moderate-size images, i.e. when $p, q \approx 10$.
    \cite{radhakrishnan2021simple} propose a way to reconstruct the $2^p \times 2^q \times 2^p \times 2^q$ Gram matrix from a smaller Gram matrix of size $2^{2s+p+q}$.
    Moreover, this smaller Gram matrix requires computing the "usual" Gram matrices only for images of size $2^{s+1} \times 2^{s+1}$ which requires only $O(L 2^{4s})$ time.

    \subsubsection{Approximate integration with application to federated learning}

    Even in the case when the NTK Gram matrix can be computed and stored, the exact solution (\ref{eq:inf_wide_solution_square_loss}) requires inverting the kernel Gram matrix, which costs $O(m^3)$ when performed naively.
    Fortunately, mixing continuous-time and discrete-time formulations allows one to avoid computing the inverse explicitly.
    
    Denote $H_{t,ij} = \hat\Theta_t(x_i,x_j)$, $Z_{t,ik} = \partial_{\theta_i} f(x_k;\theta)$, and $u_{t,k} = f_t(x_k)$.
    Note that $H_t = Z_t^T Z_t$.
    Discrete-time weight evolution with learning rate $\eta$ is given by
    \begin{equation}
        \theta_{t+1}
        = \theta_t + \eta Z_t (\vec y - \vec u_t).
    \end{equation}
    Recall that assuming stationary kernel $H_t = H_0$ is equivalent to assuming stationary jacobian $Z_t = Z_0$.
    With this assumption, the dynamics above is solved as
    \begin{equation}
        \theta_t
        = \theta_0 + \eta Z_0 \sum_{s=0}^{t-1} (\vec y - \vec u_s).
    \end{equation}
    Recall that integrating continuous-time gradient descent dynamics under assumption $H_t = H_0$ gives
    \begin{equation}
        \vec u_s
        = \vec y + e^{-\eta s H_0} (\vec u_0 - \vec y).
    \end{equation}
    Combining the two latter equations, we get the weights at any time-step $t$:
    \begin{equation}
        \theta_t
        = \theta_0 + \eta Z_0 \sum_{s=0}^{t-1} e^{-\eta s H_0} (\vec y - \vec u_0).
    \end{equation}
    The continuous analogue of the above evolution is obtained by replacing the sum with an integral:
    \begin{equation}
        \theta_t
        = \theta_0 + \eta Z_0 \int_0^t e^{-\eta s H_0} (\vec y - \vec u_0) \, ds
        = \theta_0 + Z_0 H_0^{-1} \left(I - e^{-\eta s H_0}\right) (\vec y - \vec u_0).
    \end{equation}
    Here we get the inverse, as expected.

    Note that in this approach we do not assume that the network to be infinitely wide, we just assume it to be linear in its weights. 
    This allows us to reason in terms of the network weight vector $\theta_t$ instead of reasoning in terms of some abstract feature space associated to the kernel.
    This aspect gives us one additional advantage: we can integrate the dynamics up to some time $t_1$ and, since we know the weights $\theta_{t_1}$, compute $Z_{t_1}$ and $H_{t_1}$.
    We can then proceed integration with these updated matrices.
    This method lies in between the usual gradient descent training and kernel gradient descent with constant kernel.
    The latter never updates the kernel, while the former updates the kernel at each timestep.
    In contrast, the method we discuss updates the kernel only at given timesteps.
    
    The approach under discussion requires computing and storing $Z$ of size $N \times m$, which is an obvious disadvantage.
    As a remedy, \cite{yue2021neural} propose splitting the job of computing $Z$ between several workers.
    A server joins the parts together, integrates the dynamics up to some timestep $t$, and sends $\theta_t$ to all of the workers, starting a new iteration.
    Tuning the timesteps of kernel updates may help balancing load between the server and the workers.
    The data used to compute $Z$ is never stored on the server, making this approach promising for federated learning.
    However, since the server may attempt reconstructing the data from $Z$, one has to ensure each worker's privacy cannot be compromised; see \cite{yue2021neural} for further details.

    \subsection{Pathology analysis}


    \begin{figure}[t]
        \centering
        \subfigure[Ground truth]{\includegraphics[width=0.18\textwidth]{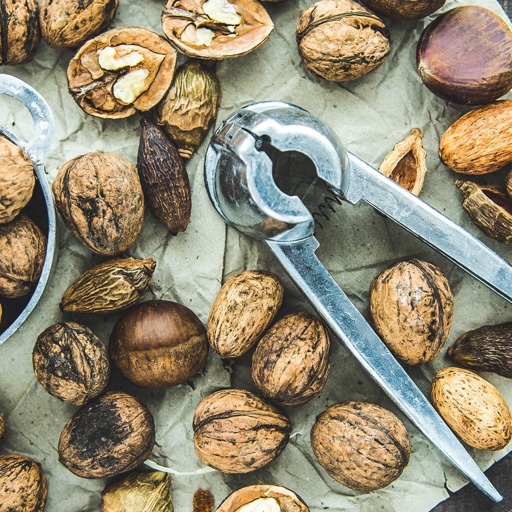}}
        \subfigure[No mapping]{\includegraphics[width=0.18\textwidth]{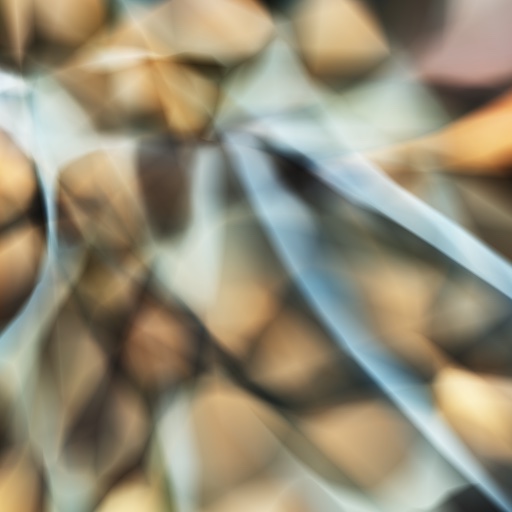}}
        \subfigure[Basic]{\includegraphics[width=0.18\textwidth]{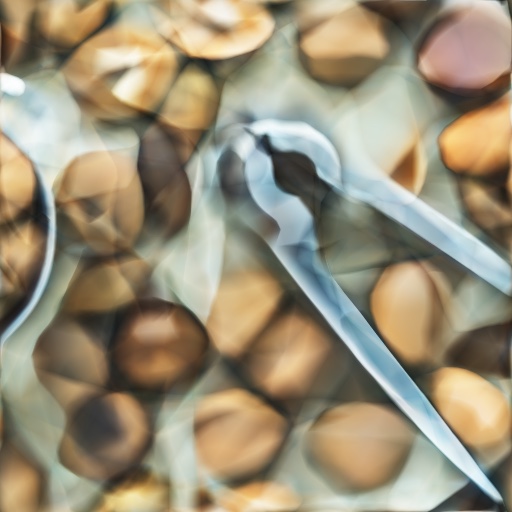}}
        \subfigure[Positional enc.]{\includegraphics[width=0.18\textwidth]{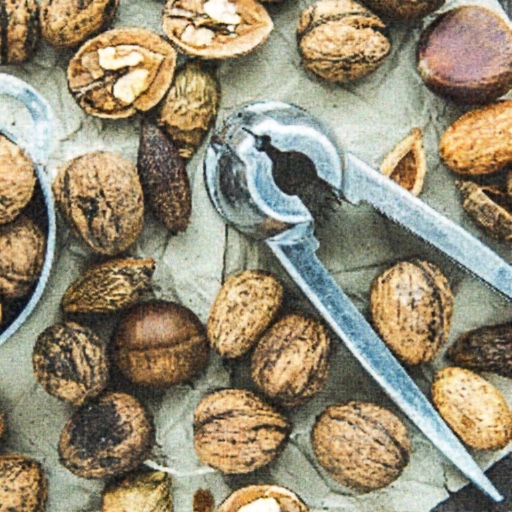}}
        \subfigure[Gaussian]{\includegraphics[width=0.18\textwidth]{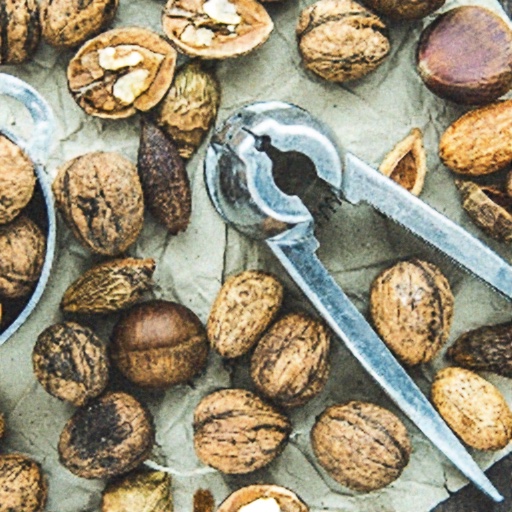}}
        \caption{Images are borrowed from \cite{tancik2020fourier}.}
        \label{fig:image_regression}
    \end{figure}
    \begin{figure}
        \centering
        \subfigure{\includegraphics[width=0.2\textwidth]{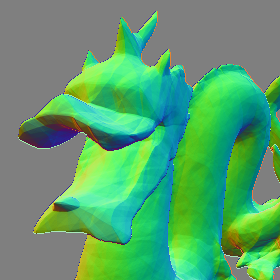}}
        \subfigure{\includegraphics[width=0.2\textwidth]{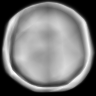}}
        \subfigure{\includegraphics[width=0.2\textwidth]{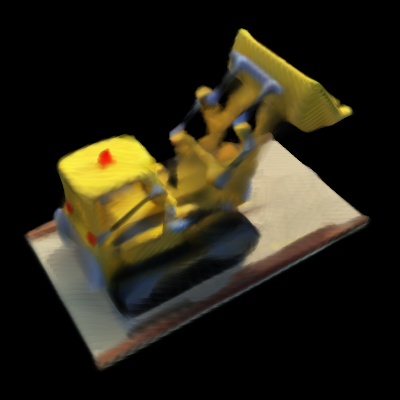}}
        \\
        \subfigure[3D shape regression]{\includegraphics[width=0.2\textwidth]{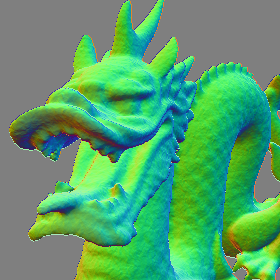}}
        \subfigure[MRI reconstruction]{\includegraphics[width=0.2\textwidth]{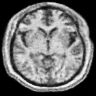}}
        \subfigure[Inverse rendering]{\includegraphics[width=0.2\textwidth]{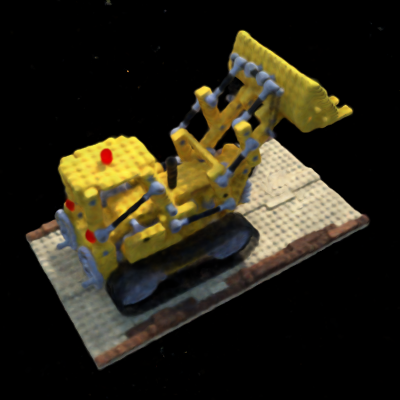}}
        \caption{Images are borrowed from \cite{tancik2020fourier}.}
        \label{fig:low_dim_regression}
    \end{figure}

    While the empirical NTK of a neural network is not the same as its limit NTK, they may have certain properties in common.
    In particular, certain issues of a finite-width network may reflect in certain issues of its limit NTK, and fixing these issues in the limit NTK may result in fixing them in a finite-width net.

    As an example where this approach is proven to work, consider image regression.
    In this task, input samples are image coordinates, $x \in [0,1]^d$ for $d=2$, and targets are pixel colors; we assume grey-scale images with $y \in [0,1]$.
    The task is therefore to regress the full image given a set of pixels.

    Let us consider applying a fully-connected network for this task.
    As we have already observed in \cref{sec:limit_fc_nets}, the limit NTK $\Theta(x,x')$ of a fully-connected network depends only on $x^T x$, $x^{\prime,T} x'$, and $x^T x'$.
    All of these terms are rotation-invariant, hence the kernel itself is rotation-invariant.
    However, none of this terms is translation-invariant, hence the kernel cannot be translation-invariant (otherwise, it has to be constant).
    Therefore it is quite unlikely that the empirical kernel will be invariant to translations.

    On the other hand, both translation and rotation invariance are desirable for a kernel used for image regression.
    Indeed, this means that applying these transformations to the train set of pixels results in the same image as without them, up to translation and rotation.
    In order to achieve this property, one may start working on translationaly invariant embeddings of image coordinates. 
    The simplest non-trivial embedding of this kind is $z(x) = [\cos(2\pi x), \sin(2\pi x)]^T$, where $\cos$ and $\sin$ are applied elementwise.
    Following \cite{tancik2020fourier}, we shall refer it as "basic".
    Comparing (b) and (c) of Figure~\ref{fig:image_regression}, this indeed results in better perceived quality.

    However the regressed image is still blurry: see Figure~\ref{fig:image_regression} (c).
    As we shall see shortly, NTK kernel regression learns low-frequency components of the image before its high-frequency ones.
    If we assume that the same property holds for the corresponding finite-width net then achieving sharp images may be impossible for a given number of gradient steps.

    Recall the training dynamics of a kernel regression with kernel $\Theta$ trained to minimize square loss on a training dataset $(\vec x, \vec y)$:
    \begin{equation}
        \dot f_t(\vec x)
        = \Theta(\vec x, \vec x) (\vec y - f_t(\vec x)).
    \end{equation}
    $\Theta$ is a kernel, therefore its Gram matrix is positive-semidefinite.
    Consider its eigenvalue decomposition: $\Theta(\vec x, \vec x) = \sum_{k=1}^m \lambda_k \vec v_k \vec v_k^T$, where $\lambda_1 \geq \ldots \geq \lambda_m \geq 0$, and $(\vec v_k)_{k=1}^m$ forms an orthonormal basis.

    Let us decompose our model's predictions as $f_t(\vec x) = \sum_{k=1}^m u_{t,k} \vec v_k$.
    Then the dynamics above decomposes as
    \begin{equation}
        u_{t,k}
        = \lambda_k (\vec v_k^T \vec y - u_{t,k}),
    \end{equation}
    which solves as
    \begin{equation}
        u_{t,k} 
        = \vec v_k^T \vec y - e^{-\lambda_k t} (\vec v_k^T \vec y - u_{0,k}).
    \end{equation}

    As one clearly sees, time required to learn the $k$-th principal component of the target is inversely proportional to its strength $\lambda_k$.
    In other words, strong components are learned before weak ones.
    
    The question is: what are the eigenvectors of the NTK Gram matrix?
    It is hard to answer this question in general since a Gram matrix depends on the dataset.
    However, for a kernel, there is an analogue of eigenvalue decomposition called Mercer's representation.

    Let $X$ be a compact metric space and let $\mu$ be a sigma-additive measure on $X$ with $\supp \mu = X$.
    Suppose $K: \; X \times X \to \RR$ is continuous, symmetric, and satisfies $\int_X \int_X K(x,x') f(x) f(x') \, d\mu(x) \, d\mu(x') < \infty$ $\forall f \in L^2_\mu(X)$.
    Define Gram-Schmidt operator $T_K: L^2_\mu(X) \to L^2_\mu(X)$ as $T_K[f](x) = \int_X K(x,x') \, d\mu(x')$.
    Then the above operator admits an eigenvalue decomposition with eigenfunctions $(\psi_k)_{k=1}^\infty$ and corresponding eigenvalues $(\lambda_k)_{k=1}^\infty$, and the set of eigenfunctions forms an orthonormal basis in $L^2_\mu(X)$.
    The Mercer's representation is the corresponding decomposition of the kernel:
    \begin{equation}
        K(x,x')
        = \sum_{k=1}^\infty \lambda_k \psi_k(x) \psi_k(x').
    \end{equation}
    The series converges uniformly in $X \times X$.

    From the above, we have $\int_X \int_X K(x,x') \psi_k(x) \psi_k(x') \, d\mu(x) \, d\mu(x') = \lambda_k$ $\forall k \geq 1$.
    Hence if $\vec x = (x_k)_{k=1}^m$ and $\vec x' = (x'_k)_{k=1}^m$ are sampled iid from $\mu$ then 
    \begin{multline}
        \frac{1}{m^2} \psi_k^T(\vec x) K(\vec x, \vec x') \psi_k(\vec x')
        =\\= \frac{1}{m^2} \sum_{i,j=1}^m K(x_i, x'_j) \psi(x_i) \psi(x'_k) 
        \to \int_X \int_X K(x,x') \psi_k(x) \psi_k(x') \, d\mu(x) \, d\mu(x') 
        = \lambda_k
    \end{multline}
    a.s. as $m \to \infty$ by the Law of Large Numbers (LLN).
    Note that considering $\psi^T(\vec x) K(\vec x, \vec x) \psi(\vec x)$ instead of $\psi^T(\vec x) K(\vec x, \vec x') \psi(\vec x')$ may result in a different limit because the diagonal of $K$ is now calculated on two dependent arguments.
    Nevertheless, there are only $m$ elements on the diagonal, which results in $O(m^{-1})$ error vanishing in the limit.
    Hence
    \begin{equation}
        \frac{1}{m^2} \psi_k^T(\vec x) K(\vec x, \vec x) \psi_k(\vec x)
        \to \lambda_k
    \end{equation}
    a.s. as $m \to \infty$.
    In other words, given $\vec x$ sampled iid from $\mu$, $(\psi_k(\vec x))_{k=1}^m$ are approximately the eigenvectors of $K(\vec x, \vec x)$ with eigenvalues $(m^2 \lambda_k)_{k=1}^m$.

    Recall that, as was noted above, the limit NTK of a fully-connected net $\Theta(z,z')$ depends only on $z^T z'$, $\|z\|_2$, and $\|z'\|_2$.
    Recall also that we have decided to embed inputs with $z(x) = [\cos(2\pi x), \sin(2\pi x)]^T$.
    This embedding maps $[0,1]^d$ on a $d$-dimensional torus that lies inside a $2d-1$-dimensional sphere.
    In this case, our $\Theta(x,x') = \Theta(z(x),z(x'))$ depends only on $z^T(x) z(x')$.

    Kernels with this property are called zonal.
    Any zonal kernel $K: S^{p-1} \times S^{p-1} \to \RR$ admits the following Mercer's decomposition with respect to the uniform measure on $S^{p-1}$:
    \begin{equation}
        K(z^T z')
        = \sum_{k=0}^\infty \lambda_k \sum_{j=1}^{N(p,k)} Y_{k,j}(z) Y_{k,j}(z'),
    \end{equation}
    where $N(p,k)$ are so-called Gegenbauer polynomials and $Y_{k,j}$ are spherical harmonics.
    For $p=2$, this decomposition gets a simpler form:
    \begin{equation}
        K(z^T z')
        = \frac{1}{4\pi^2} + \frac{1}{\pi^2} \sum_{k=1}^\infty \lambda_k \cos(k \arccos(z^T z')).
        \label{eq:mercer_zonal_2d}
    \end{equation}

    As we see, large $k$'s correspond to high-frequency harmonics, while small $k$'s correspond to low-frequency ones.
    A recent result of \cite{chen2020deep} states that the NTK of a fully-connected net with inputs lying on $S^{p-1}$ has eigenvalues decaying as a power-law: $\lambda_k \sim k^{-p}$ as $k \to \infty$; see also \cite{geifman2020similarity} for an earlier result for shallow nets and \cite{bietti2019inductive} for an even earlier result for bias-free shallow nets.
    This means that learning the $k$-th harmonic of the input image requires $O(k^p)$ time.
    Hence for a finite amount of training steps, high-frequency components remain not learned, which results in blurry images similar to Figure~\ref{fig:image_regression} (c).

    The possible remedy would be increasing $\lambda_k$ for large $k$.
    But how to achieve it?
    We illustrate the solution proposed in \cite{tancik2020fourier} in the following.
    
    Consider the case $d=1$ for simplicity.
    In this case, the embedding map $z(x) = [\cos(2\pi x), \sin(2\pi x)]^T$ traverses a circle.
    Consider a modified embedding $\tilde z(x) = [\cos(2\pi b x), \sin(2\pi b x)]^T$ instead, where $b \in \mathbb{N}$ is a tunable parameter.
    The corresponding kernel is then given as
    \begin{multline}
        K(\tilde z^T \tilde z')
        = \frac{1}{4\pi^2} + \frac{1}{\pi^2} \sum_{k=1}^\infty \lambda_k \cos(k \arccos(\tilde z^T \tilde z'))
        =\\= \frac{1}{4\pi^2} + \frac{1}{\pi^2} \sum_{k=1}^\infty \lambda_k \cos(4\pi k b (x-x'))
        = \frac{1}{4\pi^2} + \frac{1}{\pi^2} \sum_{k=1}^\infty \lambda_k \cos(k b \arccos(z^T z')),
    \end{multline}
    which means that $\lambda_k$ becomes the $kb$-th eigenvalue in the original embedding space.
    If $\lambda_k$ decreased monotonically this would mean that each $kb$-th eigenvalue increased from $\lambda_{kb}$ to $\lambda_k$, implying faster convergence to $kb$-th principal component.

    The obvious downside of the method above is that in a new parameterization some of the eigenvalues become zero --- therefore they are never learned.
    A simple solution is to enlarge the embedding: $\tilde z(x) = [\cos(2\pi \sigma^{j/M} x), \sin(2\pi \sigma^{j/M} x)]^T$, where $M \in \mathbb{N}$ and $\sigma \in \RR_+$ are tunable parameters; this referred as "positional encoding" in \cite{tancik2020fourier}.
    Another solution proposed by \cite{tancik2020fourier} is random Gaussian projections: $\tilde z(x) = [\cos(2\pi B x), \sin(2\pi B x)]^T$, where $B \in \RR^{M \times d}$, each element of $B$ is sampled independently from $\NN(0,\sigma^2)$, and $M$ and $\sigma$ are tunable parameters.
    Both solution perform on par with each other and much better than the original embedding: compare (c), (d), and (e) in Figure~\ref{fig:image_regression}.

    The same method suites other low-dimensional regression problems as well; \cite{tancik2020fourier} provide examples of 3D shape regression, MRI reconstruction, and inverse rendering.
    See Figure~\ref{fig:low_dim_regression} for comparison of outputs of a neural net with no enconding of inputs (top row) and the proposed Gaussian encoding (bottom row).

    One more notable example is Solid Isotropic Material Penalisation, an instance of topology optimization.
    The task here is to optimize over material density at $N$ points $y \in [0,1]^N$ to obtain a shape that can withstand forces applied at certain points.

    Given a density $y$ and a force vector $F$, the SIMP method constructs a stiffness matrix $K(y)$, and derives a displacement vector $U(y)$ by solving a linear system $K(y) U(y) = F$.
    The resulting construction is stable if the forces do not do any work, i.e. $U^T(y) F = 0$.
    The density is therefore optimized to minimize the work $C(y) = U^T(y) F U(y) \to \min_y$ under a volume constraint $\sum_{i=1}^N y_i = V$; $C$ is usually called compliance.

    We can cast the constrained optimization problem as an unconstrained one by introducing pre-density $x \in \RR^N$ and constructing density as $y_i = \sigma(x_i + b(x))$, where $b$ is a function that ensures the volume constraint.
    Denoting this operation as $y = \Sigma(x)$, we get a new unconstrained optimization problem in the space of pre-densities: $C(\Sigma(x)) \to \min_x$.

    While the above problem is not a regression problem, we can still model $x$ as outputs of a neural net at the corresponding grid points.
    However, lack of translation invariance results in unplausible patterns.
    \cite{dupuis2021dnn} used a similar embedding scheme as \cite{tancik2020fourier} to control this issue.
    On the other hand, in contrast to \cite{tancik2020fourier}, \cite{dupuis2021dnn} used $\sin(\omega x)$ as activation instead of ReLU, and used $\omega$ together with bias initialization variance to control sharpness of output shapes, instead of modifying the embedding.
    Both methods aim to "widen" the spectrum of the limit NTK.

    \subsection{A theoretical tool}
    \label{sec:app_theory}


    Apart from providing a meaningful kernel for kernel methods, NTK can be used as a concept useful for reasoning about neural nets of large width.
    Indeed, as stated in \cref{sec:convergence}, NTK, while being random and evolving, converges to a constant deterministic limit as width goes to infinity.
    One can hope that for large enough width, the NTK stays close to its limit with high probability.
    Therefore, any result valid for kernel regression with NTK taken as a kernel, may become also valid with high probability for a wide enough net.

    \subsubsection{Global GD convergence}

    Let us start with the following result valid for kernel regression with a constant kernel:
    when the kernel is positive-definite, kernel regression learns the dataset.
    Indeed, recall the training dynamics of a kernel regression with kernel $\Theta$ trained to minimize square loss on a training dataset $(\vec x, \vec y)$:
    \begin{equation}
        \dot f_t(\vec x)
        = \Theta(\vec x, \vec x) (\vec y - f_t(\vec x)).
    \end{equation}
    Assuming $\Theta(\vec x, \vec x) \geq \lambda$,
    \begin{equation}
        \frac{d}{dt}\left(\frac{1}{2} \| \vec y - f_t(\vec x) \|_2^2\right)
        = -(\vec y - f_t(\vec x))^T \Theta(\vec x, \vec x) (\vec y - f_t(\vec x))
        \leq -\lambda \| \vec y - f_t(\vec x) \|_2^2,
    \end{equation}
    which gives
    \begin{equation}
        \| \vec y - f_t(\vec x) \|_2^2
        \leq e^{-2\lambda t} \| \vec y - f_0(\vec x) \|_2^2.
    \end{equation}
    Hence $\lambda > 0$ suffices to guarantee that $f_t(\vec x)$ converges to $\vec y$ as $t \to \infty$.

    Suppose now our kernel regression uses a random time-dependent kernel $\hat\Theta_t$ instead of $\Theta$:
    \begin{equation}
        \dot f_t(\vec x)
        = \hat\Theta_t(\vec x, \vec x) (\vec y - f_t(\vec x)).
    \end{equation}
    If we manage to guarantee that with probability $\geq 1-\delta$ $\forall t \geq 0$ $\hat\Theta_t(\vec x, \vec x) \geq \lambda$ then $\lambda > 0$ suffices to guarantee that $f_t(\vec x)$ converges to $\vec y$ as $t \to \infty$ with probability $\geq 1-\delta$.
    Indeed,
    \begin{equation}
        \frac{d}{dt}\left(\frac{1}{2} \| \vec y - f_t(\vec x) \|_2^2\right)
        = -(\vec y - f_t(\vec x))^T \hat\Theta_t(\vec x, \vec x) (\vec y - f_t(\vec x))
        \leq -\lambda \| \vec y - f_t(\vec x) \|_2^2
        \quad
        \text{w.p. $\geq 1-\delta$},
    \end{equation}
    which gives
    \begin{equation}
        \| \vec y - f_t(\vec x) \|_2^2
        \leq e^{-2 \lambda t} \| \vec y - f_0(\vec x) \|_2^2
        \quad
        \text{w.p. $\geq 1-\delta$}.
    \end{equation}

    One of the first results of this kind concerns ReLU nets with one hidden layer under NTK parameterization:
    \begin{equation}
        f(x; a_{1:n}, w_{1:n})
        = \frac{1}{\sqrt{n}} \sum_{i=1}^n a_i [w_i^T x]_+.
        \label{eq:two_layered_ReLU_net_ntk}
    \end{equation}
    We aim to minimize square loss on a dataset $(\vec x, \vec y)$ of size $m$ with gradient descent on the input weights:
    \begin{equation}
        \dot w_i(t)
        = \frac{1}{\sqrt{n}} \sum_{k=1}^m (y_k - f(x_k; a_{1:n}, w_{1:n}(t))) a_i [w_i^T(t) x_k > 0] x_k
        \quad
        \forall i \in [n].
    \end{equation}
    We sample $w_i \sim \NN(0,I_{n_0})$ and $a_i \in U(\{-1,1\})$ $\forall i \in [n]$ independently.
    The goal of sampling $a_i$ from this particular distribution is mere simplification: in this case $a_i^2 = 1$, which simplifies the NTK Gram matrix a little bit:
    \begin{equation}
        \hat\Theta_t(x_k, x_l) =
        \frac{1}{n} \sum_{i=1}^n [w_i^T(t) x_k > 0] [w_i^T(t) x_l > 0] x_k^T x_l.
    \end{equation}
    However, it is possible to apply the same technique to any distribution of the output layer not depending on $n$.
    Note that the Gram matrix depends merely on activation patterns of the hidden layer computed on the dataset.

    The limit NTK is therefore given as:
    \begin{equation}
        \Theta(x_k, x_l) =
        \EE_{w \sim \NN(0, I_{n_0})} [w^T x_k > 0] [w^T x_l > 0] x_k^T x_l.
    \end{equation}
    Note that in our two-layered case, $\Theta(x,x') = \lim_{n \to \infty} \hat\Theta_t(x,x') = \EE \hat\Theta_0(x,x')$.
    In the sequel, we denote the Gram matrices $\hat\Theta_t(\vec x, \vec x)$ as $H(t)$ and $\Theta(\vec x, \vec x)$ as $H^\infty$.
    Let $\lambda_0$ to be the least eigenvalue of $H^\infty$.

    \begin{theorem}[\cite{du2018gradient}]
        Consider the setting discussed above and further assume $\|x_k\|_2 \geq 1$ and $|y_k| \leq 1$ $\forall k \in [m]$.
        Then $\exists C, C_0 > 0$ such that $\forall \delta \in (0,1)$ taking
        \begin{equation}
            n >
            \max\left(
                C \frac{m^6}{\lambda_0^4 \delta^3}, \;
                C_0 \frac{m^2}{\lambda_0^2} \log\left(\frac{2m}{\delta}\right)
            \right)
        \end{equation}
        guarantees $H(t) \geq \lambda_0/2$ $\forall t \geq 0$ w.p. $\geq 1-\delta$.
        \label{thm:convergence_2layer}
    \end{theorem}
    This result implies $\| \vec y - f_t(\vec x) \|_2^2 \leq e^{-\lambda_0 t} \| \vec y - f_0(\vec x) \|_2^2$ w.p. $\geq 1-\delta$, as discussed above.
    
    For the full proof, see the original paper \cite{du2018gradient} or lecture notes \cite{golikov2020notes}.
    We are going to discuss, very briefly, only crucial parts of the proof in the sequel.

    The proof is based on four lemmas.
    The first lemma states that as long as $n = \Omega(m^2 \lambda_0^{-2} \log(m/\delta))$, where $\Omega$ hides a certain constant, $\|H(0) - H^\infty\|_2 \leq \lambda_0/4$, where $\|\cdot\|_2$ denotes a singular norm, w.p. $\geq 1-\delta$; this implies $H(0) \geq 3\lambda_0/4$ with the same probability.
    As already noted above, $\EE H(0) = H^\infty$.
    This allows one to apply a concentration inequality to each element of $H(0)$.
    Union bound then gives a bound that holds uniformly for all elements of $H(0)$.
    This implies a bound on $\|H(0) - H^\infty\|_F$, hence on a singular norm as well.

    The second lemma states that as long as $\forall i \in [n]$ $\|w_i - w_i(0)\|_2 \leq R$ for certain $R = R(\delta,\lambda_0,m)$, $\| H - H(0) \|_2 \leq \lambda_0/4$ w.p. $\geq 1-\delta$.
    In other words, as long as weights are close to initialization, the corresponding Gram matrix is close to the initial one too.
    The idea is that as long as the weights are not far from their initialization, with certain probability, not many of the hidden neurons can alter their activation patterns on the train dataset.
    Since as already noted above, our Gram matrices depend only on activation patterns on the train dataset, this implies a tail bound on $|H_{kl}(0) - H_{kl}^\infty|$ $\forall k,l \in [m]$, which gives a tail bound on $\|H(0) - H^\infty\|_2$ with the same technique as used in the first lemma.

    The third lemma states that as long as $H(s) \geq \lambda_0/2$ $\forall s \in [0,t]$ (we haven't proven it yet), weights indeed stay close to their initialization: $\forall i \in [n]$ $\|w_i(t) - w_i(0)\|_2 \leq R'$ for certain $R' = R'(\lambda_0,m,n)$.
    This can be proven by a very simple estimate:
    \begin{multline}
        \left\|\frac{dw_i(s)}{ds}\right\|_2 =
        \left\|\frac{1}{\sqrt{n}} \sum_{k=1}^m (y_k - f_s(x_k)) a_i [w_i^T(s) x_k > 0] x_k\right\|_2 \leq
        \\\leq
        \frac{1}{\sqrt{n}} \sum_{k=1}^m |y_k - f_s(x_k)| \leq
        \sqrt{\frac{m}{n}} \|\vec y - f_s(\vec x)\|_2 \leq
        \sqrt{\frac{m}{n}} e^{-\lambda_0 s / 2} \|\vec y - f_0(\vec x)\|_2.
    \end{multline}
    This gives $\forall i \in [n]$:
    \begin{multline}
        \| w_i(t) - w_i(0) \|_2 =
        \left\|\int_0^t \frac{dw_i(s)}{ds} \, ds\right\|_2 \leq
        \int_0^t \left\|\frac{dw_i(s)}{ds}\right\|_2 \, ds \leq
        \\\leq
        \frac{2 \sqrt{m}}{\lambda_0 \sqrt{n}} \left(1 - e^{-\lambda_0 t / 2}\right) \|\vec y - f_0(\vec x)\|_2 \leq
        \frac{2 \sqrt{m}}{\lambda_0 \sqrt{n}} \|\vec y - f_0(\vec x)\|_2.
    \end{multline}

    Finally, the fourth lemma states that as long as $R' < R$, $\| H(t) - H(0) \|_2 \leq \lambda_0/4$ $\forall t \geq 0$ w.p. $\geq 1-\Omega(\delta)$ where $\Omega$ hides a certain constant.
    Combined with the first lemma, this implies $H(t) \geq \lambda_0/2$ $\forall t \geq 0$ w.p. $\geq 1-\Omega(\delta)$.
    The condition $R'(\lambda_0,m,n) < R(\delta,\lambda_0,m)$ gives the second lower bound on $n$ (the first one is given be the first lemma).
    By changing $\delta$, we get the desired result.

    The fourth lemma is proven as follows.
    Let $t_0$ be the first moment of time when the second lemma becomes no longer applicable, i.e. $t_0 = \inf\left\{t \geq 0: \; \max_{i \in [n]} \| w_i(t) - w_i(0) \|_2 > R\right\}$.
    Assume it is finite.
    Since weights are continuous functions of time, $\max_{i \in [n]} \| w_i(t_0) - w_i(0) \|_2 = R$.
    Hence the second lemma holds for $w_{1:n} = w_{1:n}(t)$ $\forall t \in [0,t_0]$ and $\| H(t) - H(0) \|_2 \leq \lambda_0/4$ w.p. $\geq 1-\delta$ $\forall t \in [0,t_0]$, therefore $H(t) \geq \lambda_0/2$ w.p. $\geq 1-\Omega(\delta)$ $\forall t \in [0,t_0]$.
    But then the third lemma holds as well: $\forall i \in [n]$ $\|w_i(t_0) - w_i(0)\|_2 \leq R' < R$; contradiction.
    Hence $\forall t \geq 0$ $\max_{i \in [n]} \| w_i(t) - w_i(0) \|_2 \leq R$ and the second lemma gives the desired statement.

    \cref{thm:convergence_2layer} requires the number of hidden units $n$ to grow as $m^6$ with the size of a train dataset and as $\delta^{-3}$ with the failure probability.
    This bound is way too loose for practical purposes: indeed, even for very small datasets $m \geq 100$ which results in a bound of the order at least $10^8$.
    If we want the bound to be valid with at least $90\%$ probability, we pay three orders of magnitude more.
    Note that modern architectures designed to be trained on large datasets like ImageNet ($m=10^6$) have width barely exceeding $10^4$.

    We state one of the existing improvements of \cref{thm:convergence_2layer} below:
    \begin{theorem}[\cite{song2019quadratic}]
        Under the same setting as \cref{thm:convergence_2layer}, $\exists C, C_0 > 0$ such that $\forall \delta \in (0,1)$ taking
        \begin{equation}
            n >
            \max\left(
                C \frac{m^4}{\lambda_0^4} \log^3\left(\frac{m}{\delta}\right), \;
                C_0 \frac{m^2}{\lambda_0^2} \log\left(\frac{2m}{\delta}\right)
            \right)
        \end{equation}
        guarantees $H(t) \geq \lambda_0/2$ $\forall t \geq 0$ w.p. $\geq 1-\delta$.
        \label{thm:convergence_2layer_quartic}
    \end{theorem}
    This result decreases the exponent of $m$ from $6$ to $4$ and makes the $\delta$-dependence logarithmic.
    The proof follows the same path as above.
    Note however that the previous result aimed for elementwise tail bounds on $H(0) - H^\infty$ or $H - H(0)$ which lead to tail bounds on $\|H(0) - H^\infty\|_2$ and $\|H - H(0)\|_2$ by union bound, which gives an $m^2$ factor.
    One of the improvements proposed by \cite{song2019quadratic} is to replace these elementwise bounds with matrix-Chernoff bounds --- they do not give this $m^2$ factor, thus leading to better bounds.
    The other improvement is to replace Markov inequalities that result in $1/\delta$ factors with Bernstein inequality that results only in $\log(1/\delta)$ ones.

    The $m^4$ width bound is still far from being realistically tight.
    We are not aware of any further improvements of the results discussed above that apply the idea of NTK stability.
    Global gradient descent convergence can be, however, proved by first proving gurantees on convergence to local minima and then proving that all minima are global for wide enough nets.
    See \cite{lee2016gradient,panageas2017gradient,mertikopoulos2020almost} for the first line of works and \cite{yu1995local,nguyen2017loss,nguyen2019connected,nguyen2021note} for the second.
    None of the works of both lines use the idea of NTK stability and they neither rely on NTK parameterization.
    \cite{nguyen2019connected} proves that $n = m$ is enough of leaky ReLU nets to have only global "local valleys" (generalization of global minima to certain losses such as cross-entropy) and \cite{nguyen2021note} demonstrates that this bound cannot be improved for two-layered nets and general data.

    \cite{du2019gradient} extends \cref{thm:convergence_2layer} to deep nets.
    Their proof idea is the same: first show that $H(0)$ is close to $H^\infty$, then show that $H(t)$ stays close to $H(0)$.
    However for the multilayer case, $H(0)$ cannot be proven to be close to $H^\infty$ just by concentration of measure.
    When layers are many, perturbations caused by finite width result in deviations exponential with respect to the number of layers $L$.
    For this reason, their bound grows exponentially with $L$.
    See also \cite{allen2019convergence} for a similar result with a bound depending on $m$ only polynomially, proved using a different technique.

    \subsubsection{Generalization guarantees}

    Stability of NTK has another interesting consequence.
    Suppose the empirical NTK is constant, i.e. $\hat\Theta_t = \hat\Theta_0$.
    It is equivalent to say that the corresponding model is linearized:
    \begin{equation}
        f(x; \theta)
        = f(x; \theta_0) + \nabla_\theta^T f(x; \theta_0) (\theta - \theta_0).
    \end{equation}
    For brevity, denote $\vec u_t = f_t(\vec x)$ and $Z_t^{ik} = \partial_{\theta_i} f(x_k; \theta_t)$.
    Hence $Z_t \in \RR^{N \times m}$ where $N$ is the total number of parameters and $\vec u_t = \vec u_0 + Z_0^T (\theta_t - \theta_0)$.

    Note that $H_t = Z_t^T Z_t$.
    Recall the train set predictions for constant kernel:
    \begin{equation}
        \vec u_t
        = \vec y + e^{-H_0 t} (\vec u_0 - \vec y).
    \end{equation}
    In our linearized dynamics, the weights evolve as follows:
    \begin{equation}
        \dot\theta_t
        = Z_0 (\vec y - \vec u_t)
        = Z_0 e^{-H_0 t} (\vec y - \vec u_0).
    \end{equation}
    Straightforward integration gives:
    \begin{equation}
        \theta_t
        = \theta_0 + Z_0 H_0^{-1} \left(I - e^{-H_0 t}\right) (\vec y - \vec u_0).
    \end{equation}
    Recalling $H_0 = Z_0^T Z_0$, at the end of training ($t \to \infty$) we get
    \begin{equation}
        \|\theta_\infty - \theta_0\|_2^2
        = (\theta_\infty - \theta_0)^T (\theta_\infty - \theta_0)
        = (\vec y - \vec u_0)^T H_0^{-1} (\vec y - \vec u_0).
    \end{equation}

    Define $\FF_B^{w_{1:n}(0), a_{1:n}}$ as a set of models of the form (\ref{eq:two_layered_ReLU_net_ntk}) with output weights $a_{1:n}$ and input weights $w_{1:n}$ such that $\| W - W(0) \|_F \leq B$ for given $w_{1:n}(0)$.
    The above considerations state that a trained model always lies in $\FF_B^{w_{1:n}(0), a_{1:n}}$ with $B = (\vec y - \vec u_0)^T H_0^{-1} (\vec y - \vec u_0)$.
    
    Hence our training procedure outputs models in a certain set rather than any model in of the form (\ref{eq:two_layered_ReLU_net_ntk}).
    Upper-bounding Rademacher complexity of this model set will give us a generalization bound as we shall see below.
    Let us upper-bound the Rademacher complexity conditioned on a dataset $(\vec x, \vec y)$ of size $m$:
    \begin{multline}
        \Rad{\FF_B^{w_{1:n}(0), a_{1:n}}}{\vec x, \vec y} =
        \EE_{\sigma_{1:m} \sim \{-1,1\}^m} \sup_{f \in \FF_B^{w_{1:n}(0), a_{1:n}}} \left(\frac{1}{m} \sum_{k=1}^m \sigma_k u_k\right) =
        \\=
        \frac{1}{m} \EE_{\sigma_{1:m} \sim \{-1,1\}^m} \sup_{\| W - W(0) \|_F \leq B} \left(\sum_{k=1}^m \sigma_k \frac{1}{\sqrt{n}} \sum_{i=1}^n a_i [w_i^T(0) x_k \geq 0] w_i^{T} x_k\right) =
        \\=
        \frac{1}{m} \EE_{\sigma_{1:m} \sim \{-1,1\}^m} \sup_{\| W - W(0) \|_F \leq B} \left( \vec\sigma^T Z^{T}(0) \theta \right) =
        \\=
        \frac{1}{m} \EE_{\sigma_{1:m} \sim \{-1,1\}^m} \sup_{\| W - W(0) \|_F \leq B} \left( \vec\sigma^T \tilde Z^{T}(0) (\theta - \theta_0) \right) =
        \\=
        \frac{B}{m} \EE_{\sigma_{1:m} \sim \{-1,1\}^m} \| Z(0) \vec\sigma \|_2 \leq
        \frac{B}{m} \sqrt{\EE_{\sigma_{1:m} \sim \{-1,1\}^m} \| Z(0) \vec\sigma \|_2^2} =
        \frac{B}{m} \| Z(0) \|_F.
    \end{multline}
    
    Note that
    \begin{equation}
        \| Z(0) \|_F^2 =
        \frac{1}{n} \sum_{i=1}^n \sum_{k=1}^m [w_i^T(0) x_k \geq 0].
    \end{equation}
    It is an average of i.i.d random variables, which allows for Hoeffding's inequality:
    \begin{equation}
        \PP(\| Z(0) \|_F^2 - \frac{m}{2} \geq \epsilon) \leq 
        e^{-2n \epsilon^2 / m^2}.
    \end{equation}
    This gives w.p. $\geq 1-\delta$ over initialization,
    \begin{equation}
        \| Z(0) \|_F^2 \leq
        \frac{m}{2} + \sqrt{\frac{m^2}{2n} \log\left(\frac{1}{\delta}\right)}.
    \end{equation}
    Finally, we got that w.p. $\geq 1-\delta$ over initialization,
    \begin{equation}
        \Rad{\FF_B^{w_{1:n}(0), a_{1:n}}}{(\vec x, \vec y)} \leq
        \frac{B}{\sqrt{m}} \sqrt{\frac{1}{2} + \sqrt{\frac{1}{2n} \log\left(\frac{1}{\delta}\right)}}.
    \end{equation}

    Consider zero-one risk: $r(y,z) = [y z < 0]$; we have $R(f) = \EE_{x,y \sim \DD} r(y,f(x))$ and $\hat R(f) = \EE_{x,y \in S_m} r(y,f(x))$, correspondingly.
    From the generalization theory, we know that for any $B$ and for any initialization $w_{1:n}(0), a_{1:n}$, w.p. $\geq 1-\tilde\delta$ over the training dataset, $\forall f \in \FF_B^{w_{1:n}(0), a_{1:n}}$,
    \begin{equation}
        R(f) \leq
        \hat R_m(f) + \EE_{(\vec x, \vec y)} \Rad{\FF_B^{w_{1:n}(0), a_{1:n}}}{(\vec x, \vec y)} + \sqrt{\frac{1}{2m} \log \frac{1}{\tilde\delta}}
        \quad
        \text{w.p. $\geq 1 - \tilde\delta$ over $(\vec x, \vec y)$.}
    \end{equation}

    We want to take $B = (\vec y - \vec u_0)^T H_0^{-1} (\vec y - \vec u_0)$ but it depends on the dataset $(\vec x, \vec y)$.
    Take a sequence $\{B_j\}_{j=1}^\infty$ monotonically increasing to infinity and a sequence $\{\tilde\delta_j\}_{j=1}^\infty$ of deltas $\in (0,1)$ that sum to $\tilde\delta$.
    This allows us to apply a union bound: w.p. $\geq 1-\tilde\delta$ over the training dataset, for any initialization $w_{1:n}(0), a_{1:n}$, $\forall j \in \mathbb{N}$, $\forall f \in \FF_{B_j}^{w_{1:n}(0), a_{1:n}}$,
    \begin{equation}
        R(f) \leq
        \hat R_m(f) + \EE_{(\vec x, \vec y)} \Rad{\FF_{B_j}^{w_{1:n}(0), a_{1:n}}}{(\vec x, \vec y)} + \sqrt{\frac{1}{2m} \log \frac{1}{\tilde\delta_j}}.
    \end{equation}
    We are free to choose minimal $j$ such that $B_j \geq (\vec y - \vec u_0)^T H_0^{-1} (\vec y - \vec u_0)$; denote it by $\hat j$.
    Let for definiteness $B_j = j$.
    Then $B_{\hat j} \leq 1 + (\vec y - \vec u_0)^T \hat\Theta_0^{-1} (\vec y - \vec u_0)$.
    
    Putting all together, we have w.p. $\geq 1-\tilde\delta$ over the training dataset, w.p. $\geq 1-\delta$ over initialization,
    \begin{multline}
        R(f(\theta_\infty)) \leq
        \hat R_m(f(\theta_\infty)) + 
        \\+
        \frac{1 + (\vec y - \vec u_0)^T H_0^{-1} (\vec y - \vec u_0)}{\sqrt{m}} \sqrt{\frac{1}{2} + \sqrt{\frac{1}{2n} \log\left(\frac{1}{\delta}\right)}} + \sqrt{\frac{1}{2m} \log \frac{1}{\tilde\delta_{\hat j}}}.
        \label{eq:generalization_bound_for_fixed_acts_model}
    \end{multline}

    Recall that the bound above was obtained under the assumption of constant NTK.
    In order to relax this assumption, one has to show that, possibly for large enough width, $H_t^{-1}$ stays close to $H_0^{-1}$.
    Note that when proving global GD convergence we had to prove that $H_t$ stays close to $H_0$, which is different.
    The required closeness result is proven in \cite{arora2019fine}, it leads to the following theorem:
    \begin{theorem}[\cite{arora2019fine}]
        Under the same setting as \cref{thm:convergence_2layer}, $\exists p, C, C_0 > 0$ such that $\forall \delta \in (0,1)$ taking
        \begin{equation}
            n >
            \max\left(
                C \frac{m^7}{\lambda_0^4 \delta^p}, \;
                C_0 \frac{m^2}{\lambda_0^2} \log\left(\frac{2m}{\delta}\right)
            \right)
        \end{equation}
        guarantees w.p. $\geq 1-\delta$ over the training dataset of size $m$ and w.p. $\geq 1-\delta$ over initialization,
        \begin{multline}
            R(f(\theta_\infty)) \leq
            \hat R_m(f(\theta_\infty)) + 
            \\+
            \frac{1 + (\vec y - \vec u_0)^T \left(H^{\infty}\right)^{-1} (\vec y - \vec u_0)}{\sqrt{m}} \sqrt{\frac{1}{2} + \sqrt{\frac{1}{2n} \log\left(\frac{1}{\delta}\right)}} + \sqrt{\frac{1}{2m} \log \frac{1}{\delta}}.
        \end{multline}
        \label{thm:generalization_2layer}
    \end{theorem}

    \section{Standard parameterization and kernel evolution}
    \label{sec:standard_param}


    \begin{figure}
        \label{fig:kernel_velocity}
        \includegraphics[width=0.9\textwidth]{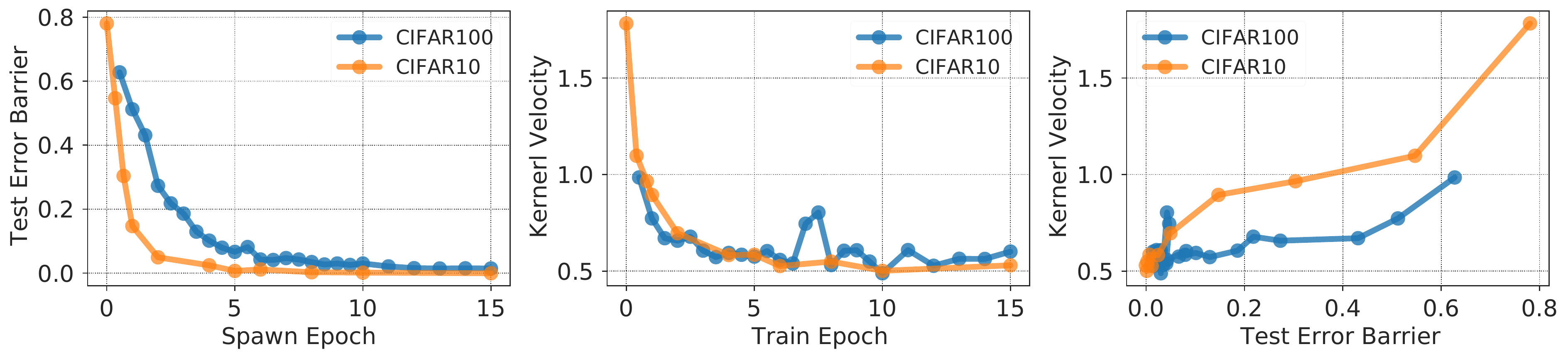}
        \caption{The figure is borrowed from \cite{fort2020deep}.}
    \end{figure}

    As was noted in \cref{sec:convergence}, NTK diverges under standard parameterization.
    Recall the example of a two-layered net:
    \begin{equation}
        f(x; a_{1:n}, w_{1:n})
        = \sum_{i=1}^n a_i \phi(w_i x),
        \quad
        a_{1:n} \sim \NN(0, n^{-1} I),
        \quad
        w_{1:n} \sim \NN(0, I);
    \end{equation}
    \begin{equation}
        \hat\Theta_t(x,x')
        = \sum_{i=1}^n \left(\phi(w_i(t) x) \phi(w_i(t) x') + a_i^2(t) \phi'(w_i(t) x) \phi'(w_i(t) x') x x'\right).
    \end{equation}
    At $t=0$, since $w_i$ are independent and of the order of $O(1)$, the sum diverges proportionaly to $n$.
    Since under square loss, $\dot f_t(x) = \hat\Theta_t(x,\vec x) (\vec y - f_t(\vec x))$, the model prediction at any point $x$ receive a $O(n)$ increment at the very beginning of training.
    In other words, model predictions diverge with width, making the model useless for regression.

    However, if the goal is classification, magnitude of predictions does not matter; what matters is their signs for binary classification, or indices of the largest logits when classes are multiple.
    Therefore in this case, an infinite-width limit under standard parameterization still may make sense besides of divergent NTK, see \cite{golikov2020dynamically}.

    In order to deal with divergence, consider a normalized empirical NTK $\tilde\Theta_t(x,x') = \hat\Theta_t(x,x') / n$; its infinite-width limit at initialization is $\EE_{w \sim \NN(0,1)} \phi(w x) \phi(w x')$; we shall refer it as normalized NTK and denote as $\tilde\Theta(x,x')$.
    In contrast to NTK under NTK parameterization, normalized NTK under standard parameterization evolves with time \cite{golikov2020dynamically}:
    \begin{multline}
        \frac{d\tilde\Theta_t(x,x')}{dt}
        = \frac{1}{n} \sum_{i=1}^n \left(\phi(w_i(t) x) \phi'(w_i(t) x') x' + \phi'(w_i(t) x) \phi(w_i(t) x') x\right) \frac{dw_i(t)}{dt} 
        +\\+ \frac{1}{n} \sum_{i=1}^n a_i^2(t) x x' \left(\phi'(w_i(t) x) \phi''(w_i(t) x') x' + \phi''(w_i(t) x) \phi(w_i(t) x') x\right) \frac{dw_i(t)}{dt}
        +\\+ \frac{1}{n} \sum_{i=1}^n 2 a_i(t) \phi'(w_i(t) x) \phi'(w_i(t) x') x x' \frac{da_i(t)}{dt}.
    \end{multline}
    Recall the gradient flow dynamics under standard parameterization:
    \begin{equation}
        \frac{a_k(t)}{dt} 
        = \sum_{j=1}^m \phi(w_k(t) x_j),
        \quad
        \frac{w_k(t)}{dt} 
        = \sum_{j=1}^m a_k(t) \phi'(w_k(t) x_j) x_j.
    \end{equation}
    At $t=0$, we have $\dot a_k = O(1)$, while $\dot w_k = O(n^{-1/2})$.
    Since $a_k(0) = O(n^{-1/2})$ and $w_k(0) = O(1)$, it means that for any $t > 0$ independent on $n$, $a_k(t) = O(1)$, $\dot a_k(t) = O(1)$, $w_k(t) = O(1)$, and $\dot w_k(t) = O(1)$.
    A naive estimate of the sums then gives $\frac{d\tilde\Theta_t(x,x')}{dt} = O(1) + O(1) + O(1) = O(1)$ for any $t > 0$ independent on $n$.
    Therefore the normalized kernel keeps evolving with time even in the limit of infinite width.
    
    This can be the reason for superior performance of neural networks to conventional kernel methods and NTK.
    A kernel measures similarity between points in a feature space.
    While for NTK this feature space is fixed, a neural net varies its corresponding kernel feature space, hopefully making it better suitable for the task at hand; moreover, under standard parameterization, this feature does not vanish for large width.

    The way an empirial NTK varies with time can be measured with kernel velocity, defined as kernel distance between the kernels corresponding to two consequent optimization steps.
    Kernel distance is in its turn defined as one minus cosine similarity between Gram matrices $H$ and $H'$ of the corresponding kernels:
    \begin{equation}
        \rho(H, H') = 1 - \frac{\tr(H H^{\prime,T})}{\sqrt{\tr(H H^T) \tr(H' H^{\prime,T})}}.
    \end{equation}

    After measuring kernel velocity for a realistic net under standard parameterization, \cite{fort2020deep} distinguished two phases of training: a phase of rapid kernel evolution, and a phase of almost constant NTK, see Figure~\ref{fig:kernel_velocity}.
    The first phase is called \emph{chaotic}, while the second one is coined \emph{ordered}.
    Curiously enough, these two phases can be distinguished not only by kernel velocity.
    Suppose the network is trained up to time $T$, called \emph{spawn epoch}.
    Two independent copies of the same network is then trained further.
    In other words, we train two networks which remain the same up to time $T$ and may diverge afterwards due to randomness of training procedure.
    We then measure \emph{test error barrier} between these two networks, i.e. height of the error "hill" on a straight segment between their corresponding weights.
    A small error barrier would mean that training of the two networks ended up in the same valley of test error, which likely means that they are similar.
    As one can see in Figure~\ref{fig:kernel_velocity}, the test error barrier drops dramatically with growth of spawn epoch.
    Also, the two quantities under discussion, kernel velocity and error barrier appear to be strongly correlated, see again Figure~\ref{fig:kernel_velocity}.
    There are also other quantities that experience sharp transition on the border of the two phases: kernel distance between child networks as a function of spawn epoch, ReLU activation Hamming distance, and Hamming distance between responses on the test set; see \cite{fort2020deep} for details.

    \section{Beyond NTK}
    \label{sec:beyond}


    While NTK kernel regression has a natural interpretation of training an infinitely wide neural network under certain parameterization with gradient flow (see \cref{sec:convergence}), NTK is not the only possible kernel that can be constructed using a neural net.

    \subsection{NNGP kernel}

    One of the other notable "neural kernels" is the NNGP-kernel \cite{lee2018deep}, defined as $K(x,x') = \EE_\theta f(x; \theta) f(x'; \theta)$, where $f(\cdot; \theta)$ is a parametric model with weights $\theta$ and scalar output.
    Suppose $f$ is a neural network with the output layer of the form $f(x) = v^T h(x)$, where $h(x) \in \RR^n$ is its last layer representation and $v \sim \NN(0, I_n / n)$ independent on $h$.
    Then $K(x,x') = \frac{1}{n} \EE h^T(x) h(x')$.
    As we have seen in \cref{sec:limit} on the example of fully-connected and convolutional nets, the last layer representations tend to iid Gaussians as width go to infinity.
    In other words, $\forall i \in [n]$ $h^i$ tend to identical and independent Gaussian processes with covariance $\EE h^i(x) h^i(x') = \frac{1}{n} \EE h^T(x) h(x')$, which is exactly $K(x,x')$.
    This motivates the term "NNGP" --- \emph{Neural Network Gaussian Process}.
    
    Note that we have already seen the object $\EE h^i(x) h^i(x')$ in \cref{sec:limit}: when $h = h_l$ --- the $l$-th layer hidden representation of a fully-connected network, the above object is hidden layer covariance $q_l(x,x')$.
    Therefore the NNGP of this fully-connected network is nothing else but $q_L(x,x')$.
    This can be generalized to the whole class of architectures expressible by tensor programs: see the Master theorem of \cite{yang2019tensor_i} mentioned in \cref{sec:convergence}.
    That is, any neuron of any hidden representation of a neural network expressible by a tensor program tends to a Gaussian process.

    Learning a Gaussian process with zero mean and covariance $K(\cdot,\cdot)$ on a training dataset $(\vec x, \vec y)$ means computing its Bayesian prosterior, which is again a Gaussian with mean $\mu(\cdot \,|\, (\vec x, \vec y))$ and covariance $K(\cdot,\cdot \,|\, (\vec x, \vec y))$ given below:
    \begin{equation}
        \mu(x \,|\, (\vec x, \vec y)) = K(x,\vec x) K^{-1}(\vec x, \vec x) \vec y;
    \end{equation}
    \begin{equation}
        K(x,x' \,|\, (\vec x, \vec y)) = K(x,x') - K(x,\vec x) K^{-1}(\vec x, \vec x) K(\vec x,x').
    \end{equation}
    
    Interestingly, training the last layer of an infinitely wide network with NNGP $K(\cdot,\cdot)$ results in exactly the same Gaussian process.
    When only the last layer is trained, the NNGP coincides with the NTK.
    Indeed, an NTK-parameterized NN of width $n$ with readout weights $v$ can be expressed as $f(x) = \frac{1}{\sqrt{n}} v^T h(x)$ with $v \sim \NN(0, I_n)$.
    The empirical NTK is therefore given by $\hat\Theta_0(x,x') = \frac{1}{n} \nabla^T_v (v^T h(x)) \nabla_v (v^T h(x')) = \frac{1}{n} h^T(x) h(x')$, which converges to $\EE h^i(x) h^i(x') = K(x,x')$ as $n \to \infty$; note that $h(\cdot)$ also depends on $n$.
    
    Recall the model prediction dynamics under constant NTK which is $K$ in our case:
    \begin{equation}
        f_t(x) = f_0(x) - K(x,\vec x) K^{-1}(\vec x,\vec x) \left(I - e^{-K(\vec x,\vec x) t}\right) (f_0(\vec x) - \vec y).
    \end{equation}
    Since $f_0(\cdot)$ is a Gaussian process as discussed before and $K(\vec x,\vec x)$ is deterministic, $f_t(\cdot)$ is a Gaussian process for any $t \geq 0$.
    Its mean $\mu_t(\cdot)$ and covariance $K_t(\cdot,\cdot)$ are:
    \begin{equation}
        \mu_t^{NNGP}(x) = K(x,\vec x) K^{-1}(\vec x,\vec x) \left(I - e^{-K(\vec x,\vec x) t}\right) \vec y;
    \end{equation}
    \begin{multline}
        K_t^{NNGP}(x,x') 
        = K(x,x') 
        +\\+ K(x,\vec x) K^{-1}(\vec x,\vec x) \left(I - e^{-K(\vec x,\vec x) t}\right) K(\vec x,\vec x) \left(I - e^{-K(\vec x,\vec x) t}\right) K^{-1}(\vec x,\vec x) K(\vec x,x') 
        -\\- \left[K(x,\vec x) K^{-1}(\vec x,\vec x) \left(I - e^{-K(\vec x,\vec x) t}\right) K(\vec x,x') + K(x',\vec x) K^{-1}(\vec x,\vec x) \left(I - e^{-K(\vec x,\vec x) t}\right) K(\vec x,x)\right].
    \end{multline}
    It is easy to see that $\mu_t^{NNGP}(x) \to \mu(x \,|\, (\vec x, \vec y))$ and $K_t^{NNGP}(x,x') \to K(x,x' \,|\, (\vec x, \vec y))$ as $t \to \infty$ $\forall x,x'$.

    If not only the last layer is trained, NNGP does not generally correspond to NTK.
    The corresponding training dynamics is given by
    \begin{equation}
        f_t(x) = f_0(x) - \Theta(x,\vec x) \Theta^{-1}(\vec x,\vec x) \left(I - e^{-\Theta(\vec x,\vec x) t}\right) (f_0(\vec x) - \vec y).
    \end{equation}
    While $f_t(\cdot)$ is again a Gaussian process for any $t \geq 0$, its mean and covariance are different.
    In particular, as $t \to \infty$, they tend to
    \begin{equation}
        \mu_\infty^{NTK}(x) = \Theta(x,\vec x) \Theta^{-1}(\vec x,\vec x) \vec y;
    \end{equation}
    \begin{multline}
        K_\infty^{NTK}(x,x') 
        = K(x,x') 
        + \Theta(x,\vec x) \Theta^{-1}(\vec x,\vec x) K(\vec x,\vec x) \Theta^{-1}(\vec x,\vec x) \Theta(\vec x,x') 
        -\\- \left[\Theta(x,\vec x) \Theta^{-1}(\vec x,\vec x) K(\vec x,x') + \Theta(x',\vec x) \Theta^{-1}(\vec x,\vec x) K(\vec x,x)\right].
    \end{multline}
    As was shown in \cite{lee2019wide}, there does not exist an initial covariance matrix (a "prior") such that these mean and covariance correspond to Bayesian posterior given the training data.

    The "empirical" counterpart of NNGPs is $\hat K(x,x') = \frac{1}{n} h^T(x) h(x')$.
    Compared to empirical NTKs, empirical NNGPs are easier to compute as they do not require a backward pass.
    The corresponding memory footprint is also lower for empirical NNGPs as they do not require computing Jacobian matrices that scale as $O(N)$ where $N$ is the number of weights.
    This makes NNGPs more suitable for large models.
    As an example, \cite{park2020towards} used performance of empirical NNGPs as a proxy measure for neural architecture search.
    They argue that first, empirical NTKs are too costly to compute, and second, they provide worse learning signal for their task.

    NNGP of a generic neural network can be computed in a recursive manner, as was demonstrated in \cref{sec:limit} on the example of fully-connected and convolutional nets: $q_{l+1}(x,x') = \EE_{[z,z']^T \sim \NN(0,\Sigma_l(x,x'))} \phi(z) \phi(z')$, where $\Sigma_l(x,x') = \begin{pmatrix} q_l(x,x) & q_l(x,x') \\ q_l(x',x) & q_l(x',x') \end{pmatrix}$; the Master theorem of \cite{yang2019tensor_i} gives similar fomulas for a generic neural net.
    In the above example, there is an operation that maps a kernel $q_l(x,x')$ to a subsequent kernel $q_{l+1}(x,x')$.
    \cite{shankar2020neural} presents an algebra of operations on kernels.
    While this algebra consists of operations of only three types, it is enough to express NNGP of a fully-connected or a convolutional network with any elementwise nonlinearities.

    \subsection{Label-aware NTK}

    One of the major problems of kernel methods is \emph{label agnosticism.}
    Recall that a kernel evaluated at a pair of points is a scalar product of their mappings to some feaure space: $K(x,x') = \langle \Phi(x), \Phi(x') \rangle$.
    Therefore a kernel measures how similar the two points are, and a kernel method uses this information to derive responses on unseen data: $f(x) = K(x,\vec x) \vec\alpha$.
    Intuitively, a kernel $K$ should result in a good-generalizing model if $K(x,x')$ is positive when $y=y'$ and negative otherwise.
    Therefore the "perfect" kernel would be $K^*(x,x') = y y'$; the obvious problem is that it cannot be computed on unseen data.

    A kernel that can be computed on unseen data cannot depend on labels.
    Therefore, if data has several possible labelings, for a pair of data points $(x,x')$, there could be a labeling with $y=y'$ and a labeling with $y\neq y'$.
    At the same moment, $K(x,x')$ stays the same on both cases; therefore, the corresponding kernel method cannot generalize well on both of the labelings.

    As an example of several possible labelings on a single dataset, consider a dataset of pictures with two objects in each frame, and let the two objects belong to two disjoint sets of classes.
    Then one of the labelings may consider only the objects of the first classes set, while the other may consider the objects of the second set.

    \cite{chen2020label} propose two ways of making a kernel \emph{label-aware.}
    The first is mixing the kernel at hand with the perfect kernel $K^*(x,x') = y y'$: $K^{HR}(x,x') = (1-\lambda) K(x,x') + \lambda K^*(x,x')$ for $\lambda \in [0,1]$.
    If the perfect kernel was available, the best choice would be to take $\lambda=1$.
    Since it is not available, we have to approximate it somehow, therefore making the optimal $\lambda$ to become less than one.

    In order to approximate $K^*(x,x')$, we need a model that maps $(x,x')$ to $y y'$.
    Since the training dataset for this model consists $O(m^2)$ samples, and since the model itself has to be evaluated on $O(m)$ samples for each test point $x$, the model has to be relatively simple.
    \cite{chen2020label} consider models of the form $Z(x,x') = \vec y^T M(x,x',\vec x) \vec y$, where $M \in \RR^{m \times m}$.
    One of the possible choices of $M$ is $M(x,x',\vec x)_{ij} = \psi(K(x,x'),K(x_i,x_j))$, where $\psi(z_1,z_2)$ measures similarity.
    As one can see, this choice of $Z$ takes a linear combination of $y_i y_j$ with weights being similarities of $K(x,x')$ and $K(x_i,x_j)$.
    Intuitively, this reads as "$y y'$ and $y_i y_j$ are similar if $K(x,x')$ and $K(x_i,x_j)$ are close".

    While the above proposal can be applied to any kernel $K$, the second label-aware kernel of \cite{chen2020label} is a specific modification of NTK.
    Let us recall the construction of $\Theta^{NTH}$ resulted from integrating the learning dynamics up to the order $n^{-1}$, taking the limit of $t \to \infty$, and taking expectation (see \cref{sec:finite_width} and specifically Eq.~(\ref{eq:lantk_nth})):
    \begin{multline}
        \Theta^{NTH}(x_1,x_2)
        = O_{2,0}^{(0)}(x_1,x_2) + n^{-1} \EE O_{2,\infty}^{(1)}(x_1,x_2)
        =\\= \Theta(x_1,x_2) + n^{-1} \EE\left[O_{2,0}^{(1)}(x_1,x_2)\right] - n^{-1} \EE\left[O_{3,0}^{(1)}(x_1, x_2, \vec x) \Theta^{-1}(\vec x,\vec x) f_0^{(0)}(\vec x)\right]
        +\\+ n^{-1} \vec y^T \Theta^{-1}(\vec x,\vec x) \EE\left[O_{4,0}^{(1)}(x_1, x_2, \vec x, \vec x)\right] \Theta^{-1}(\vec x,\vec x) \vec y
        +\\+ n^{-1} \EE\left[f_0^{(0),T}(\vec x) \Theta^{-1}(\vec x,\vec x) O_{4,0}^{(1)}(x_1, x_2, \vec x, \vec x) \Theta^{-1}(\vec x,\vec x) f_0^{(0)}(\vec x)\right]
        -\\- n^{-1} \sum_{k,l=1}^m \frac{1}{\lambda_k (\lambda_k+\lambda_l)} \vec y^T \vec v_k \vec v_k^T \EE\left[O_{4,0}^{(1)}(x_1, x_2, \vec x, \vec x)\right] \vec v_l \vec v_l^T \vec y
        -\\- n^{-1} \sum_{k,l=1}^m \frac{1}{\lambda_k (\lambda_k+\lambda_l)} \EE\left[f_0^{(0),T}(\vec x) \vec v_k \vec v_k^T O_{4,0}^{(1)}(x_1, x_2, \vec x, \vec x) \vec v_l \vec v_l^T f_0^{(0)}(\vec x)\right].
    \end{multline}
    Since $\hat\Theta_0(x_1,x_2) = O_{2,0}^{(0)}(x_1,x_2) + n^{-1} O_{2,0}^{(1)}(x_1,x_2) + O(n^{-2})$, we have $\Theta(x_1,x_2) + n^{-1} \EE\left[O_{2,0}^{(1)}(x_1,x_2)\right] = \EE\hat\Theta_0(x_1,x_2) + O(n^{-2})$ and $\Theta(x_1,x_2) = \EE\hat\Theta_0(x_1,x_2) + O(n^{-1})$.
    For the same reason, $\EE\left[O_{4,0}(x_1, x_2, x_3, x_4)\right] = n^{-1} \EE\left[O_{4,0}^{(1)}(x_1, x_2, x_3, x_4)\right] + O(n^{-2})$.
    Suppose $f_0^{(0)}(\vec x) = 0$.
    Given this approximation, up to order $O(n^{-2})$,
    \begin{multline}
        \Theta^{NTH}(x_1,x_2)
        \approx \EE\hat\Theta_0(x_1,x_2) + \vec y^T \left(\EE\hat\Theta_0(\vec x,\vec x)\right)^{-1} \EE\left[O_{4,0}(x_1, x_2, \vec x, \vec x)\right] \left(\EE\hat\Theta_0(\vec x,\vec x)\right)^{-1} \vec y
        -\\- \sum_{k,l=1}^m \frac{1}{\lambda_k (\lambda_k+\lambda_l)} \vec y^T \vec v_k \vec v_k^T \EE\left[O_{4,0}(x_1, x_2, \vec x, \vec x)\right] \vec v_l \vec v_l^T \vec y.
    \end{multline}
    As one can see, $\Theta^{NTH}(x_1,x_2)$ depends on train labels $\vec y$.
    Roughly speaking, this kernel corresponds to the NTK of a network trained until convergence ($t\to\infty$); obviously, this kernel should depend on training data.

    As an interesting observation $\Theta^{NTH}(x_1,x_2) = \EE\hat\Theta_0(x_1,x_2) + \vec y^T M(x_1,x_2,\vec x) \vec y$ for a certain matrix $M$ --- recall that $K^{(HR)}(x_1,x_2)$ considered previously has a similar form.

    Note that computing the Gram matrix $\Theta^{NTH}(\vec x, \vec x)$ requires computing the Gram "matrix" of the expected 4-th order empirical kernel $\EE\left[O_{4,0}(\vec x, \vec x, \vec x, \vec x)\right]$.
    Instantiating this tensor requires $O(m^4)$ time and $O(m^4)$ memory which is only possible for very small datasets.

    \section{Limits of applicability}
    \label{sec:experiments}

    \begin{figure}
        \centering
        \includegraphics[width=0.9\textwidth]{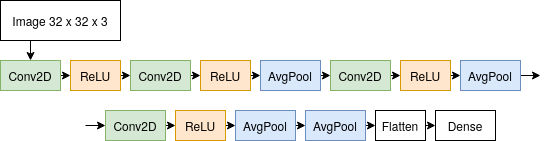}
        \caption{
            \label{fig:myrtle}
            Myrtle architecture.
        }
    \end{figure}

    \begin{figure}
        \centering
        \includegraphics[width=0.49\textwidth]{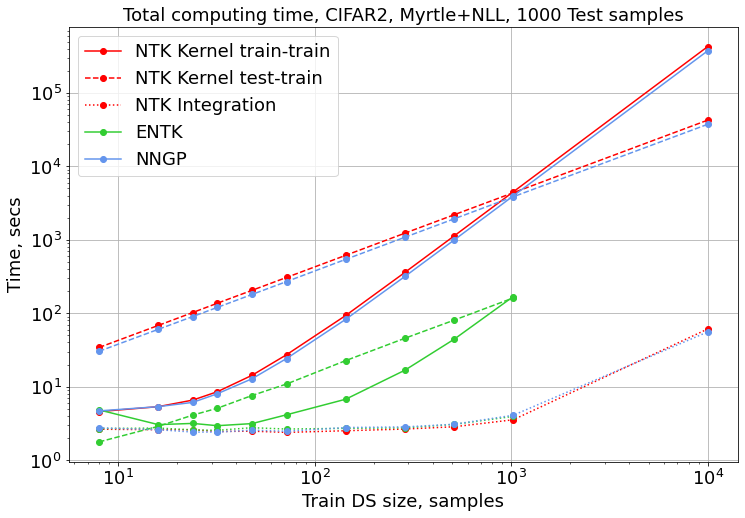}
        \includegraphics[width=0.49\textwidth]{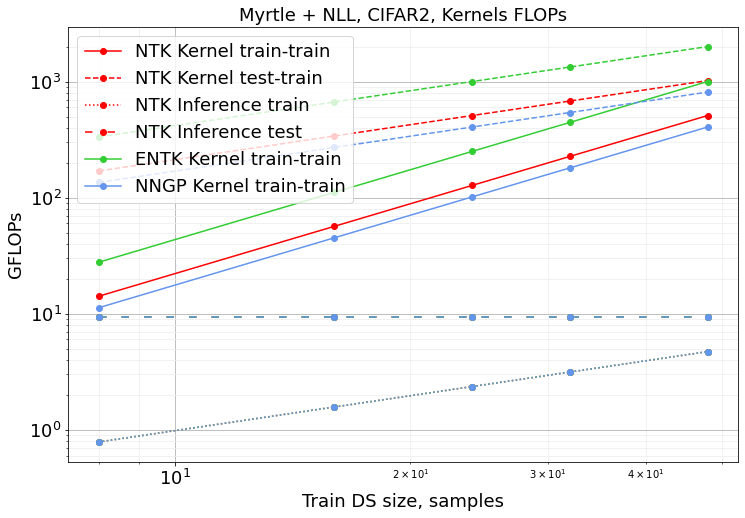}
        \caption{
            \label{fig:myrtle_bce_cifar2_time_flops}
            Myrtle network trained on subsets of CIFAR2 of different sizes.
            Different lines refer to different regimes of training (e.g. NTK, NNGP etc.) and different stages of training (e.g. cosntructing the kernel, integrating the dynamics etc.).
            We use BCE loss, and integrate the dynamics numerically for $T=10^4$ steps.
            We measure training time and number of FLOPS.
        }
    \end{figure}

    \begin{figure}
        \centering
        \includegraphics[width=0.49\textwidth]{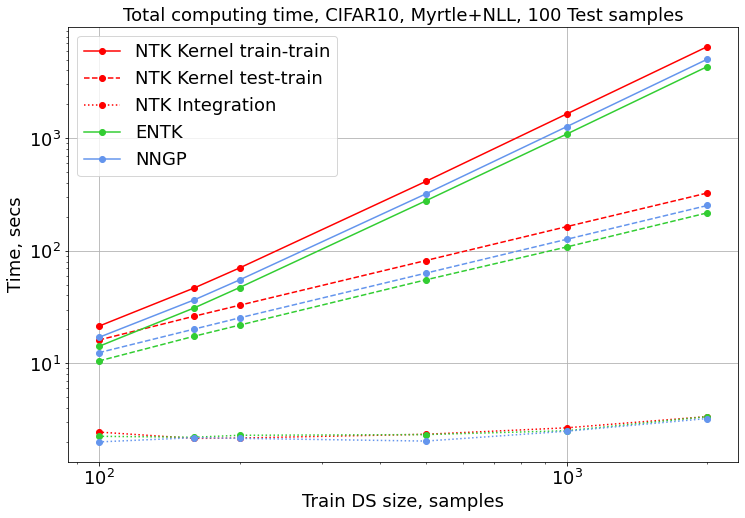}
        \includegraphics[width=0.49\textwidth]{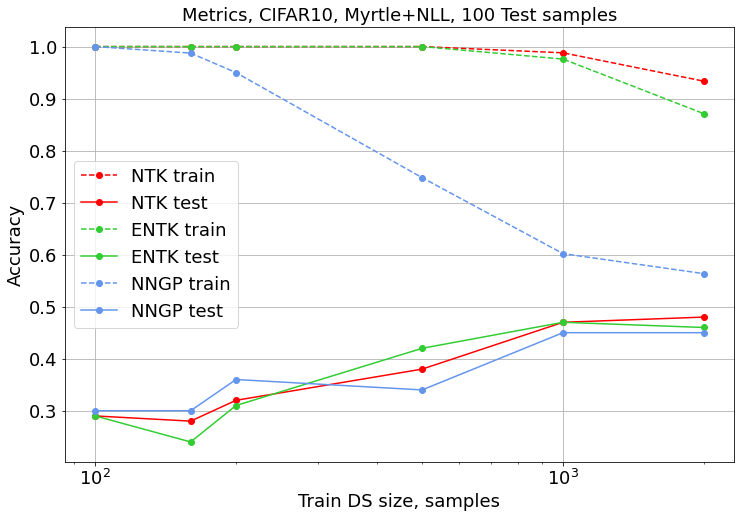}
        \caption{
            \label{fig:myrtle_bce_cifar10_time_accuracy}
            Myrtle network trained on subsets of CIFAR10 of different sizes.
            Different lines refer to different regimes of training (e.g. NTK, NNGP etc.) and different stages of training (e.g. cosntructing the kernel, integrating the dynamics etc.).
            We use cross-entropy loss, and integrate the dynamics numerically for $T=10^4$ steps.
            We measure training time and accuracy.
        }
    \end{figure}

    \begin{figure}
        \centering
        \includegraphics[width=0.49\textwidth]{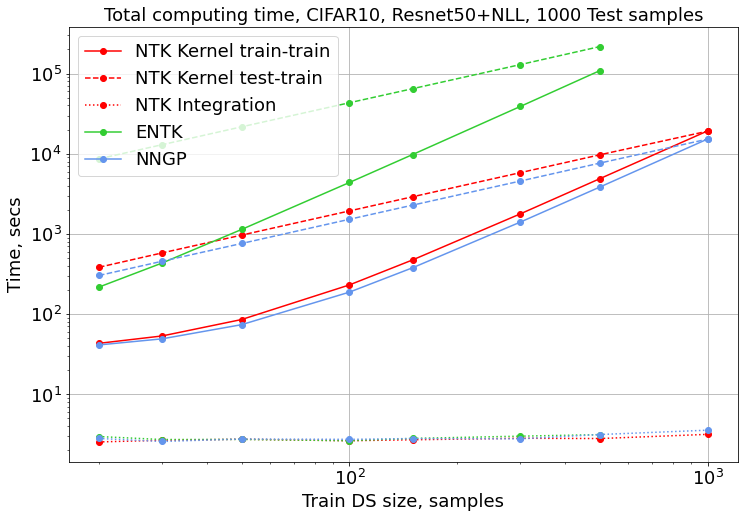}
        \includegraphics[width=0.49\textwidth]{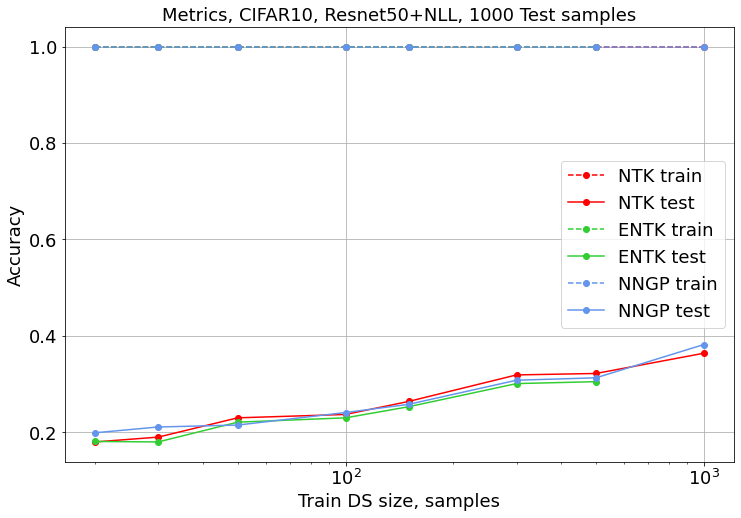}
        \caption{
            \label{fig:resnet50_nll_cifar10_time_accuracy}
            Resnet50 trained on subsets of CIFAR10 of different sizes.
            Different lines refer to different regimes of training (e.g. NTK, NNGP etc.) and different stages of training (e.g. cosntructing the kernel, integrating the dynamics etc.).
            We use cross-entropy loss, and integrate the dynamics numerically for $T=10^4$ steps.
            We measure training time and accuracy.
        }
    \end{figure}

    \begin{figure}
        \centering
        \includegraphics[width=0.49\textwidth]{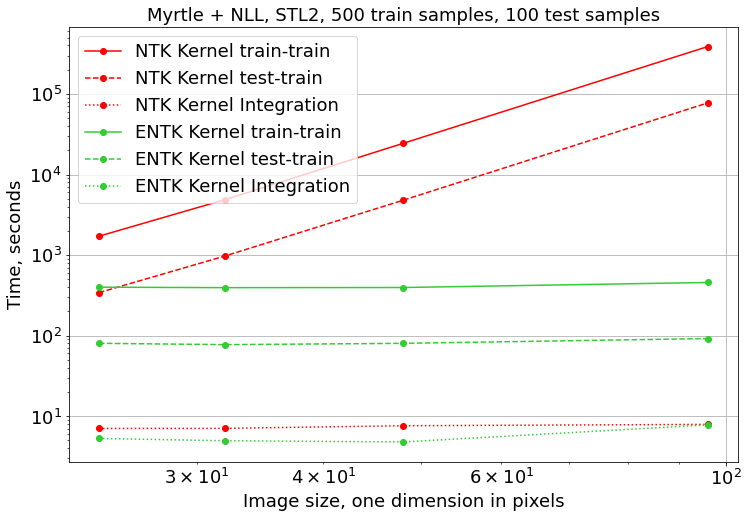}
        \includegraphics[width=0.49\textwidth]{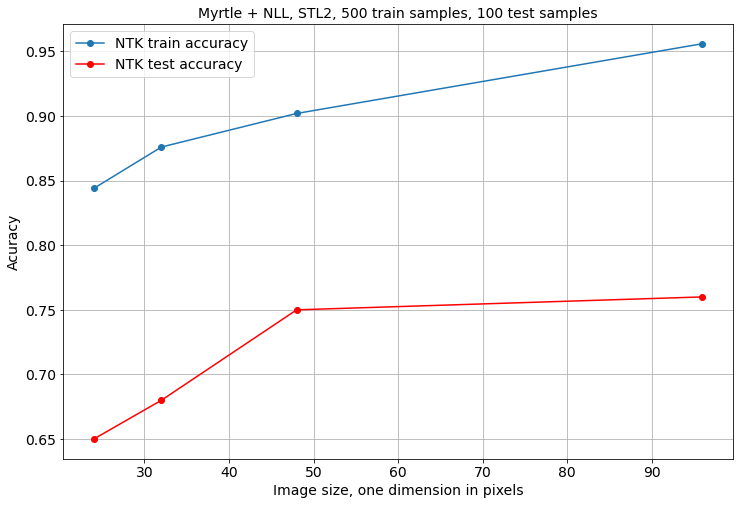}
        \caption{
            \label{fig:myrtle_bce_stl2_time_accuracy}
            Myrtle network trained on a subset of STL2 of size 500 with images of different resolutions.
            Different lines refer to different regimes of training (e.g. NTK, NNGP etc.) and different stages of training (e.g. cosntructing the kernel, integrating the dynamics etc.).
            We use BCE loss, and integrate the dynamics numerically for $T=10^4$ steps.
            We measure training time and accuracy.
        }
    \end{figure}
    
    In this section, we present a small experimental study on scope of applicability for NTK regression to real-time scenarios.
    In particular, we would like to investigate first, what is the maximal size $m$ of training dataset of images of given size we can afford with limited computational resources.
    Second, what is the maximal image resolution $d$ we can afford given fixed dataset size.
    We restrict ourselves to these two questions since for practical purposes, dependence of NTK regression complexity on these two parameters is the most worrying: it is $O(m^2 d^4)$ for constructing the Gram matrix, $O(m^3)$ for integrating the dynamics analytically, and $O(m^2 T)$ for integrating the dynamics numerically for $T$ steps; see \cref{sec:computations}.

    We use NeuralTangents \cite{novak2019neural} and perform all our experiments on a single GTX 1080Ti GPU with 12 GiB of memory.
    We consider a Myrtle network\footnote{\url{https://myrtle.ai/how-to-train-your-resnet-4-architecture/}} with 64 channels in all convolutional layers, see \cref{fig:myrtle}.
    We pick this architecture because it is lightweight and uses only those layers for which NTK can be computed analytically.

    For the first experiment, we consider two classes of CIFAR10 and refer this dataset as CIFAR2.
    We pick a subset of 1000 samples of the original test set of CIFAR2 and vary the size of the training subset.
    We optimize binary cross-entropy (BCE) and integrate the dynamics numerically for $T = 10^4$.
    We compute the Gram matrix of a kernel using batch size 4.
    On \cref{fig:myrtle_bce_cifar2_time_flops}, we plot training time and the number of floating-point operations (FLOPS) for different stages (i.e. Gram matrix computation, integrating the dynamics, inference on a test set) and for different regimes of training (analytical NTK, analytical NNGP, and empirical NTK) versus size of training dataset.
    As one can see, already for relatively small datasets ($m=10^4$), the most time-demanding stage is construction of the Gram matrix $\Theta(\vec x, \vec x)$ (solid line), but not integration (which is also takes time quadratic to size of the dataset) (dotted line).
    Also, the time to compute the NNGP kernel is almost the same as the one for NTK, since both are computed analytically; see \cref{sec:limit}.
    We could not obtain the point $m=10^4$ for empirical NTK (ENTK) due to numerical reasons.
    If we extrapolate the solid line to $m=10^6$, the size of ImageNet, noting the quadratic growth, we will get $5 \times 10^9$ seconds, which is around 160 years of computations.
    While our time measurements are device-dependent, we also measure the number of FLOPS, which while being device-independent, grows the same way as time and is also quite large.
    This experiment demonstrates that indeed, the naive approach for integrating the NTK dynamics falls short on datasets of realistic sizes, thus striving for major optimizations.
    As mentioned in \cref{sec:computations}, a promising approach could be the one of \cite{meanti2020kernel}.

    On \cref{fig:myrtle_bce_cifar10_time_accuracy}, we present the same experiment but with all 10 classes of CIFAR10.
    We observe the same quadratic time growth issue for all three regimes of training (analytical NTK, analytical NNGP, and empirical NTK).
    We also report accuracy for comparison with previous works on small data training with kernel methods (i.e. \cite{arora2019harnessing}).

    In addition to experiments with a small network, we experimented with a variant of Resnet50 \cite{he2016deep}.
    We modify this architecture by removing batch normalizations and substituting max poolings with average poolings, so to make analytical computations possible.
    Results are shown on \cref{fig:resnet50_nll_cifar10_time_accuracy}.
    Doing the same extrapolation to ImageNet size, we get $6.25 \times 10^{11}$ seconds, which is around $20000$ years.


    Lastly, we consider two classes of STL10 and similarly to CIFAR2, refer this dataset as STL2.
    We pick a subset of 100 samples of the original test set of STL2 and 500 samples of its original train set.
    While STL10 has fewer labeled examples compared to CIFAR10, it has larger images: $96 
    \times 96$ for STL10 versus $32 \times 32$ for CIFAR10.
    We vary size of the input image and measure training time and accuracy, similarly to the first experiment.
    As before, we optimize binary cross-entropy (BCE) and integrate the dynamics numerically for $T = 10^4$.
    However, we use batch size 1 for computing the Gram matrix, since larger batch sizes do not fit in GPU memory for large image sizes.
    Results are shown on \cref{fig:myrtle_bce_stl2_time_accuracy}.
    As before, the most time-demanding part is kernel Gram matrix computation (blue line): it grows as $O(d^4)$, where $d$ is image resolution; see \cref{sec:limit}.
    If we extrapolate this line to $d=224$, the resolution on which traditional ImageNet classification models operate, we will get around 150 days of computations.
    This experiment therefore demonstrates that not only dataset size, but also image resolution complexity can also be a serious bottleneck in applying NTK approach in practice.
    Also, while for dataset size, certain optimizations are available (e.g. \cite{meanti2020kernel}), we are not aware of any optimizations aiming for decreasing image resolution complexity.

    \section{Conclusions}

    The use of NTK theory is twofold: first, it relates neural networks to kernel methods, a far more well-developped class of models.
    Second, it gives a machine learning practitioner a kernel that shares some properties with neural nets.
    
    Recall what we have concerning the first application.
    We have a theorem (\cref{thm:master_theorem}) that implies that a neural tangent kernel of a wide class of architectures is deterministic and does not evolve with time in the limit of infinite width, and provides a recurrent formula for the limit.
    Therefore a network that is wide enough should share some properties, i.e. convergence and generalization, see \cref{sec:app_theory}, with the corresponding kernel method.
    However, the resulting width bounds are far from realistic.
    Second, the limit kernel does not evolve with time only under certain non-standard parameterization rarely used in practice.
    In contrast, standard parameterization results in evolving (normalized) kernel, see \cref{sec:standard_param}.
    The fact that the kernel evolves may be the key to understanding superior performance of neural nets to kernel methods.
    Unfortunately, we have little understanding of this aspects at the moment.
    Lastly, \cref{thm:master_theorem} requires Gaussian weight initialization rarely used in practice.
    Generalizing it to non-Gaussian weight distribution remains to be done in the future.

    Let us discuss the second application.
    At the moment of writing, computing the exact limit kernel was available only for convolutional and fully-connected networks with average poolings and nonlinearities in a certain class, see \cref{sec:computations}.
    For other architectures, one has to rely on empirical NTK which is a biased estimate of the limit one.
    Computing the empirical NTK requires instantiating output-by-weight jacobians at every pair of training points, which is especially memory risky for realistically large architectures.
    Storing the Gram matrix of the kernel also requires $O(m^2)$ memory where $m$ is dataset size.
    Even if the kernel is sucessfully computed on every pair of training points, integrating the training dynamics naively requires inverting the Gram matrix, which costs $O(m^3)$ time, while for datasets of size $10^6$ one can barely afford more than $O(m)$ time and memory.
    We study applicability limits of this naive approach in \cref{sec:experiments}.
    Still, certain optimization are available, see \cref{sec:computations}.

    Also concerning the second application, NTK is not the only kernel that can be constructed using a neural network; certain other kernels may have computational or performance gains compared to NTK, see \cref{sec:beyond}.

    \bibliography{references}

\begin{thebibliography}{}

\bibitem[Allen-Zhu et~al., 2019]{allen2019convergence}
Allen-Zhu, Z., Li, Y., and Song, Z. (2019).
\newblock A convergence theory for deep learning via over-parameterization.
\newblock In {\em International Conference on Machine Learning}, pages
  242--252. PMLR.

\bibitem[Arora et~al., 2019a]{arora2019fine}
Arora, S., Du, S., Hu, W., Li, Z., and Wang, R. (2019a).
\newblock Fine-grained analysis of optimization and generalization for
  overparameterized two-layer neural networks.
\newblock In {\em International Conference on Machine Learning}, pages
  322--332.

\bibitem[Arora et~al., 2019b]{arora2019exact}
Arora, S., Du, S.~S., Hu, W., Li, Z., Salakhutdinov, R.~R., and Wang, R.
  (2019b).
\newblock On exact computation with an infinitely wide neural net.
\newblock In {\em Advances in Neural Information Processing Systems}, pages
  8141--8150.

\bibitem[Arora et~al., 2019c]{arora2019harnessing}
Arora, S., Du, S.~S., Li, Z., Salakhutdinov, R., Wang, R., and Yu, D. (2019c).
\newblock Harnessing the power of infinitely wide deep nets on small-data
  tasks.
\newblock {\em arXiv preprint arXiv:1910.01663}.

\bibitem[Bietti and Mairal, 2019]{bietti2019inductive}
Bietti, A. and Mairal, J. (2019).
\newblock On the inductive bias of neural tangent kernels.
\newblock {\em arXiv preprint arXiv:1905.12173}.

\bibitem[Bradbury et~al., 2018]{jax2018github}
Bradbury, J., Frostig, R., Hawkins, P., Johnson, M.~J., Leary, C., Maclaurin,
  D., Necula, G., Paszke, A., Vander{P}las, J., Wanderman-{M}ilne, S., and
  Zhang, Q. (2018).
\newblock {JAX}: composable transformations of {P}ython+{N}um{P}y programs.

\bibitem[Chen and Xu, 2020]{chen2020deep}
Chen, L. and Xu, S. (2020).
\newblock Deep neural tangent kernel and laplace kernel have the same rkhs.
\newblock {\em arXiv preprint arXiv:2009.10683}.

\bibitem[Chen et~al., 2020]{chen2020label}
Chen, S., He, H., and Su, W.~J. (2020).
\newblock Label-aware neural tangent kernel: Toward better generalization and
  local elasticity.
\newblock {\em arXiv preprint arXiv:2010.11775}.

\bibitem[Chen et~al., 2021]{chen2021neural}
Chen, W., Gong, X., and Wang, Z. (2021).
\newblock Neural architecture search on imagenet in four gpu hours: A
  theoretically inspired perspective.
\newblock {\em arXiv preprint arXiv:2102.11535}.

\bibitem[Dong and Yang, 2020]{dong2020bench}
Dong, X. and Yang, Y. (2020).
\newblock Nas-bench-201: Extending the scope of reproducible neural
  architecture search.
\newblock {\em arXiv preprint arXiv:2001.00326}.

\bibitem[Du et~al., 2019a]{du2019gradient}
Du, S., Lee, J., Li, H., Wang, L., and Zhai, X. (2019a).
\newblock Gradient descent finds global minima of deep neural networks.
\newblock In {\em International Conference on Machine Learning}, pages
  1675--1685. PMLR.

\bibitem[Du et~al., 2019b]{du2018gradient}
Du, S.~S., Zhai, X., Poczos, B., and Singh, A. (2019b).
\newblock Gradient descent provably optimizes over-parameterized neural
  networks.
\newblock In {\em International Conference on Learning Representations}.

\bibitem[Dupuis and Jacot, 2021]{dupuis2021dnn}
Dupuis, B. and Jacot, A. (2021).
\newblock Dnn-based topology optimisation: Spatial invariance and neural
  tangent kernel.
\newblock {\em arXiv preprint arXiv:2106.05710}.

\bibitem[Dyer and Gur-Ari, 2020]{Dyer2020Asymptotics}
Dyer, E. and Gur-Ari, G. (2020).
\newblock Asymptotics of wide networks from feynman diagrams.
\newblock In {\em International Conference on Learning Representations}.

\bibitem[Fort et~al., 2020]{fort2020deep}
Fort, S., Dziugaite, G.~K., Paul, M., Kharaghani, S., Roy, D.~M., and Ganguli,
  S. (2020).
\newblock Deep learning versus kernel learning: an empirical study of loss
  landscape geometry and the time evolution of the neural tangent kernel.
\newblock {\em arXiv preprint arXiv:2010.15110}.

\bibitem[Geifman et~al., 2020]{geifman2020similarity}
Geifman, A., Yadav, A., Kasten, Y., Galun, M., Jacobs, D., and Basri, R.
  (2020).
\newblock On the similarity between the laplace and neural tangent kernels.
\newblock {\em arXiv preprint arXiv:2007.01580}.

\bibitem[Golikov, 2020a]{golikov2020dynamically}
Golikov, E.~A. (2020a).
\newblock Dynamically stable infinite-width limits of neural classifiers.
\newblock {\em arXiv preprint arXiv:2006.06574}.

\bibitem[Golikov, 2020b]{golikov2020notes}
Golikov, E.~A. (2020b).
\newblock Notes on deep learning theory.
\newblock {\em arXiv preprint arXiv:2012.05760}.

\bibitem[He et~al., 2015]{he2015delving}
He, K., Zhang, X., Ren, S., and Sun, J. (2015).
\newblock Delving deep into rectifiers: Surpassing human-level performance on
  imagenet classification.
\newblock In {\em Proceedings of the IEEE international conference on computer
  vision}, pages 1026--1034.

\bibitem[He et~al., 2016]{he2016deep}
He, K., Zhang, X., Ren, S., and Sun, J. (2016).
\newblock Deep residual learning for image recognition.
\newblock In {\em Proceedings of the IEEE conference on computer vision and
  pattern recognition}, pages 770--778.

\bibitem[Hron et~al., 2020]{hron2020infinite}
Hron, J., Bahri, Y., Sohl-Dickstein, J., and Novak, R. (2020).
\newblock Infinite attention: Nngp and ntk for deep attention networks.
\newblock In {\em International Conference on Machine Learning}, pages
  4376--4386. PMLR.

\bibitem[Huang and Yau, 2019]{huang2019dynamics}
Huang, J. and Yau, H.-T. (2019).
\newblock Dynamics of deep neural networks and neural tangent hierarchy.
\newblock {\em arXiv preprint arXiv:1909.08156}.

\bibitem[Jacot et~al., 2018]{jacot2018neural}
Jacot, A., Gabriel, F., and Hongler, C. (2018).
\newblock Neural tangent kernel: Convergence and generalization in neural
  networks.
\newblock In {\em Advances in neural information processing systems}, pages
  8571--8580.

\bibitem[Kimeldorf and Wahba, 1970]{kimeldorf1970correspondence}
Kimeldorf, G.~S. and Wahba, G. (1970).
\newblock A correspondence between bayesian estimation on stochastic processes
  and smoothing by splines.
\newblock {\em The Annals of Mathematical Statistics}, 41(2):495--502.

\bibitem[Lee et~al., 2018]{lee2018deep}
Lee, J., Bahri, Y., Novak, R., Schoenholz, S.~S., Pennington, J., and
  Sohl-Dickstein, J. (2018).
\newblock Deep neural networks as gaussian processes.
\newblock In {\em International Conference on Learning Representations}.

\bibitem[Lee et~al., 2019]{lee2019wide}
Lee, J., Xiao, L., Schoenholz, S., Bahri, Y., Novak, R., Sohl-Dickstein, J.,
  and Pennington, J. (2019).
\newblock Wide neural networks of any depth evolve as linear models under
  gradient descent.
\newblock In {\em Advances in neural information processing systems}, pages
  8572--8583.

\bibitem[Lee et~al., 2016]{lee2016gradient}
Lee, J.~D., Simchowitz, M., Jordan, M.~I., and Recht, B. (2016).
\newblock Gradient descent only converges to minimizers.
\newblock In {\em Conference on learning theory}, pages 1246--1257.

\bibitem[Martens et~al., 2021]{martens2021rapid}
Martens, J., Ballard, A., Desjardins, G., Swirszcz, G., Dalibard, V.,
  Sohl-Dickstein, J., and Schoenholz, S.~S. (2021).
\newblock Rapid training of deep neural networks without skip connections or
  normalization layers using deep kernel shaping.
\newblock {\em arXiv preprint arXiv:2110.01765}.

\bibitem[Meanti et~al., 2020]{meanti2020kernel}
Meanti, G., Carratino, L., Rosasco, L., and Rudi, A. (2020).
\newblock Kernel methods through the roof: handling billions of points
  efficiently.
\newblock {\em arXiv preprint arXiv:2006.10350}.

\bibitem[Mertikopoulos et~al., 2020]{mertikopoulos2020almost}
Mertikopoulos, P., Hallak, N., Kavis, A., and Cevher, V. (2020).
\newblock On the almost sure convergence of stochastic gradient descent in
  non-convex problems.
\newblock {\em arXiv preprint arXiv:2006.11144}.

\bibitem[Nguyen, 2019]{nguyen2019connected}
Nguyen, Q. (2019).
\newblock On connected sublevel sets in deep learning.
\newblock In {\em International Conference on Machine Learning}, pages
  4790--4799.

\bibitem[Nguyen, 2021]{nguyen2021note}
Nguyen, Q. (2021).
\newblock A note on connectivity of sublevel sets in deep learning.
\newblock {\em arXiv preprint arXiv:2101.08576}.

\bibitem[Nguyen and Hein, 2017]{nguyen2017loss}
Nguyen, Q. and Hein, M. (2017).
\newblock The loss surface of deep and wide neural networks.
\newblock In {\em Proceedings of the 34th International Conference on Machine
  Learning-Volume 70}, pages 2603--2612.

\bibitem[Novak et~al., 2021]{novakfast}
Novak, R., Sohl-Dickstein, J., and Schoenholz, S.~S. (2021).
\newblock Fast finite width neural tangent kernel.
\newblock {\em Bayesian Deep Learning NeurIPS 2021 Workshop}.

\bibitem[Novak et~al., 2019]{novak2019neural}
Novak, R., Xiao, L., Hron, J., Lee, J., Alemi, A.~A., Sohl-Dickstein, J., and
  Schoenholz, S.~S. (2019).
\newblock Neural tangents: Fast and easy infinite neural networks in python.
\newblock {\em arXiv preprint arXiv:1912.02803}.

\bibitem[Panageas and Piliouras, 2017]{panageas2017gradient}
Panageas, I. and Piliouras, G. (2017).
\newblock Gradient descent only converges to minimizers: Non-isolated critical
  points and invariant regions.
\newblock In {\em 8th Innovations in Theoretical Computer Science Conference
  (ITCS 2017)}. Schloss Dagstuhl-Leibniz-Zentrum fuer Informatik.

\bibitem[Park et~al., 2020]{park2020towards}
Park, D.~S., Lee, J., Peng, D., Cao, Y., and Sohl-Dickstein, J. (2020).
\newblock Towards nngp-guided neural architecture search.
\newblock {\em arXiv preprint arXiv:2011.06006}.

\bibitem[Radhakrishnan et~al., 2021]{radhakrishnan2021simple}
Radhakrishnan, A., Stefanakis, G., Belkin, M., and Uhler, C. (2021).
\newblock Simple, fast, and flexible framework for matrix completion with
  infinite width neural networks.
\newblock {\em arXiv preprint arXiv:2108.00131}.

\bibitem[Shankar et~al., 2020]{shankar2020neural}
Shankar, V., Fang, A., Guo, W., Fridovich-Keil, S., Ragan-Kelley, J., Schmidt,
  L., and Recht, B. (2020).
\newblock Neural kernels without tangents.
\newblock In {\em International Conference on Machine Learning}, pages
  8614--8623. PMLR.

\bibitem[Song and Yang, 2019]{song2019quadratic}
Song, Z. and Yang, X. (2019).
\newblock Quadratic suffices for over-parametrization via matrix chernoff
  bound.
\newblock {\em arXiv preprint arXiv:1906.03593}.

\bibitem[Tancik et~al., 2020]{tancik2020fourier}
Tancik, M., Srinivasan, P.~P., Mildenhall, B., Fridovich-Keil, S., Raghavan,
  N., Singhal, U., Ramamoorthi, R., Barron, J.~T., and Ng, R. (2020).
\newblock Fourier features let networks learn high frequency functions in low
  dimensional domains.
\newblock {\em arXiv preprint arXiv:2006.10739}.

\bibitem[Xiao et~al., 2018]{xiao2018dynamical}
Xiao, L., Bahri, Y., Sohl-Dickstein, J., Schoenholz, S., and Pennington, J.
  (2018).
\newblock Dynamical isometry and a mean field theory of cnns: How to train
  10,000-layer vanilla convolutional neural networks.
\newblock In {\em International Conference on Machine Learning}, pages
  5393--5402. PMLR.

\bibitem[Yang, 2019]{yang2019tensor_i}
Yang, G. (2019).
\newblock Tensor programs i: Wide feedforward or recurrent neural networks of
  any architecture are gaussian processes.
\newblock {\em arXiv preprint arXiv:1910.12478}.

\bibitem[Yang, 2020a]{yang2020tensor_ii}
Yang, G. (2020a).
\newblock Tensor programs ii: Neural tangent kernel for any architecture.
\newblock {\em arXiv preprint arXiv:2006.14548}.

\bibitem[Yang, 2020b]{yang2020tensor_iii}
Yang, G. (2020b).
\newblock Tensor programs iii: Neural matrix laws.
\newblock {\em arXiv preprint arXiv:2009.10685}.

\bibitem[Yang and Littwin, 2021]{yang2021tensor_iib}
Yang, G. and Littwin, E. (2021).
\newblock Tensor programs iib: Architectural universality of neural tangent
  kernel training dynamics.
\newblock {\em arXiv preprint arXiv:2105.03703}.

\bibitem[Yu and Chen, 1995]{yu1995local}
Yu, X.-H. and Chen, G.-A. (1995).
\newblock On the local minima free condition of backpropagation learning.
\newblock {\em IEEE Transactions on Neural Networks}, 6(5):1300--1303.

\bibitem[Yue et~al., 2021]{yue2021neural}
Yue, K., Jin, R., Pilgrim, R., Wong, C.-W., Baron, D., and Dai, H. (2021).
\newblock Neural tangent kernel empowered federated learning.
\newblock {\em arXiv preprint arXiv:2110.03681}.

\end{thebibliography}
    \bibliographystyle{apalike}

\end{document}